\documentclass{article}

\usepackage{PRIMEarxiv}

\usepackage[dvipsnames]{xcolor}
\usepackage[utf8]{inputenc} \usepackage[T1]{fontenc}    

\usepackage{microtype}
\usepackage{graphicx}
\usepackage{booktabs} \usepackage{tabularray}
\usepackage{arydshln}
\usepackage{multirow}
\usepackage{hyperref}
\usepackage{comment}
\usepackage{caption}
\usepackage{subcaption}
\usepackage{float}
\usepackage{algorithm} 
\usepackage{algpseudocode} 
\usepackage{ stmaryrd }
\usepackage{mathtools}
\usepackage{soul}
\usepackage{tikz}

\usepackage{enumitem}
\usepackage{wrapfig,lipsum,booktabs}

\usepackage{amsmath}
\usepackage{amssymb}
\usepackage{mathtools}
\usepackage{amsthm}
\usepackage{verbatim}
\usepackage{bm}
\usepackage{amssymb}
\usepackage[toc,page,header]{appendix}
\usepackage{minitoc}

\usepackage{pifont}\newcommand{\cmark}{{\color{Green}\ding{51}}}\newcommand{\xmark}{{\color{Red}\ding{55}}}

\usepackage[capitalize,noabbrev]{cleveref}
\RequirePackage{natbib}

\theoremstyle{plain}
\newtheorem{theorem}{Theorem}
\newtheorem{proposition}[theorem]{Proposition}
\newtheorem{lemma}[theorem]{Lemma}

\newtheorem{definition}{Definition}

\theoremstyle{definition}

\usepackage{xspace}

\usepackage{hyperref}

\usepackage{xcolor}

\newcommand{\contextName}{Dynamic Network VFL\xspace} 
\newcommand{\contextAbv}{DN-VFL\xspace} \newcommand{\risk}{Dynamic Risk\xspace} \newcommand{\method}{\textbf{M}ultiple \textbf{A}ggregation with \textbf{G}ossip Rounds and \textbf{S}imulated Faults\xspace}
\newcommand{\methodName}{\method}
\newcommand{\methodSf}{MAGS\xspace}
\newcommand{\methodAbv}{\methodSf}

 \newcommand{\average}{Select Active Client\xspace}  \newcommand{\averageAbv}{active\xspace}

\newcommand{\E}{\mathbb{E}}

\newcommand{\given}{\,\big|\,}
\newcommand{\asize}{b} 

\newcommand{\dataset}{\mathcal{D}}
\newcommand{\xrv}{X}
\newcommand{\yrv}{Y}
\newcommand{\xinst}{\bm{x}}
\newcommand{\yinst}{y} \renewcommand{\paragraph}[1]{\noindent\textbf{#1}\quad}

\usepackage[capitalize,noabbrev]{cleveref}

\title{Robust Collaborative Inference with Vertically Split \\ Data Over Dynamic Device Environments}
\author{
  Surojit Ganguli, Zeyu Zhou, Christopher G. Brinton, David I. Inouye \\
Purdue University \\
  West Lafeyette\\
  \texttt{\{sganguli,zhou1059,cgb,dinouye\}@purdue.edu} \\
}

\begin{document}
\doparttoc \faketableofcontents 

\part{} 

\renewcommand{\paragraph}[1]{\textbf{#1.} }
\setlength{\abovedisplayskip}{0pt}
\setlength{\belowdisplayskip}{0pt}

\maketitle

\begin{abstract}
When each edge device of a network only perceives a local part of the environment, collaborative inference across multiple devices is often needed to predict global properties of the environment.
In safety-critical applications, collaborative inference must be robust to significant network failures caused by environmental disruptions or extreme weather.
Existing collaborative learning approaches, such as privacy-focused Vertical Federated Learning (VFL), typically assume a centralized setup or that one device never fails.
However, these assumptions make prior approaches susceptible to significant network failures.
To address this problem, we first formalize the problem of robust collaborative inference over a dynamic network of devices that could experience significant network faults.
Then, we develop a minimalistic yet impactful method called \method (\methodSf) that synthesizes simulated faults via dropout, replication, and gossiping to significantly improve robustness over baselines. 
We also theoretically analyze our proposed approach to explain why each component enhances robustness.
Extensive empirical results validate that \methodSf is robust across a range of fault rates---including extreme fault rates.
\end{abstract}

\section{Introduction}\label{sec:IntoAndMot}
Intelligent device networks in safety-critical applications (e.g., search and rescue) must function despite catastrophic faults (e.g., when 50\% devices failure). 
Achieving this level of resilience is challenging because collaborative cross-device operations on IoT or edge devices present unique obstacles not encountered in cross-silo setups \citep{yuan2023decentralized}, including limited power resources, network unreliability, and the absence of a centralized server. 
Moreover, when devices must collaborate to predict a global feature of the environment, these challenges become especially acute.
For example, deploying a network of sensors in harsh environments (e.g., deep sea, underground mines, or remote rural areas) may lead to device failure due to power constraints or extreme weather conditions. Therefore, in this work, we seek to answer the following question: \textbf{Can we develop a cross-device collaborative learning (CL) method that maintains strong performance at test time even under near-catastrophic faults in the decentralized setting?}

When working with vertically split data, Vertical Federated Learning (VFL) \citep{liu2024vertical}, a privacy-focused variant of collaborative learning, emerges as a potential solution for cross-device collaboration.
In VFL, clients share the same set of samples but have different features. 
In our environmental monitoring example, the samples correspond to unique timestamps, and the features correspond to sensor data from each device—each providing a partial view of the global environment. 
Much prior work in VFL focuses on the privacy aspect of collaborative learning via homomorphic encryption and the like \citep{li2023fedvs,tran2024differentially}, but fault-tolerance in the cross-device setting, the focus of this paper, has been given less attention.
Some previous research in VFL has explored aspects of fault tolerance and decentralized learning, primarily focusing on the training phase. 
For instance, studies have addressed asynchronous communication to handle device failures during training \citep{chen2020vafl,zhang2021secure,li2020efficient,li2023fedvs}. 
An exception is the work by \citet{sun2023robust}, who proposed Party-wise Dropout (PD) to mitigate inference-time faults caused by passive parties (clients) dropping off unexpectedly, but they assumed that the active party (server) remains fault-free. 
Other works have focused on communication efficiency or handling missing features in decentralized VFL \citep{valdeira2023multi, valdeira2024vertical}. 
However, to the best of our knowledge, no prior work in the cross-device setting simultaneously addresses  decentralized learning and arbitrary faults---including the active party or server---during inference. 
This gap, as summarized in \Cref{tab:expt}, motivates our research.

\begin{table}[!ht]
\caption{Our \methodAbv method considers the setting of vertically split data across devices, where faults can occur in both clients and the active party or server. 
This is unlike the  existing literature in decentralized or fault-tolerant VFL.}
\label{tab:expt}
\centering
\setlength{\tabcolsep}{2pt}
{\fontsize{6.7}{6.7}\selectfont
\begin{tabular}{lccccc}
\toprule
& \textbf{Context} &  \multicolumn{2}{c}{\textbf{Faults During}} & \multicolumn{2}{c}{\textbf{Fault Types}} \\
\cmidrule(lr){2-2} \cmidrule(lr){3-4} \cmidrule(lr){5-6}
& & Training & Inference & Client & {Active Party}  \\
\midrule
STCD\citep{valdeira2023multi} & Cross-silo & \xmark & \xmark & \xmark & \xmark \\ 
VAFL\citep{chen2020vafl} & Cross-silo & \cmark & \xmark & \cmark & \xmark \\
Straggler VFL\citep{li2023fedvs} & Cross-silo & \cmark & \xmark & \cmark & \xmark \\
PD \citep{sun2023robust} & Cross-silo & \cmark\footnotemark[1] & \cmark & \cmark & \xmark \\\hdashline
MAGS (ours) & Cross-device & \cmark & \cmark & \cmark & \cmark \\
\bottomrule
\end{tabular}
}
\vspace{-1em}
\end{table}
\footnotetext[1]{Even though \citet{sun2023robust} did not explicitly consider client faults during training, the method from \citet{sun2023robust} could handle training faults by treating them like Party-wise Dropout(PD) as discussed in \Cref{sec:dropout}.}

To address these challenges holistically, we first formalize this problem setup and then propose a minimalistic yet impactful solution, \methodName (\methodAbv).
A comparison of context presented in this work with VFL inspired collaborative learning techniques  are captured in \cref{Fig:Diffeereent Paradigm Comparison}. 
\methodAbv significantly improves robustness by integrating three interconnected methods that build upon and complement each other.
First, during training, we simulate high fault rates via dropout so that the model can be robust to more missing values at test time.
Second, we replicate the data aggregator to prevent catastrophic failure in case the active party (or server) goes down during test time.
Third, we introduce gossip rounds to implicitly ensemble the predictions from multiple data aggregators, reducing the prediction variance across devices.
Finally, we evaluate the effectiveness of \methodAbv by conducting experiments using five datasets (StarCraftMNIST \citep{kulinski2023starcraftimage} in the main paper and MNIST, CIFAR10, CIFAR100, Tiny ImageNet in the appendix) and different network configurations.
The results establish that \methodAbv is significantly more robust than prior methods, often improving performance more than 20\% over prior methods at high fault rates.
We summarize our contributions as follows:
\begin{itemize}[leftmargin=*]
    \item We formalize the problem of decentralized collaborative learning under dynamic network conditions, called \contextName (\contextAbv), and define \risk, which measures performance under (extreme) dynamic network conditions. 
\item We develop and analyze MAGS, that combines fault simulation, replication, and gossiping to enable strong fault tolerance during inference for \contextAbv.
\item We demonstrate that MAGS is significantly more robust to dynamic network faults than prior methods across multiple datasets, often improving performance more than 20\% compared to prior methods.
\vspace{-8pt}
\end{itemize}

\begin{figure*}[!ht]
    \centering
    \begin{subfigure}[t]{0.46\linewidth}
        \centering
        \includegraphics[clip, trim=0cm 4.5cm 13cm 4cm,width=\linewidth]{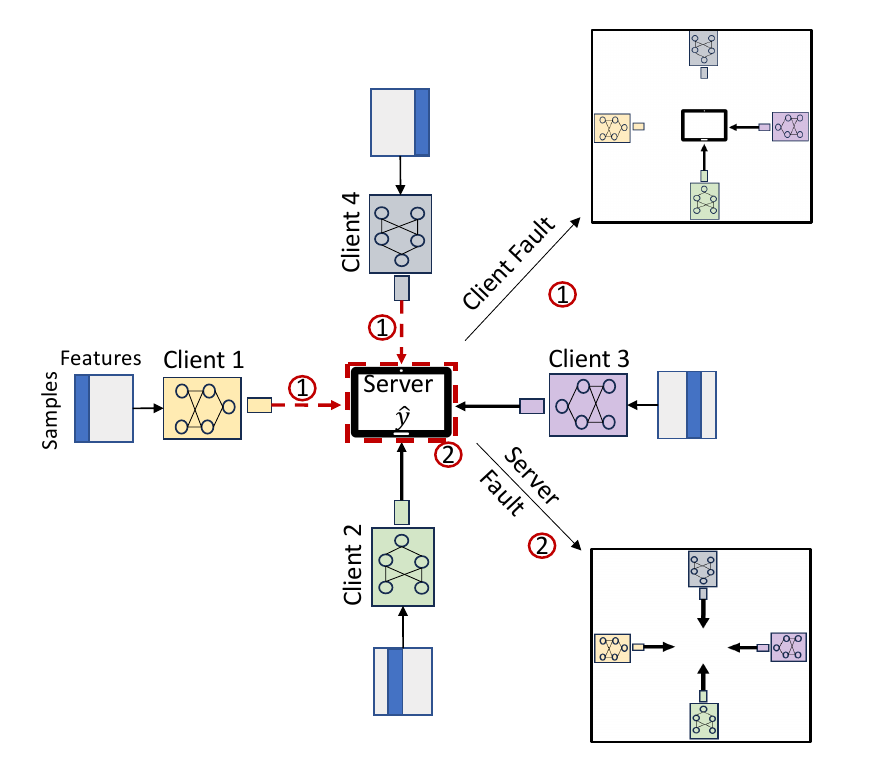}
        \caption{Collaborative Learning for Vertically Split Data (e.g., VFL)}
        \label{subfig:VFL-context}
    \end{subfigure}
\begin{subfigure}[t]{0.53\textwidth}
        \includegraphics[clip, trim=0cm 4.5cm 12cm 4cm,width=\textwidth]{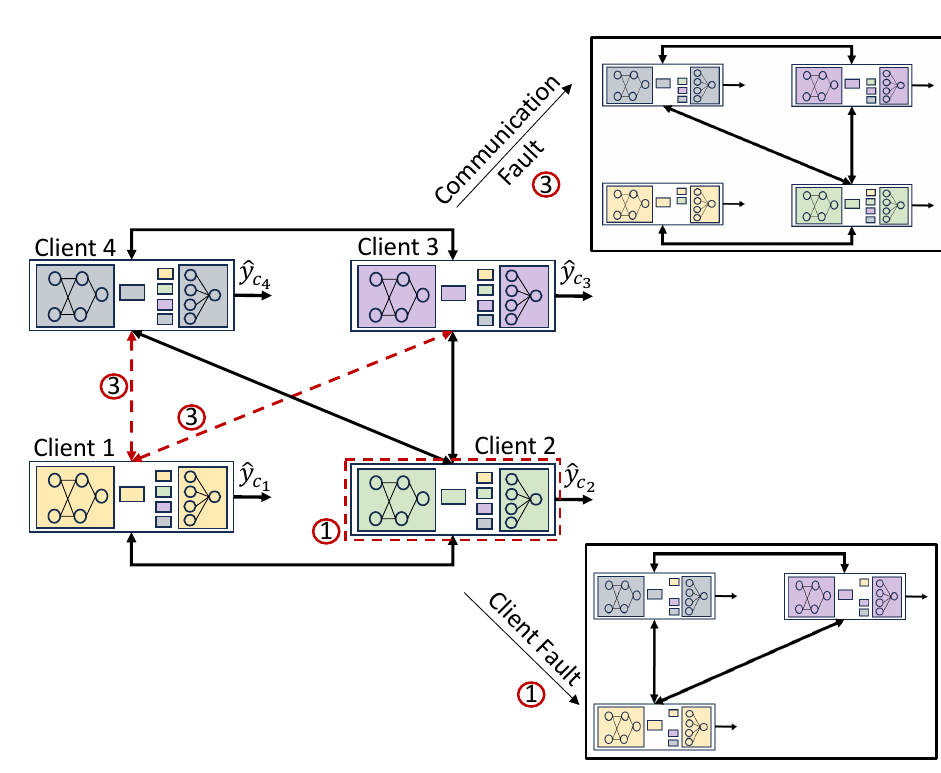}
        \caption{\contextName (\contextAbv)}
        \label{subfig:iDecentralized-context}
    \end{subfigure}
    \vspace{-0.75em}
    \caption{
    Collaborative Learning (CL)(\cref{subfig:VFL-context}) assumes samples are split across clients  with a central server.
    The data context in our study is the same as VFL where the features are split across clients. However, in our case, no centralized server node is assumed, and clients serve as data aggregators (\cref{subfig:iDecentralized-context}).
    Our goal is to obtain robust test time performance even under highly dynamic networks such as client/device faults (\textcolor{red}{\textcircled{1}}), server faults (\textcolor{red}{\textcircled{2}}) and communication faults (\textcolor{red}{\textcircled{3}}).
    }
    \vspace{-1em}
    \label{Fig:Diffeereent Paradigm Comparison}
\end{figure*}

\subsection{Related works}
\label{sec:related-works}

\paragraph{Network Dynamic Resilient Federated Learning (FL)}
In VFL, network dynamics has mostly been studied during the training phase for asynchronous client participation \citet{chen2020vafl,zhang2021secure,li2020efficient,li2023fedvs}. 
Research on VFL network dynamics during inference is limited. 
\citet{ceballos2020splitnn} noted performance drops due to random client failures during testing, and \citet{sun2023robust} studied passive parties dropping off randomly during inference and proposed Party-wise Dropout (PD). 
Thus, prior dynamic network resilient VFL methods have majorly focused on train-time faults and assumed a special node (server or active party) that is immune to failure. 
As VFL differs significantly from the horizontal FL (HFL) setting \citep{yang2019federated}, where clients share the same set of features but have different samples, HFL methods for handling faults (e.g., adaptive aggregation of different models \citet{ruan2021towards}) are inapplicable in our scenario.
Additionally, unlike HFL, where faults only affect training, faults in VFL can disrupt both training \emph{and} inference due to the need for client communication during inference.

\paragraph{Decentralized FL}
Conventional CL methods such as FL, rely on a central server, creating a single point of failure.
To address such limitations, Decentralized FL has  been considered \citep{yuan2023decentralized}. 
Unlike, the extensively studied HFL decentralized methods (\citet{tang2022gossipfl,lalitha2019peer,feng2021blockchain,gabrielli2023survey}),  VFL decentralized methods are limited.
For the special case of simple linear models, \citet{he2018cola} proposed the decentralized algorithm COLA. 
For more general split-NN models, \citet{valdeira2023multi,valdeira2024vertical} proposed decentralized STCD, Less-VFL and semi-decentralized MTCD methods. However, these methods, including COLA and STCD/MTCD, do not analyze network dynamics like faults during inference.
While Laser-VFL addresses missing values during training and inference (a special case of device faults and a complete, non-faulty communication graph), it doesn't consider robustness to faults in the broader cross-device context.

\section{Problem Formulation}\label{Sec ILFormulation}
We define a novel formulation, \emph{\contextAbv}, which specifies both an operating context and the properties of a learning system. 
The desired property of the system is robustness under dynamic conditions, which we formally define via \emph{\risk} and corresponding metrics in the next two subsections.

\paragraph{Notation}
Let $(\xrv\in \mathcal{X},\yrv \in \mathcal{Y})$ denote the random variables corresponding to the input features and target, respectively, whose joint distribution is $p(X,Y)$.
With a slight abuse of notation, we will use $\mathcal{Y}$ to denote one-hot encoded class labels and probability vectors (for predictions).
Let $\dataset=\{(\xinst_i, \yinst_i)\}_{i=1}^n$ denote a training dataset of of $n$ samples i.i.d. samples from $p(X,Y)$ with $d$ input features $\xinst_i$ and corresponding target $\yinst_i$.
Let $\xinst_\mathcal{S}$ denote the subvector associated with the indices in $\mathcal{S} \subseteq \{1,2,\cdots,d\}$, e.g., if $\mathcal{S} = \{1,5,8\}$, then $\xinst_\mathcal{S} = [x_1, x_5, x_8]^T$.
For $C$ clients, the dataset at each client $c \in \{1,2,\cdots,C\}$ will be denoted by $\dataset_c$.
Let $\mathcal{G} = (\mathcal{C}, \mathcal{E})$ denote a network (or graph) of clients, where $\mathcal{C} \subseteq \{1,2,\cdots,C\} \cup \{0\}$ denotes the clients plus an entity (possibly an external entity or one of the clients itself) that represents the device that collects the final prediction and has the labels during training (further details in \Cref{sec:DNVL-inference}) and $\mathcal{E}$ denotes the communication edges.

\subsection{\contextName Context}\label{Sec ILDefinition}

\paragraph{Data and Network Context}\label{def:partial-features}
A partial features data context means that each client has access to a subset of the features, i.e.,
        $\dataset_c = \{\xinst_{i,\mathcal{S}_c}\}_{i=1}^n,$
where $\mathcal{S}_c \subset \{1,2,\cdots,d\}$ for each client $c$.
This is the same \emph{data context} as SplitVFL \citep{liu2022vertical}, a variant of VFL, which incorporates the idea of split learning \citep{vepakomma2018split}, and jointly trains models at both server and clients.
Furthermore, we do allow for scenarios where the clients can have features among one another that are correlated. 
For instance, there can be two sensors that can have correlated features due to their physical proximity.
Through the rest of the paper, the terms clients and devices are used interchangeably.
  
\begin{definition}[Dynamic Network Context]
\label{def:dynamic-network}
    A dynamic network means that the communication graph can change across time indexed by $t$, i.e., $\mathcal{G}(t) = (\mathcal{C}(t), \mathcal{E}(t))$, where the changes over time can be either deterministic or stochastic functions of $t$.
\end{definition}

This dynamic network context includes many possible scenarios including various network topologies, clients joining or leaving the network, and communication being limited or intermittent due to power constraints or physical connection interference.
We provide two concrete dynamic models where there are device failures or communication failures.
For simplicity, we will assume there is a base network topology $\mathcal{G}_\mathrm{base}=(\mathcal{C}_\mathrm{base}, \mathcal{E}_\mathrm{base})$ (e.g., complete graph, grid graph or preferential-attachment graph), and we will assume a discrete-time version of a dynamic network where $t\in \{0,1,2,\cdots\}$, which designates a synchronous communication round.
Given this, we can formally define two simple dynamic network models that encode random device and communication faults.
\begin{definition}[Device Fault Dynamic Network]
    \label{def:device-fault-network}
    Given a fault rate $r$ and a baseline topology $\mathcal{G}_\mathrm{base}$, a \emph{device fault dynamic network} $\mathcal{G}_r(t)$ means that a client is in the network at time $t$ with probability $1-r$, i.e., $\mathrm{Pr}(c \in \mathcal{C}_r(t)) = 1-r, \forall c \in \mathcal{C}_\mathrm{base}$ and $\mathcal{E}_r(t)=\{(c,c') \in \mathcal{E}_\mathrm{base}: c,c' \in \mathcal{C}_r(t)\}$.
\end{definition}
\begin{definition}[Communication Fault Dynamic Network]
    \label{def:comm-fault-network}
    Given a fault rate $r$ and a baseline topology $\mathcal{G}_\mathrm{base}$, a \emph{communication fault dynamic network} $\mathcal{G}_r^\mathrm{CF}(t)$ means that a communication edge (excluding self-communication) is in the network at time $t$ with probability $1-r$, i.e., $\mathcal{C}_r^\mathrm{CF}(t) = \mathcal{C}_\mathrm{base}$ and $\mathrm{Pr}((c,c') \in \mathcal{E}_r^\mathrm{CF}(t)) = 1-r, \forall (c,c') \in \mathcal{E}_\mathrm{base}$ where $c\neq c'$.
\end{definition}

\subsection{\contextAbv Problem Formulation via Dynamic Risk}
\label{sec:DNVL-inference}
Given these context definitions, we now define the goal of \contextAbv in terms of the \risk which we define next.
For now, we will assume the existence of a distributed inference algorithm $\Psi(\xinst; \theta, \mathcal{G}(t)): \mathcal{X} \to \mathcal{Y}^C$, where each client makes a prediction across the data-split network under the dynamic conditions given by $\mathcal{G}(t)$.
Furthermore, we will use a (possibly stochastic) post-processing function $h: \mathcal{Y}^C \to \mathcal{Y}$ to model the final communication round between the clients and another entity (which may be external or may be one of the clients), which owns the labels for training and collects the final prediction during the test time.
The $h$ can model different scenarios including where the entity has access to all or only a single client's predictions.
As an example, the entity could represent a drone passing over a remote sensing network to gather predictions or a physical connection to the devices at test time (e.g., when the sensors are ultra-low power and cannot directly connect to the internet).
Or, this entity could represent a power intensive connection via satellite to some base station that would only activate when requested to save power.

\begin{definition}[\risk]
\label{def:risk}
                Assuming the partial features data context (\autoref{def:partial-features}) and given a dynamic network $\mathcal{G}(t)$ (\cref{def:dynamic-network}), the \emph{\risk} is defined as:
$
            R_h(\theta;\mathcal{G}(t)) \triangleq \mathbb{E}_{\xrv,\yrv,\mathcal{G}(t),h}[\ell_h(\Psi(\xrv; \theta, \mathcal{G}(t)), \yrv)] \,,
        $
    where $\Psi: \mathcal{X} \to \mathcal{Y}^C$ is a distributed inference algorithm parameterized by $\theta$ that outputs one prediction for each client and $\ell_h(\hat{\bm{y}},y) \triangleq \ell(h(\hat{\bm{y}}; \mathcal{G}(T)),y)$ is a loss function where $h:\mathcal{Y}^C \to \mathcal{Y}$ (which could be stochastic) post-processes the client-specific outputs to create a single output based on the communication graph at the final inference time $T$, and $\ell$ could be any standard loss function.
\end{definition}
This risk modifies the usual risk by taking an expectation w.r.t. the dynamic graph (which could be stochastic over time) and a the client selection function $h$ (described next).
We assume that the distributed inference algorithm produces a prediction for every client and the $h$ function (stochastically) selects the final output (note how the composition is a normal prediction function, i.e., $h \circ \Psi: \mathcal{X} \to \mathcal{Y}$).
This means that the network's parameters $\theta$ are distributed across all clients.
We also note that the model at each client could have different parameters and even different architectures, unlike in HFL.

We will now define the client selection function, $h$, which represents the communication to the external entity.
Each metric represents a distinct real-world context.
We consider two practical scenarios and two oracle methods that depend on how $h$ selects the final output of the distributed inference algorithm.
These four methods for defining $h$ will represent \risk{s} under different scenarios and form the basis for the test metrics used in the experiments.
We first formally define an active set $\mathcal{A}$ of clients at the last communication round as 
$\mathcal{A}(\mathcal{G}(T)) \triangleq \{c: (0,c) \in \mathcal{E}(T), c \in \mathcal{C}(T) \}$, which means the devices that could communicate to the special entity denoted by $0$ at the last communication round (other devices are not able to communicate their predictions).
\textbf{\average}($h_\mathrm{\averageAbv}(\hat{\bm{y}})$): randomly selects \emph{active} client; \textbf{Select Oracle Best Active Client}($h_\mathrm{best}(\hat{\bm{y}})$): if any active client prediction is correct, the correct label is predicted; \textbf{Select Oracle Worst Active Client}($h_\mathrm{worst}(\hat{\bm{y}})$): if any active
client prediction is incorrect, the wrong label is predicted; \textbf{Select Any Client}($h_\mathrm{any}(\hat{\bm{y}})$): randomly select any device, from all devices
both active and inactive. A detailed definition of these functions can be found in Appendix.

\section{\methodName (\methodAbv)}\label{sec:methods}
\label{sec:distributed-inference}
Given the novel \contextAbv context, we now propose our minimalistic yet impactful message passing distributed inference algorithm \methodSf for \contextAbv and present the relevant theoretical insights.
\textbf{First}, we extend and discuss dropout methods for simulating faults during training to enhance the robustness of the network with faults at test time.
\textbf{Second}, we overcome the problem where CL catastrophic fails if the single aggregator node faults by enabling multiple clients to be data aggregators.
\textbf{Third}, we improve both the ML performance and decrease the variability of client-specific predictions by using gossip rounds to average the final output across devices.
We assume that the aggregator of neighbor representations is simply concatenation and the network architecture are based on simple multi-layer perceptrons (MLP). 
We summarize the different proposed techniques and contrast them with VFL style CL  in \Cref{IL BaseLines} (in Appendix) and present the \methodAbv inference algorithm in \Cref{Algorithm:DN_VFL}.

\newcommand{\xvec}{\bm{x}}
\newcommand{\zvec}{\bm{z}}
\newcommand{\yvec}{\hat{y}}
\newcommand{\Bvec}{\bm{B}}

\newcommand{\note}[1]{\hfill\{#1\}}

\begin{figure*}[t] \centering
    \begin{minipage}{0.9\textwidth} \begin{algorithm}[H] \caption{\methodSf Inference Algorithm}
            \label{Algorithm:DN_VFL}
            \begin{algorithmic}[1]
                \State \textbf{Input:} Input features $\{\xvec_c\}_{c=1}^C$, parameters $\{\theta_{c}^{(t)}: \forall c, t\}$, and dynamic graph $\mathcal{G}(t) = (\mathcal{C}(t), \mathcal{E}(t))$
                \State $\zvec_c^{(0)} = f_c^{(0)}(\xvec_c; \theta_c^{(0)}), \quad \forall c \in \mathcal{C}(0)$ \Comment{Process input at all clients}
                \State $\tilde{\zvec}_k^{(1)} = g(\{\zvec_{c'}^{(0)} : (k,c') \in \mathcal{E}(1)\}), \quad \forall k \in \mathcal{K} \cap \mathcal{C}(1)$ \Comment{Aggregate messages from neighbors: \cref{sec:MulAgg}}
                \State $\zvec_k^{(1)} = f_k^{(1)}\big(\tilde{\zvec}_k^{(1)}; \theta_k^{(1)}\big), \quad \forall k \in \mathcal{K} \cap \mathcal{C}(1)$ \Comment{Apply prediction function to aggregated output}
                \For{$t \gets 2, \dots, G+1$} \Comment{Gossip prediction probabilities to neighbors: \cref{sec:gossip-layers}}
                    \State $\zvec_k^{(t)} = \mathrm{Avg}(\{\zvec_{k'}^{(t-1)} : (k,k') \in \mathcal{E}(t), k' \in \mathcal{K}\cap \mathcal{C}(t) \}), \quad \forall k \in \mathcal{K} \cap \mathcal{C}(t)$
                \EndFor
                \State \textbf{return} $\{\zvec_k^{(G+1)} \in \mathcal{Y}\}_{k \in \mathcal{K}}$\Comment{Return all aggregator-specific predictions}
            \end{algorithmic}
        \end{algorithm}
    \end{minipage}
\end{figure*}

\subsection{Decentralized Training of MAGS with Real and Simulated Faults via Dropout}
\label{sec:dropout}
To train MAGS, we use a standard VFL backpropagation algorithm \emph{without gossip rounds.}\footnote{By training without gossip, the classifier head on each device is trained independently to maximize its own performance so that its errors are uncorrelated with other devices when using gossiping as an ensembling approach as discussed in \Cref{sec:gossip-layers}}The objective function is the sum of the losses at each device.
Faults can be treated similarly to dropout, where missing values are imputed with zeros (see Appendix for more details).
Our training algorithm assumes all devices have access to the labels, which is similar to an assumption made in \citet{castiglia2022compressed} and is valid for our setup, involving a trusted but unreliable device network, where robustness is our primary goal.
Training with no faults may leave the model vulnerable to higher fault rates during inference, which can occur due to external factors like extreme weather. This presents a challenge, as large-scale inference-time faults lead to missing values, causing a distribution shift between the training and test data. Such shifts can severely degrade model performance, as noted by \citet{koh2021wilds}.

To prepare for such scenarios, we simulate inference-time faults during training by adding dropout. While \citet{sun2023robust} proposed Party-wise Dropout (PD) for a server-based setting, DN-CL requires dropout across client-to-client links. We therefore introduce Communication-wise Dropout (CD) to represent communication failures within a decentralized network.
This extra dropout helps the model adapt to higher fault rates by leveraging its regularizing effect \citet{baldi2013understanding}.
\citet{mianjy2020convergence} further showed that a model trained with dropout and tested without it can achieve near-optimal test performance in $O(1/\epsilon)$ iterations. Their work also provides evidence that, in over-parameterized models, dropout-regularized networks can generalize well even when dropout is applied during testing---exactly what is needed in \contextAbv

\subsection{Multiple Aggregators CL (MACL)}\label{sec:MulAgg}

A common vulnerability in collaborative learning (e.g., VFL) is having a single server as the only aggregator, creating a single point of failure if it goes down. 
We mitigate this by turning all clients into aggregators (\Cref{Algorithm:DN_VFL} line 3), introducing redundancy and fault-tolerance, which we term Multiple Aggregators Collaborative Learning (MACL).
An MACL setup can tolerate the failure of any node and the probability of failure of all nodes is given by $r^C$, which is very small if $C$ is large.
However, having all nodes act as aggregators could increase the communication cost.
Thus, we develop $K$-MACL as a low communication cost alternative to MACL.
In $K$-MACL, we assume there is a set of clients $\mathcal{K} \subseteq \mathcal{C}$ that act as data aggregators.
The number of aggregators ($K\triangleq|\mathcal{K}|$) will generally be less than the total number of devices, resulting in lower communication cost than MACL.
We now theoretically prove a bound on the risk that critically depends on the probability of catastrophic failure, i.e., when there are no active aggregators $|\mathcal{A}| = 0$.

\begin{proposition}
\label{prop:k-MVFL}
Given a device fault rate $r$, the number of data aggregators $K\leq C$ and the post-processing function $h_\mathrm{\averageAbv}$, and assuming the risk of a predictor (data aggregator) with faults is higher than that without faults, 
then the dynamic risk with faults is lower bounded by:
\begin{align}
\underbrace{R_h({\theta;\mathcal{G}_{r}(t)})}_\text{Risk with faults}
&\geq (1-r^K) \cdot \underbrace{R_h(\theta;\mathcal{G}_{\mathrm{base}})}_{\text{Risk without faults}} \,+\,\,
\underbrace{r^K}_{\mathrm{Pr}(|\mathcal{A}|=0)} \cdot \underbrace{\mathbb{E}[\ell(Y,p(Y))]}_{\text{Risk of random predictor}}\,.
\label{eq_prop1_2} 
\end{align}
\end{proposition}
 
Proof is in the appendix. 
As a simple application, suppose that the fault rate is very high at $r=0.3$, a VFL system is likely to fail 30\% of the time due to server failure and the dynamic risk would reduce to random guessing.
However, with just four aggregators, the chance of failure reduces to less than 1\%.
While having multiple aggregators addresses the fundamental problem of catastrophic failures, each model is often insufficient given only one communication round especially for sparse base graphs or high fault rates. 
Additionally, each device may have widely varying performance characteristics due to its local neighborhood. 
Thus, further enhancements are needed for robustness and stability.

\subsection{Gossip Layers to Ensemble Aggregator Predictions}
\label{sec:gossip-layers}
While multiple data aggregators help avoid system-level failures, the performance of each data aggregator may be poor due to faults, which could result in overall high dynamic risk. 
Because gossip is not used during training, each data aggregator model is different. 
During inference, each will have access to different client representations (due to topology and test-time faults).
This variability among data aggregators translates to inconsistent performance when viewed by an external entity as it depends on which device is selected and the best device may differ for each inference query (see the difference between ``Active Worst'' and ``Active Best'' metrics).
Thus, we propose to use gossip layers to combine predictions among data aggregators (\Cref{Algorithm:DN_VFL}, lines 5 and 6).
From one perspective, gossip implicitly produces an ensemble prediction at each aggregator, which we prove always has better or equal dynamic risk.
From another perspective, gossiping will cause the aggregator predictions to converge to the same prediction---which means that the system performance will be the same regardless of which device is selected.

We analyze the ML performance of gossip using the ensemble diversity framework from \citet{wood2023unified}, which links diversity to reduced risk.  
They showed that ensemble risk is decomposable into individual risks minus a diversity term (which will reduce the risk if positive).
We leverage this theory to prove that the dynamic risk of our ensemble is equal to the non-ensemble dynamic risk minus a diversity term---which is always non-negative and positive if there is any diversity in prediction.
This proposition shows that gossiping at inference time, which implicitly creates ensembles, will almost always improve the dynamic risk compared to not using gossip. (Proof is in appendix.)

\begin{proposition}
\label{prop:gossip_Ensembling}
The dynamic risk of an ensemble over aggregators is equal to the non-ensemble risk minus a non-negative diversity term:
\begin{align*}
    R^\mathrm{ens}_h(\theta; \mathcal{G}(t)) 
    &= \textstyle \underbrace{R_h(\theta; \mathcal{G}(t))}_\text{Non-ensemble risk} -\, \underbrace{\mathbb{E}_{\bm{x},\mathcal{G}(t),h}[\textstyle \frac{1}{K}\sum_{k=1}^K \ell_h(\Psi_{k}(\bm{x}; \theta), \Psi^\mathrm{ens}(\bm{x}; \theta)) ]}_\text{Diversity term (non-negative)}
    \leq R_h(\theta; \mathcal{G}(t)) \,,
\end{align*}
where $R^\mathrm{ens}_h$ is the ensemble risk; 
$\Psi_k$ is the $k$-th aggregator model;
$\Psi_\mathrm{ens}$ is the ensemble model where $\Psi^\mathrm{ens}_{k}(\bm{x}; \theta, \mathcal{G}(t)) \triangleq Z^{-1}\exp(\sum_{k'=1}^K \Psi_{k'}(\bm{x}; \theta, \mathcal{G}(t))), \forall k \in \mathcal{K}$, where $Z$ is the normalizing constant;
and where the notational dependence of $\Psi$ on $\mathcal{G}(t)$ is suppressed and $\ell_h$ in the diversity term applies $h$ to both loss arguments with a slight abuse of notation.
\end{proposition}

To analyze output variability using gossip, we turn to gossip consensus results.
We first introduce some additional notation.
Let $A$ be the adjacency matrix of graph $\mathcal{G}=\{\mathcal{C}, \mathcal{E}\}$ and let $V=D^{-1}A$ denote the consensus matrix where $D$ is the degree matrix (note that $V$ is row stochastic) and let $\lambda$ denote the largest eigenvalue of $V - \frac{11^T}{C}$, also known as the spectral radius.
The result below proves that with simple averaging the variability decreases exponentially with increasing gossip rounds based on the spectral radius of the (faulted) graph.
\begin{proposition}
\label{prop:gossip}
If simple averaging is used during gossip, the difference between the average output over all devices, denoted $\bar{y}$, and the original output of the $i$-th device, denoted $y_i$, after $G$ gossip rounds is bounded as follows:
$\|\bar{y} - y_i\|_2 \leq \lambda^G \sqrt{C} \max_{j,j' \in \mathcal{C}} \| y_j - y_{j'}\|_2, \forall c \in \mathcal{C}$.
\end{proposition}

\Cref{prop:gossip} follows as a special case of the proof in \citet{lin2021semi}.
Thus, assuming the graph is connected (i.e., $\lambda<1$), the variability between aggregators shrinks to zero exponentially w.r.t. the number of gossip rounds $G$ based on the spectral radius $\lambda$.
Intuitively, the spectral radius is small for dense graphs and large for sparse graphs.
As a consequence, test-time faults will make the spectral radius increase.
However, as long as the graph is still connected, this gossip protocol can significantly reduce device variability even after only a few gossip rounds.

\section{Experiments}\label{sec:experiment}

\label{sec:experiment-setup}

\paragraph{Datasets} We test on diverse datasets: MNIST, CIFAR10, CIFAR100, Tiny ImageNet, and StarCraftMNIST \citep{kulinski2023starcraftimage}, a spatial reasoning dataset based on StarCraft II replays. 
Due to space constraints, here we focus on StarCraftMNIST as it aligns with our geospatial sensor network use-case (\cref{sec:IntoAndMot}).
To simulate this, we split images into patches, with one client per patch (mainly 16 clients in a 4x4 grid).
See \Cref{app-sec:experiment-result} for results on other datasets and client numbers.

\paragraph{Method} Baseline methods are vanilla \emph{VFL} and \emph{PD-VFL} \citep{sun2023robust}.
We then include various ablation versions of our \methodSf to show the importance of each component to robust \contextAbv performance.
Specifically, \emph{MACL} refers to using all clients as aggregators.  \emph{4-MACL} refers to the low communication version of MACL, where 4 was chosen based on \Cref{prop:k-MVFL}.
The prefix of \emph{PD-} or \emph{CD-} refers to using party-wise or communication-wise dropout during training. And the suffix of \emph{-G}$g$ denotes that $g$ gossip rounds were used.
See Appendix for specifics about model architecture and hyperparameters.

\paragraph{Baseline Communication Network} We consider diverse graph types with varying sparsity: complete, grid, ring, and random geometric graphs (where devices within a given distance connect). We assume synchronous communication, following standard FL/VFL works (e.g., \citet{mcmahan2017communication, wang2022communication, crawshaw2024federated,li2023fedvs,jiang2022vf}).

\paragraph{Fault Models} We compute dynamic risk under both device and communication faults (\Cref{def:device-fault-network} and \Cref{def:comm-fault-network}), with fault rates up to 50\%, demonstrating performance under extreme scenarios.
Results shown here use a constant faulted graph during inference; see Appendix for temporally varying faults.

\paragraph{Different Test Fault Rates and Patterns}
As seen in \Cref{fig:testavg_trainfault_faultrate_starmnist_16}, across multiple test fault rates, fault types, and baseline networks, the performance of most approaches degrades significantly from about 80\% to 30\% while our proposed MAGS methods (PD-MACL-G4 and CD-MACL-G4) are relatively resilient to the increasing fault percentage.
By comparing MACL to VFL, it appears that using multiple data aggregators improves resilience to faults.
This observation is in line with \cref{prop:k-MVFL} and substantiates the benefit of having multiple data aggregators to deal with \contextAbv.
Furthermore, dropout during training leads to improved resilience. MACL models trained using PD or CD are more robust than ones without it.
Such results provide empirical evidence to support the claim that simulating training fault via dropout is a valuable technique to handle inference faults. 

\begin{figure*}[!ht]
    \centering
        \includegraphics[width=0.99\textwidth]{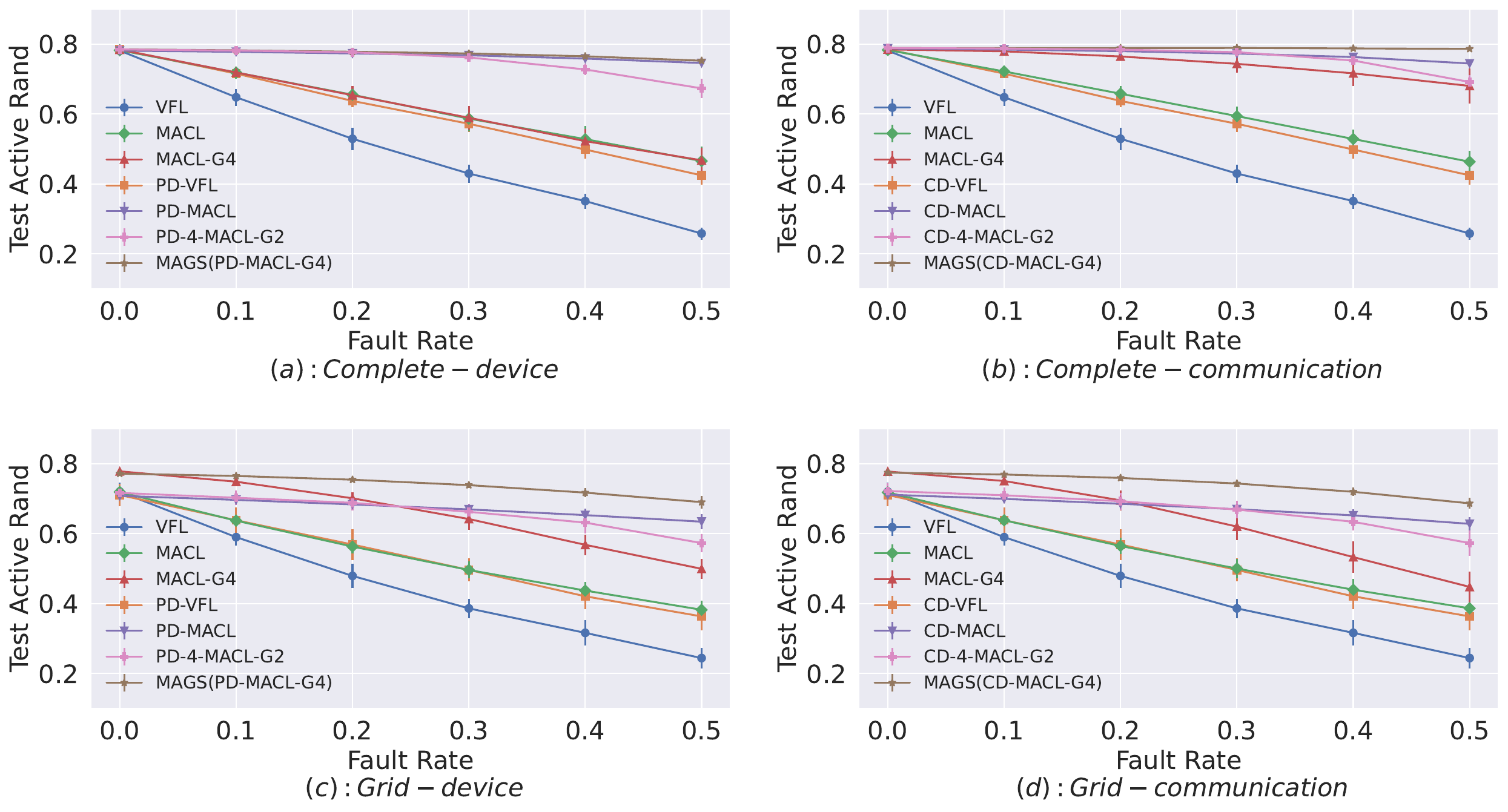}
        \vspace{-1em}
    \caption{Test accuracy with and without communication (CD-) and party-wise (PD-) Dropout method  for StarCraftMNIST with 16 devices. 
    Here we include models trained under an dropout rate of  30\% (marked by 'PD-' or 'CD-'). 
    All results are averaged over 16 runs, and the error bar represents standard deviation. Across different configurations, MAGS(PD/CD-MACL-G4) trained with feature omissions has the highest average performance, while vanilla VFL performance is not robust as fault rate increases.
    As our experiments are repeated multiple times, what we report is the expectation (Avg) over the random active client selection.
    }
    \label{fig:testavg_trainfault_faultrate_starmnist_16}
\end{figure*}

The gossip variants of the proposed methods provide a performance boost when combined with PD and CD variants across different fault rates. 
This underscores the importance of using gossiping as a part of \methodSf.
The performance of CD/PD-4-MACL-G2 in \cref{fig:testavg_trainfault_faultrate_starmnist_16}, indicates that better robustness to inference faults than VFL can be achieved at a much lower communication cost than that of MACL.
In summary, the combined effect of having a decentralized setup, gossiping and dropout clearly outperforms other methods.
In the Appendix we present some more investigation (Ablation Study) on number of gossip rounds and different dropout rates for CD and PD.
In addition, an extended version of \cref{fig:testavg_trainfault_faultrate_starmnist_16} can be found in the Appendix.

\paragraph{Communication and Performance Analysis} 
To study the trade-off between communication and accuracy, \cref{tab:table_network} (extended version in Appendix) presents the performance and approximated number of communication (\# Comm.) for different baseline networks. 
Across varying levels of graph sparseness, using a decentralized setup with gossiping improves performance by at least 10 percent points and can be as high as 32 percent points when compared to VFL.
This performance gain comes at the cost of increased communication, enabling system redundancy.
Nonetheless, 4-MACL achieves major VFL improvements with just four aggregating devices.
Comparing 4-MACL and MACL shows an efficient robustness-communication trade-off. The optimal number of aggregators depends on context and acceptable performance loss.
Furthermore, 4-MACL with a poorly connected graph is still better than VFL with well connected graph, such as 4-MACL with \textit{RGG} (r=1) versus VFL with \textit{Complete}, where the number of communications are similar in magnitude.
This indicates that given a fixed communication budget, VFL may not be the best solution despite its low communication cost.
Finally, gossip's impact lessens in sparser networks, indicating more communication doesn't always improve performance in dynamic environments

\begin{table*}[!ht]
\centering
        \caption{
        Active Rand (Avg) performance at test time with 30$\%$ communication fault rate.
Compared to VFL, MACL performs better but it comes at higher communication cost. Thus we propose 4-MACL, which is shown to be a low communication cost alternative to MACL. 
}
    \label{tab:table_network}
\resizebox{1\columnwidth}{!}{
    \begin{tabular}{l|ll|ll|ll|ll|ll|ll}
\hline
        & \multicolumn{2}{c|}{Complete} & \multicolumn{2}{c|}{\begin{tabular}[c]{@{}c@{}}RGG\\ r=2.5\end{tabular}} & \multicolumn{2}{c|}{\begin{tabular}[c]{@{}c@{}}RGG\\ r=2\end{tabular}} & \multicolumn{2}{c|}{\begin{tabular}[c]{@{}c@{}}RGG\\ r=1.5\end{tabular}} & \multicolumn{2}{c|}{\begin{tabular}[c]{@{}c@{}}RGG\\ r=1\end{tabular}} & \multicolumn{2}{c}{Ring} \\ \hline
        & Avg          & \# Comm.           & Avg                                & \# Comm.                                & Avg                               & \# Comm.                               & Avg                                & \# Comm.                                & Avg                               & \# Comm.                               & Avg         &\#  Comm.        \\ \hline
VFL     &              0.430& 10.6           &                                    0.406& 7.4&                                   0.407& 5.2&                                    0.375& 3.5&                                   0.386& 2.0                                &             0.385&             1.4\\ \cdashline{1-13}
4-MACL    &              0.591	& 42&                                    0.572& 29&                                   0.555& 20.4&                                    0.517& 14.8                                &                                   0.488& 7.98&             0.485&             5.6\\
4-MACL-G2  &              0.687& 126&                                    	0.661& 87&                                   0.623& 61.2&                                    0.566	& 44.8&                                   0.491& 23.94&             0.484&            16.8 \\
MACL    &              0.594	& 168.5&                                    0.581& 114.9&                                   0.558& 80.8&                                    0.528& 58.7                                &                                   0.503& 33.5&             0.507&             22.7\\
MACL-G4    &              0.732	& 836.2&                                    0.728& 572.1&                                   0.721& 407.2&                                    0.689& 293.9                                &                                   0.62& 168.2&             0.558&             113.4
\end{tabular}}
\end{table*}

\paragraph{Best, Worst and Select Any Metrics} 
To better understand our methods, particularly the gossip aggregations, we show the results for all four metrics on 3 datasets for 50$\%$ communication fault rate on a complete network in \Cref{tab:best-models-complete-comm}.
In the Appendix, \Cref{tab:best-models-complete-comm} also includes 30$\%$ communication fault rate.  
The dropout rate during training for PD-VFL and the communication dropout rate for CD-MACL is the same at $30\%$.
Results align with our theoretical analysis, supporting gossip's benefits for ML performance and reduced client variability (evidenced by the Best/Worst oracle metric gap).
Furthermore, communication-wise dropout improves the performance of each client individually.
This enables MAGS(CD-MACL-G4) to match or significantly outperform all other approaches across the two fault rates.
Finally, Any Rand is the most challenging metric because some devices may fail, leading to random predictions.

\begin{table*}[ht!]
\centering
    \caption{Best models for 50$\%$ \textit{complete-communication} test fault rates within 1 standard deviation are bolded. More detailed results with standard deviation are shown in the Appendix.}
\vspace{0em}
    \label{tab:best-models-complete-comm}
    \resizebox{1\columnwidth}{!}{\begin{tabular}{llllll|llll|lllll}
\toprule
                                  &          & \multicolumn{4}{c}{MNIST}    & \multicolumn{4}{c}{SCMNIST}  & \multicolumn{4}{c}{CIFAR10}   &  \\ 
                                  \cmidrule(lr){3-6}
                                  \cmidrule(lr){7-10}
                                  \cmidrule(lr){11-14}
                                 &          & 
                                 \multicolumn{3}{c}{Active}    & Any &
                                 \multicolumn{3}{c}{Active}    & Any &
                                 \multicolumn{3}{c}{Active}    & Any 
                                 \\
                                 \cmidrule(lr){3-5}
                                 \cmidrule(lr){6-6}
                                 \cmidrule(lr){7-9}
                                 \cmidrule(lr){10-10}
                                 \cmidrule(lr){11-13}
                                 \cmidrule(lr){14-14}
                                 &          & 
                                 Worst   & Rand  & Best  & Rand &
                                 Worst   & Rand  & Best  & Rand &
                                 Worst   & Rand  & Best  & Rand &
                                 \\
\midrule
\multirow{5}{3em}{Test Fault Rate = 0.5} & VFL      & nan & 0.294   & nan   & nan   & nan & 0.258   & nan   & nan   & nan & 0.181   & nan   & nan   &  \\  
                                  & PD-VFL      & nan & 0.488   & nan   & nan   & nan & 0.424   & nan   & nan   & nan & 0.263   & nan   & nan   &  \\
                                  & 4-MACL-G2      & 0.564 & 0.612   & 0.721   & 0.423   & 0.466 & 0.519   & 0.620   & 0.368   & 0.251 & 0.303   & 0.387   & 0.228   &  \\
                                  & MACL     & 0.042 & 0.518 & 0.966 & 0.313 & 0.035 & 0.465 & \textbf{0.925}& 0.280 & 0.007 & 0.264 & \textbf{0.762}& 0.182 &  \\
                                  & MACL-G4  & 0.843 & 0.847 & 0.851 & 0.474 & 0.676 & 0.680 & 0.684 & 0.389 & 0.402 & 0.401 & 0.413 & 0.252 &  \\
                                  & CD-4-MACL-G2  & 0.863 & 0.852 & 0.923 & \textbf{0.581} & 0.693 & 0.691 & 0.763 & \textbf{0.482} & 0.392 & 0.422& 0.499 & 0.305 &  \\
                                  & MAGS(CD-MACL-G4) & \textbf{0.974} & \textbf{0.975} & \textbf{0.976} & 0.538 & \textbf{0.785} & \textbf{0.786} & 0.787 & 0.443 & \textbf{0.501} & \textbf{0.504} & 0.507 & \textbf{0.301} & \\
                                  \bottomrule
\end{tabular}
}
\vspace{-0.7em}
\end{table*}

\section{Discussion and Conclusion} 
\label{sec:discussion}

In this paper, our focus was on the machine learning aspects of decentralized CL, with a level of analysis (e.g., communication cost, device simulation) comparable to other studies in HFL \citep{mcmahan2017communication, wang2022communication, crawshaw2024federated} and VFL \citep{li2023fedvs, castiglia2023flexible, jin2021cafe, jiang2022vf}.
MAGS aims for robust performance under extreme faults in safety-critical, machine-generated IoT/Sensor Network applications (e.g., search and rescue), where the network must operate even with near-catastrophic failures (e.g., 50\% device loss).  Thus, privacy is not the primary objective for these applications.
While we briefly discuss some system-level aspects in the Appendix, real-world deployment necessitates further research.
Additionally, we assume synchronous communication, but asynchronous models could be a future extension for \contextAbv context. 
The Appendix also compares MAGS to fault-tolerant consensus algorithms like Paxos or Raft \citep{lamport2001paxos,ongaro2014search}, showing their insufficiency for \contextAbv .
While our focus is robustness,  orthogonal privacy-preserving approaches (e.g., blockchain-based\citep{li2023fedvs}, homomorphic encryption\citep{tran2024differentially}) can be integrated with MAGS.

Further research could explore the impact of faults during training. 
While our method can handle faults in the backward propagation phase (see Appendix), a thorough analysis, especially with high fault rates during training, remains open.  
Here, we focused on near-catastrophic faults at \emph{inference}. 
Additionally, while we prioritized robustness, advanced architectures (e.g., CNNs, transformers) could improve absolute performance. 
We believe MAGS is adaptable to various architectures, offering an orthogonal contribution to architecture design.

In summary, we carefully defined \contextAbv, proposed and theoretically analyzed \methodSf as a method for \contextAbv, developed a testbed, and evaluated and compared performance across various fault models and datasets.
Simulated faults via dropout increase the robustness of MAGS to distribution shifts.
MACL allows MAGS to avoid catastrophic faults since any device (including active parties) can fail.
Gossiping outputs at inference time implicitly ensembles the predictions for neighboring devices that theoretically improve the robustness and reduce the variance.
Our work lays the foundation for \contextAbv, opening up many directions for future research, such as handling heterogeneous devices or models, exploring new architectures, and considering different fault models.

\section*{Acknowledgments}
S.G., Z.Z., C.B., and D.I. acknowledge support from ONR (N00014-23-C-1016).  S.G., Z.Z., and D.I. acknowledge support from NSF (IIS-2212097) and ARL (W911NF-2020-221).

\clearpage

\bibliography{references}
\bibliographystyle{plainnat}

\newpage
\appendix
\onecolumn

\addcontentsline{toc}{section}{Appendix} \part{Appendix} \parttoc \title{Appendix}
\author{John Doe}
\maketitle

In the main paper, the following information were listed to be provided in the Appendix section and they can be found in the Appendix per the listing below:
\begin{enumerate}
    \item More detailed version of the related works from \cref{sec:related-works} can be found in \cref{sec:related-works-Appendix}
    \item Proofs for the propositions from \cref{sec:methods} can be found in Section \ref{sec:proofs}
    \item Relation to GNN from Section \ref{sec:distributed-inference} can be found in Section \ref{disc_Lim}
    \item Analysis for different number of devices/clients from Section \ref{sec:experiment-setup} can be found in Section \ref{subsec:DifDevices}
    \item Specifics on model architecture from Section \ref{sec:experiment-setup} can be found in Section \ref{Training}
    \item Further analysis on gossiping methods from Section \ref{sec:experiment} can be found in Section \ref{subsec:AblationStudies}
    \item Grid graph construction details from Section \ref{sec:experiment-setup} can be found in Section \ref{Graph_det}
    \item Further study on different drop out rates for Community-wise and Party-wise drop out from \cref{sec:experiment} can be found in Section \ref{subsec:TrainingF}
\item Further details on metrics from Table \ref{tab:best-models-complete-comm} can be found in Section \ref{Subsec:OtherMetric}
    \item Extension of Discussion from \cref{sec:experiment} can be found in Section \ref{disc_Lim}
    \item Illustration of PD and CD for MVFL, as mentioned in \cref{sec:methods} can be found in \cref{CD_PD_MVFL}
    \item Extended form of Communication and Performance analysis table from \cref{sec:experiment} is shown in \cref{subsec:PerformanceCommTradeOff}
    \item Elaborate version of Different Test Fault Rates and Patterns results from \cref{sec:experiment} can be found at \cref{Subsec:Testing_Fault_Extension}
    \item Justification of using a particular choice of K in \cref{sec:experiment-setup} is presented in \cref{subsec:AblationStudies}

\end{enumerate}

\paragraph{Notation Clarification:} In the main section of the paper, we denote the idea of using multiple aggregators as MACL ( Multiple Aggregators CL).
However, in the Appendix, this same idea is captured by MVFL (Multiple VFL).
The reason for making this change was that VFL as a CL method emphasizes privacy, while we are more focused on the robustness aspects of CL. 
Thus, to avoid confusion due to the nomenclature, we changed MVFL to MACL in the main paper, but owing to a lack of time, we follow the MVFL nomenclature in the Appendix.

\section{Extension to Related works}
\label{sec:related-works-Appendix}

\paragraph{Network Dynamic Resilient  FL}
Network dynamics is an important consideration for FL. While this topic has received attention in HFL \citep{wang2022unified,gu2021fast,ruan2021towards,liu2021blockchain,chen2021towards,van2020asynchronous}, the methods are not directly applicable in VFL due to difference in context.
In VFL network dynamics has been predominantly studied during the training phase. \citet{chen2020vafl,zhang2021secure,li2020efficient} studied model training with asynchronous client participation. 
\citet{li2023fedvs} noted that asynchronous VFL performance degradation is due to staleness of model updates, and proposed FedVS.

Literature on VFL network dynamics during test/inference is sparse. 
\citet{ceballos2020splitnn} studied distributed vertical features and acknowledged the degradation in performance due to clients dropping randomly during testing
\citet{sun2023robust} studied clients dropping off randomly during testing and proposed Party-wise Dropout (PD). 
The server, referred to as the active party, was assumed not to fault either during testing or training. 
Other participating clients were called the passive party. 
The PD method relies on using a model that is trained with randomly omitting representations from passive parties, and this was shown to provide robustness only against the unexpected exit of passive parties during testing.
Thus, dynamic network resilient VFL methods studied thus far have predominantly focused on train-time faults and have assumed a special node (server or active party) that is immune to failure.

\paragraph{Decentralized FL}
Conventional FL uses a central server for aggregation. However, this approach results in server being the single point of failure and it being a communication bottleneck. 
To address these limitations, Decentralized FL has  been considered. 
In literature, this has also been designated as Decentralized FL, peer-to-peer FL, server free FL, device-to-device FL, etc\citep{yuan2023decentralized}. 
HFL Decentralized methods have been studied quiet extensively \citet{tang2022gossipfl,lalitha2019peer,feng2021blockchain}.
Although gossip and peer to peer learning based methods are viable options for Decentralized HFL, they are not, by themselves, considered reliable when network conditions are dynamic \citep{gabrielli2023survey}.
Nonetheless, we borrow some ideas from such protocols in our proposed methods. 

A decentralized algorithm, COLA, for training linear models over vertically split dataset was proposed by \citet{he2018cola}. 
Due to the setup, COLA was not applied to learn more general models.
\citet{valdeira2023multi} proposed decentralized (STCD) and semi-decentralized (MTCD) methods for VFL. 
While STCD, unlike MTCD, does not rely on a special server node, it's performance compared to other algorithms \citep{alghunaim2021dual} decreases rapidly with increase in network connectivity. 
Neither STCD nor COLA considers network dynamics during inference time.

\section{\methodSf Proposed Techniques}
\label{sec:proposed_techniques}
In \cref{IL BaseLines} we provide a visual representation of the techniques used in \methodSf. 
Furthermore, we provide a representation of VFL in \cref{IL BaseLines} and show how each of the techniques build upon the VFL setup.

\section{Proofs}
\label{sec:proofs}

\subsection{Proof of \Cref{prop:k-MVFL}}

Before we prove the proposition, we will prove the following lemma about the conditional probability of a client being selected given a known active set size.
\begin{lemma}[Conditional Client Selection Probability]
\label{thm:conditional-selection-prob}
    The conditional probability of selecting each data aggregator $k \in \mathcal{K} \cup\{\emptyset\}$ (aggregators plus the possible fake aggregator $k=\emptyset$) given a current active set size $|\mathcal{A}|$ is as follows:
    \begin{align}
        p(S=k \given |\mathcal{A}| ) = \begin{cases}
            \frac{1}{K}, &\text{if $|\mathcal{A}| > 0$ and $k \in \{1,\dots,K\}$} \\
            1, &\text{if $|\mathcal{A}| =0$ and $k = \emptyset$} \\
            0, &\text{otherwise} 
        \end{cases} \,.
    \end{align}
\end{lemma}

\begin{proof}[Proof of \Cref{thm:conditional-selection-prob}]
Let $S$ denote the final selected index.
Let $A_k$ denote whether $k$ is in the active set (assuming device or communication fault models).
First, we derive the probability of selection for a specific active set size $\asize$.
We notice that $p(S=\emptyset \given |\mathcal{A}| = 0) = 1$ (i.e., the fake client is selected if there are no active real clients) and $p(S=\emptyset \given |\mathcal{A}| > 0) = 0$ (i.e., the fake client is never selected if there is at least one active client.
Similarly, $p(S=k \given |\mathcal{A}| = 0) = 0, \forall c$ because there are no active clients.
The last remaining case is when $|\mathcal{A}|>0$ and $\asize \neq \emptyset$, i.e., $p(S=\asize \given |\mathcal{A}| > 0), \forall \asize \neq \emptyset$, which we derive as $\frac{1}{K}$ below:
\begin{align}
    &p(S=k \given |\mathcal{A}| = \asize) \\
    &=\sum_{a=0}^1 p(S=k, A_k=a \given |\mathcal{A}| = \asize) \label{eqn:joint-marginal} \\
    &=\sum_{a=0}^1 p(A_k=a \given |\mathcal{A}| = \asize) p(S=k \given A_k=a, |\mathcal{A}| = \asize) \\
    &=p(A_k=1 \given |\mathcal{A}| = \asize) p(S=k \given A_k=1, |\mathcal{A}| = \asize) + p(A_k=0 \given |\mathcal{A}| = \asize) p(S=k \given A_k=0, |\mathcal{A}| = \asize)\\
    &=p(A_k=1 \given |\mathcal{A}| = \asize) p(S=k \given A_k=1, |\mathcal{A}| = \asize) + p(A_k=0 \given |\mathcal{A}| = \asize) \cdot 0 \label{eqn:zero-prob-selected} \\
    &=p(A_k=1 \given |\mathcal{A}| = \asize) p(S=k \given A_k=1, |\mathcal{A}| = \asize)\\
    &=\left(\frac{\asize}{K} \right) \left(\frac{1}{\asize}\right) \label{eqn:prob-of-subsets}\\
    &=\frac{1}{K} \,.
\end{align}
where \eqref{eqn:joint-marginal} is by marginalization of joint distribution, \eqref{eqn:zero-prob-selected} is by noticing that if the device is not active, then it will not be selected, and \eqref{eqn:prob-of-subsets} is by the uniform distribution for Select Active $h$ and by noticing that 
\begin{align}
    p(A_k=1 \given |\mathcal{A}| = \asize) = \frac{\text{Num. subsets of size $\asize$ with $c$ in them}}{\text{Num. subsets of size $\asize$}} = \frac{\binom{K-1}{\asize-1}}{\binom{K}{\asize}} = \frac{K}{\asize}
\end{align}
where the numerator can be thought of as finding all possible subsets of size $\asize-1$ from $K-1$ clients (where client $c$ has been removed) and then adding client $c$ to get a subset of size $\asize$.

Putting this altogether we arrive at the following result for the probability of selection given various sizes of the active set:
\begin{align}
    p(S=k \given |\mathcal{A}| \, ) = \begin{cases}
        \frac{1}{K}, &\text{if $|\mathcal{A}| > 0, c \in \{1,\dots,K\}$} \\
        1, &\text{if $|\mathcal{A} =0, c = \emptyset$} \\
        0, &\text{otherwise} 
    \end{cases}
\end{align}
\end{proof}

Given this lemma, we now give the proof of the proposition.

\begin{proof}[Proof of \Cref{prop:k-MVFL}]
Let $S \in \{\emptyset,1,2,\dots,K\}$ denote a random variable that is the index of the final client prediction selected based on $h$, where $\emptyset$ denotes a fake client that represents the case where a non-active client is selected (which could happen in Select Any Client $h$ or if no clients are active for Select Active Client $h$).
The output of this fake client is equivalent to the marginal probability of $Y$ since the external client would know nothing about the input and would be as good as random guessing.
Furthermore, let $\Psi_S$ denote the $S$-th client's prediction.
Given this notation, we can expand the risk in terms of $S$ instead of $h$:
\begin{align}
    &R_h(\theta; \mathcal{G}(t)) \\
    &=\E_{\bm{x},y,\mathcal{G}(t),h}[\ell_h(\Psi(\bm{x}; \theta, \mathcal{G}(t)), y)] \\
    &=\E_S[\E_{\bm{x},y,\mathcal{G}(t)|S}[\ell(\Psi_S(\bm{x}; \theta, \mathcal{G}(t)), y)]] \\
    &=\mathrm{Pr}(S\neq \emptyset)\E_{\bm{x},y,\mathcal{G}(t)|S\neq \emptyset}[\ell(\Psi_S(\bm{x}; \theta, \mathcal{G}(t)), y)] + \mathrm{Pr}(S=\emptyset) \E_{\bm{x},y,\mathcal{G}(t)|S=\emptyset}[\ell(\Psi_S(\bm{x}; \theta, \mathcal{G}(t)), y)] \\
    &=\mathrm{Pr}(S\neq \emptyset)\E_{\bm{x},y,\mathcal{G}(t)|S\neq \emptyset}[\ell(\Psi_S(\bm{x}; \theta, \mathcal{G}(t)), y)] + \mathrm{Pr}(S=\emptyset) R(\theta; \mathcal{G}_\mathrm{empty}) \\
    &=(1-r^K) \E_{\bm{x},y,\mathcal{G}(t)|S\neq \emptyset}[\ell(\Psi_S(\bm{x}; \theta, \mathcal{G}(t)), y)] + r^K R(\theta; \mathcal{G}_\mathrm{empty}) \\
\end{align}
where the last term is by noticing that the probability of the fake one being chosen is equivalent to $|\mathcal{A}|=0$ and thus all devices fail which would have a probability of $r^K$.
We now decompose the second term in terms of clean risk:
\begin{align}
    &\E_{\bm{x},y,\mathcal{G}(t)|S\neq \emptyset}[\ell(\Psi_S(\bm{x}; \theta, \mathcal{G}(t)), y)] \\
    &=\E_{S|S \neq \emptyset}[\E_{\bm{x},y,\mathcal{G}(t)|S}[\ell(\Psi_S(\bm{x}; \theta, \mathcal{G}(t)), y)]] \\
    &=\E_{S| \|\mathcal{A}\| > 0}[\E_{\bm{x},y,\mathcal{G}(t)|S}[\ell(\Psi_S(\bm{x}; \theta, \mathcal{G}(t)), y)]] \\
    &=\sum_k p(S=k| \|\mathcal{A}\| > 0)\E_{\bm{x},y,\mathcal{G}(t)|S}[\ell(\Psi_S(\bm{x}; \theta, \mathcal{G}(t)), y)] \\
    &=\sum_k \frac{1}{K} \E_{\bm{x},y,\mathcal{G}(t)|S}[\ell(\Psi_S(\bm{x}; \theta, \mathcal{G}(t)), y)] \label{eqn:apply-lemma}\\
    &=\sum_k \frac{1}{K} R_{h_k}(\theta; \mathcal{G}(t)) \\
    &\geq \sum_k \frac{1}{K} R_{h_k}(\theta; \mathcal{G}_\mathrm{clean}) \\
    &=\sum_k \frac{1}{K} R_{h_k}(\theta; \mathcal{G}_\mathrm{clean}) \\
    &=R_h(\theta; \mathcal{G}_\mathrm{clean}) \,,
\end{align}
where \eqref{eqn:apply-lemma} is by \Cref{thm:conditional-selection-prob}, the inequality is due to our assumption that risk on a faulty graph is less than the risk on a clean graph, and the last line is by definition of the clean risk where $h$ is $h_\mathrm{active}$.
Combining the results, we have the final result:
\begin{align}
    R_h(\theta; \mathcal{G}(t))
    &=(1-r^K) R_h(\theta; \mathcal{G}_\mathrm{clean}) + r^K R(\theta; \mathcal{G}_\mathrm{empty}) \,.
\end{align}
\end{proof}

\subsection{Proof of \Cref{prop:gossip_Ensembling}}

\begin{proof}
    We first note that using a geometric average of probabilities (implemented using log probabilities for stability) satisfies the conditions in the Generalised Ambiguity Decomposition \citet[Proposition 3]{wood2023unified} for the ensemble combiner.
    (Similarly, if the problem was regression, we could use the squared loss with an arithmetic mean ensemble combiner for gossip.)
    As a reminder, let $\Psi_k$ denote the models that output probabilities of each class and let the ensemble model be denoted as $\Psi_\mathrm{ens}$ where $\Psi^\mathrm{ens}_{k}(\bm{x}; \theta, \mathcal{G}(t)) \triangleq Z^{-1}\exp(\sum_{k'=1}^K \Psi_{k'}(\bm{x}; \theta, \mathcal{G}(t))), \forall k \in \mathcal{K}$, where $Z$ is the normalizing constant to ensure the final output is a probability vector.
    Furthermore, let $\Psi_h(\xvec; \theta, \mathcal{G}(t)) \triangleq h(\Psi(\xvec; \theta, \mathcal{G}(t)), \mathcal{G}(T))$, i.e., it is merely the postprocessing of the original $\Psi$ function with $h$.
    This allows us to interchange the $h$ between the loss function and a modified $\Psi$, i.e., $\ell_h(\Psi(\xvec; \theta, \mathcal{G}(t)), y) = \ell(\Psi_h(\xvec; \theta, \mathcal{G}(t)),y)$.
    Similarly, with a slight abuse of notation, if $\Psi$ is on both sides of the loss function, we will apply $h$ to both inputs before passing to the loss function, i.e., $\ell_h(\Psi_1(\xvec; \theta, \mathcal{G}(t)), \Psi_2(\xvec; \theta, \mathcal{G}(t))) = \ell(\Psi_{h,1}(\xvec; \theta, \mathcal{G}(t)),\Psi_{h,2}(\xvec; \theta, \mathcal{G}(t)))$.
    Given this, assuming that $h=h_\mathrm{active}$, we can decompose the risk as follows:
    \begin{align}
        &R^\mathrm{ens}_h(\theta; \mathcal{G}(t)) \\
        &=\mathbb{E}_{\bm{x},y,\mathcal{G}(t),h}[\ell_h(\Psi^\mathrm{ens}(\bm{x}; \theta, \mathcal{G}(t)), y)] \\
        &=\mathbb{E}_{\bm{x},y,\mathcal{G}(t),h}[\ell(\Psi_h^\mathrm{ens}(\bm{x}; \theta, \mathcal{G}(t)), y)] \\
        &=\mathbb{E}_{\bm{x},y,\mathcal{G}(t),h}[
        \textstyle \frac{1}{K}\sum_{k=1}^K \ell(\Psi_{h,k}(\bm{x}; \theta, \mathcal{G}(t)), y) 
        - \frac{1}{K}\sum_{k=1}^K \ell(\Psi_{h,k}(\bm{x}; \theta, \mathcal{G}(t)), \Psi^\mathrm{ens}_h(\bm{x}; \theta, \mathcal{G}(t))) ]  \\
        &=\mathbb{E}_{\bm{x},y,\mathcal{G}(t),h}[
        \textstyle \frac{1}{K}\sum_{k=1}^K \ell_h(\Psi_{k}(\bm{x}; \theta, \mathcal{G}(t)), y) 
        - \frac{1}{K}\sum_{k=1}^K \ell_h(\Psi_{k}(\bm{x}; \theta, \mathcal{G}(t)), \Psi^\mathrm{ens}(\bm{x}; \theta, \mathcal{G}(t))) ]  \\
        &=
        \textstyle R_h(\theta; \mathcal{G}(t))
        - \mathbb{E}_{\bm{x},\mathcal{G}(t),h}[\frac{1}{K}\sum_{k=1}^K \ell_h(\Psi_{k}(\bm{x}; \theta, \mathcal{G}(t)), \Psi^\mathrm{ens}(\bm{x}; \theta, \mathcal{G}(t))) ]  \\
        &\leq R_h(\theta; \mathcal{G}(t)) \,.
    \end{align}
    where the first equals is by definition, the second is by pushing the $h$ function into $\Psi$ so that we have the raw loss function $\ell$, the third equals is by \citet[Proposition 3]{wood2023unified}, the fourth is by pulling the $h$ function back out into the loss function with a slight abuse of notation where the RHS term the $h$ function is applied to both arguments before passing to the original loss function, the fifth is by noticing that the non-ensemble risk is equal to the average risk of each aggregator-specific model, and the last inequality is by noticing that hte loss function is always non-negative.
\end{proof}

\section{Device and Communication Fault Visualization}\label{Dev_ComFault}

We show in Figure \ref{Fig: FaultVisualization} the visual representation of communication and device fault under the MVFL method. Although, we present the scenario for only one method, by extension the visualization is similar for VFL, DMVFL and the gossip variants.

\begin{figure*}[!ht]
    \centering
    \begin{subfigure}[t]{0.28\linewidth}
        \centering
        \includegraphics[clip, trim=8cm 3cm 6cm 4cm,width=\linewidth]{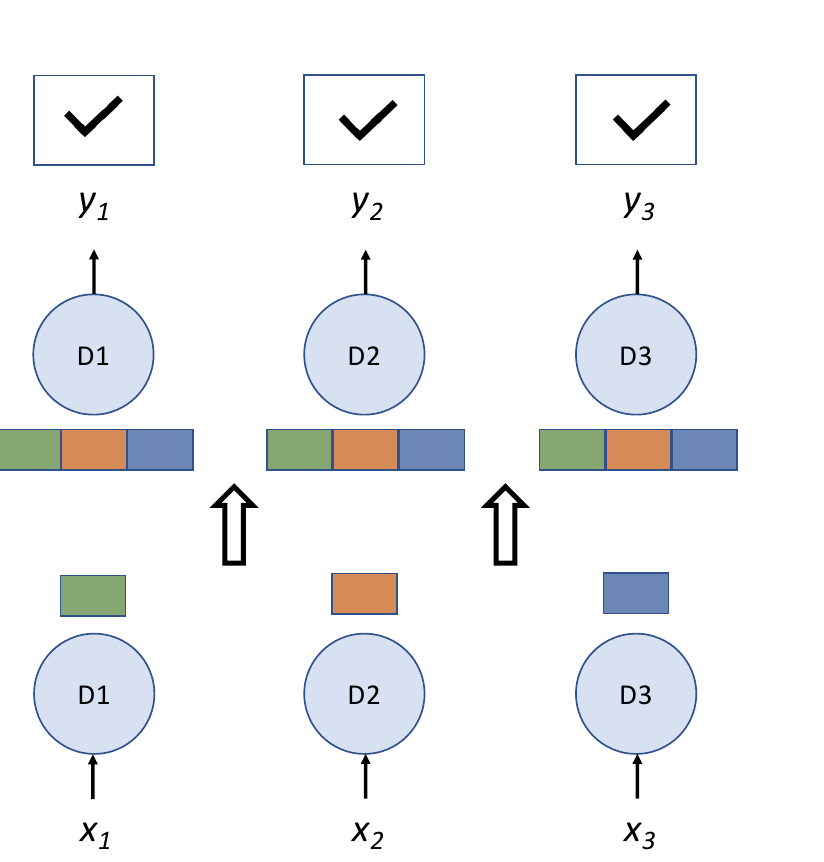}
        \caption{MVFL}
    \end{subfigure}\hspace{0.02\textwidth}
\begin{subfigure}[t]{0.28\textwidth}
        \includegraphics[clip, trim=8cm 3cm 6cm 4cm,width=\linewidth]{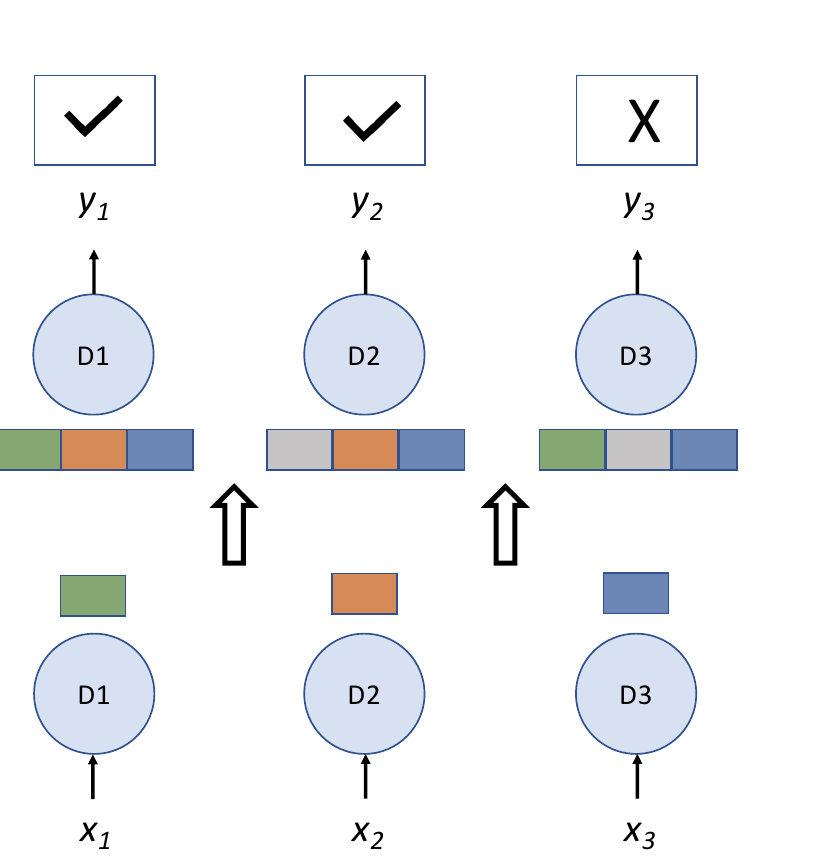}
        \caption{MVFL-Communication Fault}
    \end{subfigure}\hspace{0.02\textwidth}
\begin{subfigure}[t]{0.36\textwidth}
\includegraphics[clip, trim=8cm 3cm 2cm 4cm,width=\linewidth]{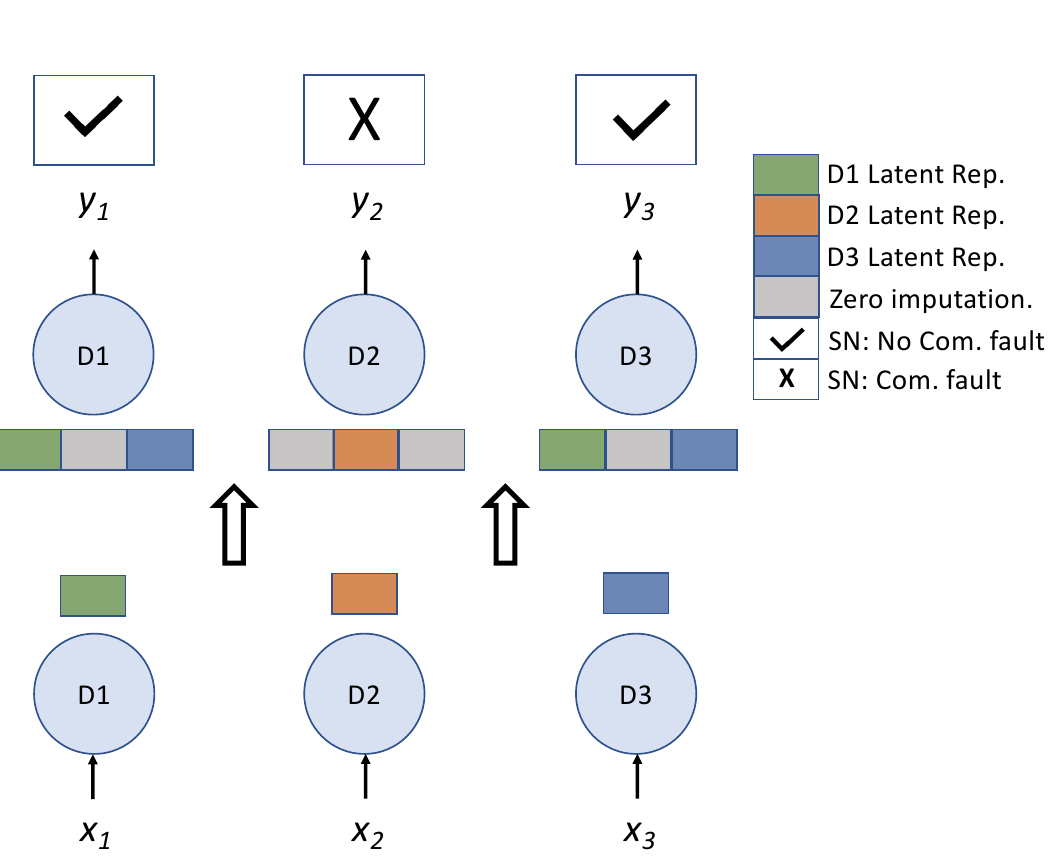}
        \caption{MVFL-Device Fault}
    \end{subfigure}
    \vspace{0em}
    \caption{
    Illustration of communication and device faults for a 3 device network for the MVFL method. (a) Fully connected MVFL setup. The check mark indicates that there is no fault in the final communication between device and special node (SN) as defined in Section \ref{sec:distributed-inference} of the main paper (b) Representation with communication faults. In this example communication from D1 to D2 and D2 to D3 is faulted. To account for the missing values, we do zero imputation. $X$ indicates that the communication between D3 and SN is faulted. Hence, the output at SN will be a class selected with uniform probability among all the classes (c) In device faults, the faulted device do not communicate with any other devices and missing values are accounted for by zero imputation. In this example, D2 is assumed to be faulted, hence the information from D2 is not passed to D1 or D3 and it does not produce an output. The output at SN for $D_2$ will be a class selected with uniform probability among all the classes }
    \label{Fig: FaultVisualization}
    \vspace{-1em}
\end{figure*}

\section{Party-wise Dropout and Communication-wise Dropout MVFL}\label{CD_PD_MVFL}
In \cref{sec:methods} of the main paper, we presented the Party-wise and communication-wise Dropout method for MVFL. 
In \cref{Fig: PD_CD_MVFL} we represent CD-MVFL and PD-MVFL for a group of 3 devices.

\begin{figure*}[!ht]
    \centering
    \begin{subfigure}[b]{0.28\linewidth}
        \centering
        \includegraphics[clip, trim=8cm 3cm 6cm 4cm,width=\linewidth]{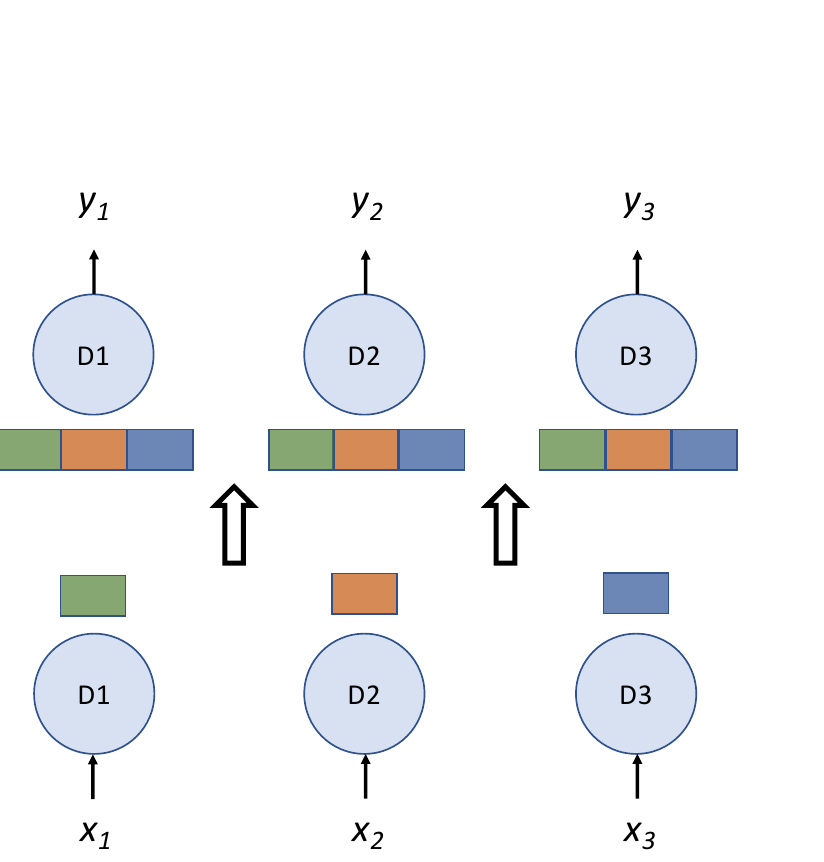}
        \caption{MVFL}
    \end{subfigure}\hspace{0.02\textwidth}
\begin{subfigure}[b]{0.28\textwidth}
        \includegraphics[clip, trim=8cm 3cm 6cm 4cm,width=\linewidth]{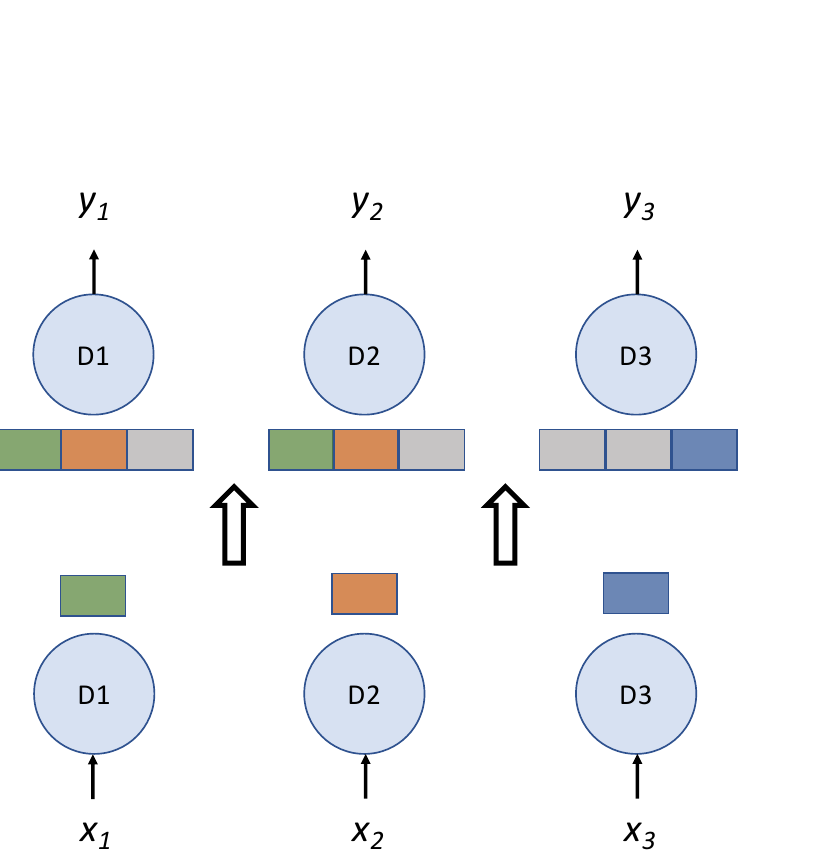}
        \caption{PD-MVFL}
    \end{subfigure}\hspace{0.02\textwidth}
\begin{subfigure}[b]{0.36\textwidth}
\includegraphics[clip, trim=8cm 3cm 2cm 4cm,width=\linewidth]{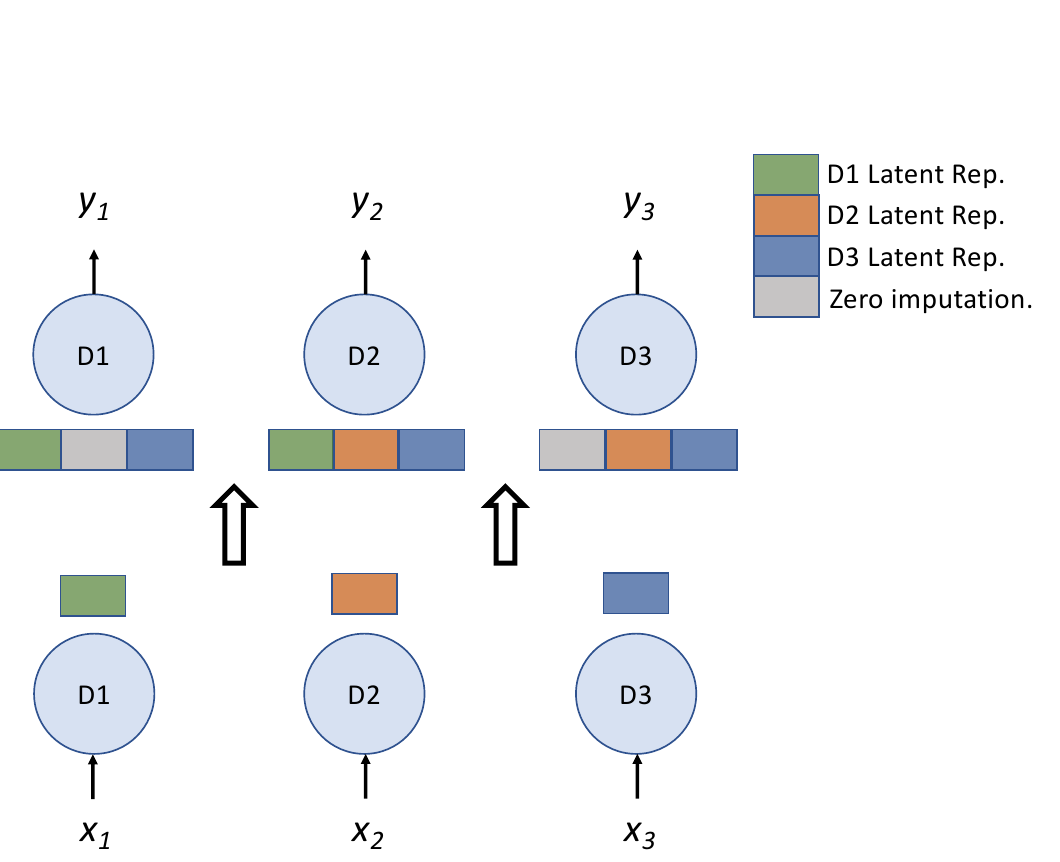}
        \caption{CD-MVFL}
    \end{subfigure}
    \vspace{0em}
    \caption{
    Illustration of Party wise and Communication wise dropout for a 3 device network for the MVFL method. (a) Fully connected MVFL setup. (b) For Party wise Dropout (PD), during training if D3 is dropped then none of the devices gets representations from D3 and the missing values are imputed by zeros. (c) In communication wise Dropout (CD) certain representations are omitted during training. In this example representations from from D2 to D1 and D1 to D3 are omitted by design during the training.}
    \label{Fig: PD_CD_MVFL}
    \vspace{-1em}
\end{figure*}

\section{Deep MVFL}\label{D_MVFL}
In \cref{sec:methods} of the main paper, we presented a few innovations that introduces redundancy in the system.
Extending the MVFL setup, we propose another variant, Deep MVFL (DMVFL), which stacks MVFL models on top of each other and necessitates multiple rounds of communication between devices for each input. We believe that multiple communication rounds and deeper processing could lead to more robustness on dynamic networks. Comparing \cref{IL_DMVFL_BaseLines} (b) and (c), the setup is same till $L_2$ but in DMVFL, there is an additional round of communication in $L_3$ following which the predictions are made. In \cref{IL_DMVFL_BaseLines} (c) we have illustrated DMVFL with just one additional round of communication over MVFL and hence the depth of DMVFL is 1. Nonetheless, the depth in DMVFL need not be restricted to 1 and is a hyperparameter. Furthermore, to guarantee a fair comparison between DMVFL and MVFL, it was ensured that the number of parameters for both these setups be the same.

In DMVFL the redundancy is over depth, However, based on our experiments we did not observe a significant performance gain and hence did not present it in the main paper.

\begin{figure*}
    \centering
    \begin{subfigure}[t]{0.22\linewidth}
        \centering
        \includegraphics[clip, trim=8cm 3cm 10cm 3cm,width=\linewidth]{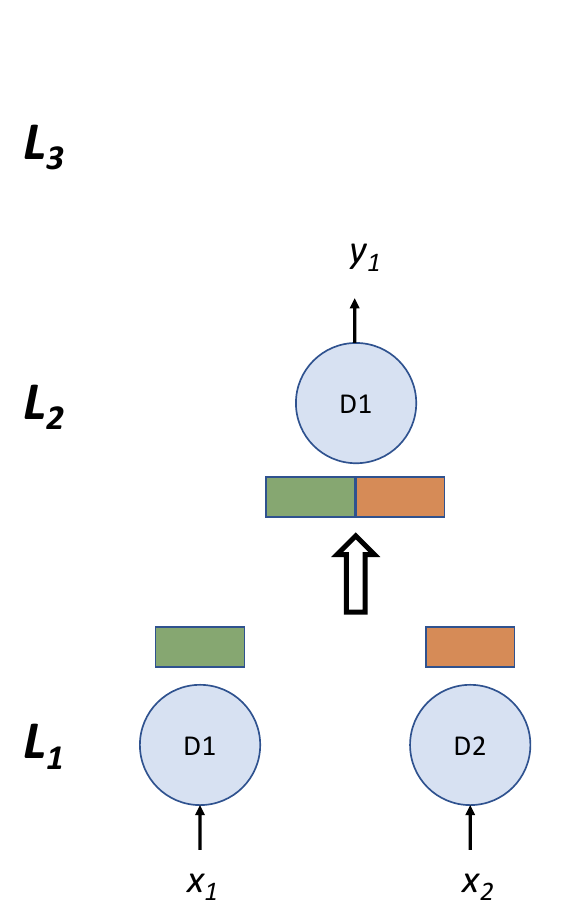}
        \caption{VFL}
    \end{subfigure}\hspace{0.02\textwidth}
\begin{subfigure}[t]{0.2\textwidth}
        \includegraphics[clip, trim=9cm 3cm 10cm 3cm,width=\linewidth]{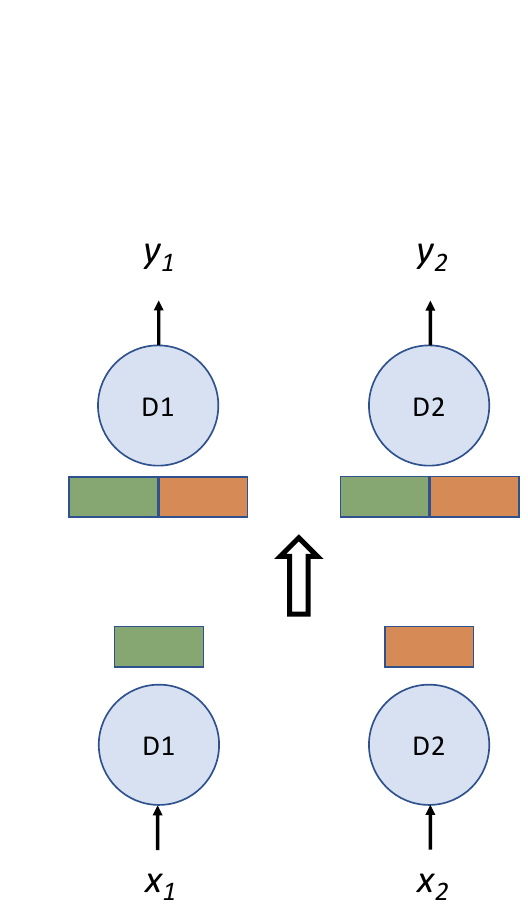}
        \caption{MVFL}
    \end{subfigure}\hspace{0.02\textwidth}
\begin{subfigure}[t]{0.2\textwidth}
        \includegraphics[clip, trim=9cm 3cm 10cm 3cm,width=\linewidth]{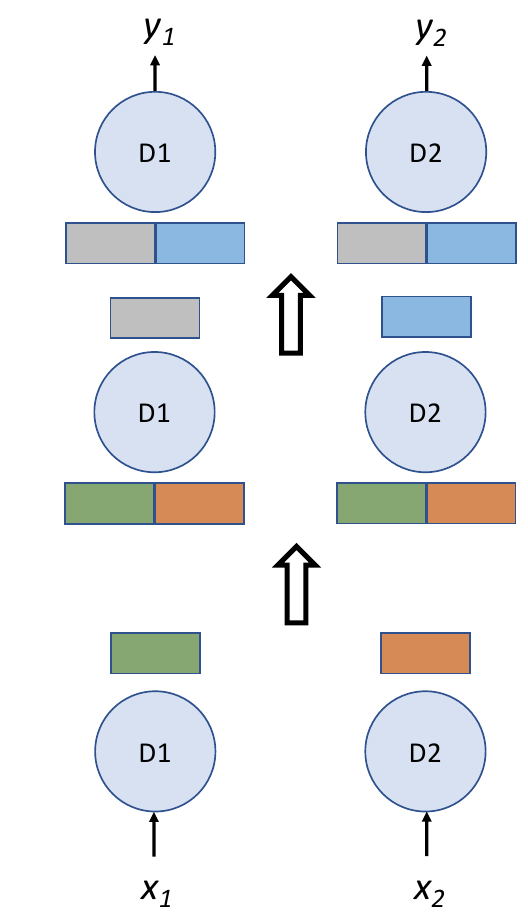}
        \caption{DMVFL}
    \end{subfigure}\hspace{0.02\textwidth}
\begin{subfigure}[t]{0.2\textwidth}
        \includegraphics[clip, trim=9cm 3cm 10cm 3cm,width=\linewidth]{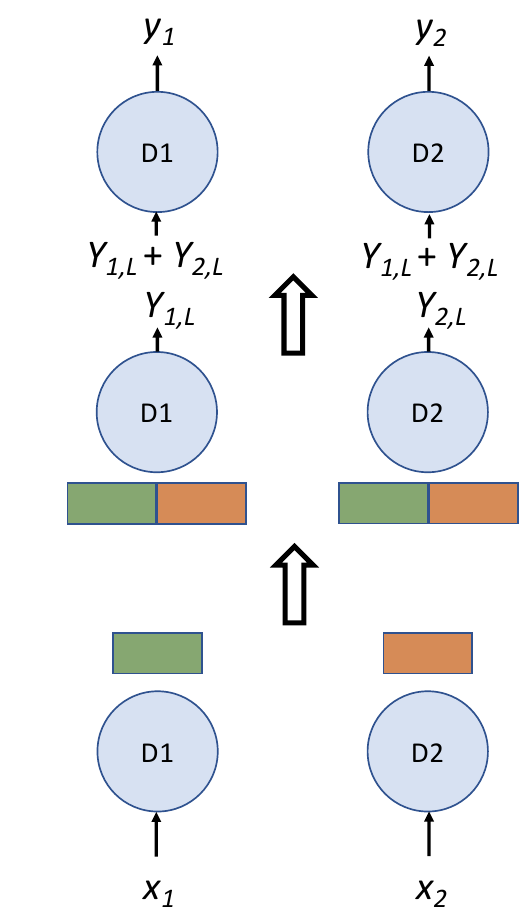}
        \caption{MVFL-G}
    \end{subfigure}\hspace{0.02\textwidth}
    \vspace{-1em}
    \caption{
    VFL as a baseline and the proposed innovations are illustrated for a network of two fully connected devices, D1 and D2. (a)VFL setup with D1 acting as a client as well as the aggregating server. The input to the devices at the first layer $L_1$ are $x_1$ and $x_2$ and output are the latent representations. The input to the server on the second layer $L_2$ is the concatenated latent representation and the output is the prediction $y_1$ (b) MVFL arrangement has both the devices acting as servers aside from being clients. (c) DMVFL has a similar arrangement as MVFL, expect that there is an additional layer of processing, $L_3$, that has the concatenated features from the previous layer as an input and the output are the predictions. (d) MVFL-G is an extension of MVFL wherein the output log probabilities ($Y_{i,L}$) from each device are averaged before being used for final prediction.}
    \label{IL_DMVFL_BaseLines}
    \vspace{-1em}
\end{figure*}

\section{Further Discussion and Limitations}\label{disc_Lim}
\paragraph{Distributed Inference Algorithm's Resemblance to GNNs:}
The form of our distributed inference algorithm in the main paper has a superficial resemblance of the computation of graph neural networks (GNN) \citep{scarselli2008graph} but with important semantic and syntactic differences.
Semantically, unlike GNN applications whose goal is to predict global, node, or edge properties based on the graph edges, our goal is to do prediction well given \emph{any arbitrary} edge structure.
Indeed, the edges in our dynamic network are assumed to be independent of the input and task---rather they are simply constraints based on the network context of the system.
Syntactically, our inference algorithm differs from mainstream convolutional GNNs because convolutional GNNs share the parameters across clients (i.e., $\theta_c^{(t)} = \theta^{(t)}$) whereas in our algorithm the parameters at each client are \emph{not shared} across clients (i.e., $\theta_c^{(t)} \neq \theta_{c'}^{(t)}$).
Additionally, most GNNs assume the aggregation function $g$ is permutation equivariant such as a sum, product or maximum function.
However, we assume $g$ could be any aggregation function.
Finally, this definition incorporates the last processing function $h$ that represents the final communication round to an external entity (Main Paper Section \ref{sec:distributed-inference}).

\paragraph{Limitations:} In recent times, a key consideration for distributed learning paradigms, such as Federated Learning, has been to ensure that clients or devices data remains private. Given the distributed nature and the ability of the devices/clients in Decentralized VFL to interact with each other, we believe that, though not the primary property, privacy and security of Decentralized VFL is an important direction for future work. Furthermore, the experiments in the main paper were simulated with no communication latency. However, it is a salient practical consideration. Therefore, in the future it will be crucial to consider this and develop methods that can accommodate for asynchronous or semi-synchronous updates.

\paragraph{Alternative Approach using Fault-Tolerant Consensus Algorithms}
Instead of direct replication via MVFL, one alternative fault-tolerant approach would be to first run a fault-tolerant consensus algorithm such as Paxos \citep{lamport2001paxos} or Raft \citep{ongaro2014search} and then run standard VFL inference with the elected leader.
This could reduce the communication load during distributed inference but would reduce the robustness or increase latency compared to \methodAbv.
For example, Paxos may fail or wait indefinitely for extreme fault rates near 50\%.

Additionally, Paxos would increase the latency as consensus would need to be arrived before continuing.
On the other hand, \methodAbv would provide an answer (perhaps degraded but that is expected) with the same latency no matter the percentage of faults even for more extreme faults beyond 50\%.
Secondly, we point out that accuracy performance with \methodAbv actually benefits between 2-3\% because of the ensembling of multiple devices' predictions. For example, on the grid graph in Figure 2c in the main paper, there is a 2-3\% gap between PD-MVFL (which would roughly correspond to PD+Paxos) compared to PD-MVFL-G4 (which has gossip and produces an ensembling effect). 

While distributed algorithms such as Paxos and Raft \citep{lamport2001paxos,ongaro2014search} are useful to generate consensus in a system that encounters fault, they do not enable representations that are robust to faults, which is essential for achieving good performance. Thus, they cannot be naively applied for addressing \contextAbv and more carefully constructed method like \methodAbv is required.

\section{Experiment Details}\label{ExpDet}
\subsection{Datasets}\label{DataSets}
For the experiments presented in this paper, following are the datasets that were used:

\textbf{StarCraftMNIST(SCMNIST):} Contains a total of 70,000 28x28 grayscale images in 10 classes. The data set has 60,000 training and 10,000 testing images. For experiments, all the testing images were used, 48,000 training images were used for training and 12,000 training images were used for validation study.

\textbf{MNIST:} Contains a total of 70,000 28x28 grayscale images in 10 classes. The data set has 60,000 training and 10,000 testing images. For experiments, all the testing images were used, 48,000 training images were used for training and 12,000 training images were used for validation study.

\textbf{CIFAR-10:} Contains a total of 60,000 32x32 color images in 10 classes, with each class having 6000 images. The data set has 50,000 training and 10,000 testing images. For experiments, all the testing images were used, 40,000 training images were used for training and 10,000 training images were used for validation study. 

\subsection{Graph construction}\label{Graph_det}
In the main paper as well as in the Appendix the terms client and devices are used interchangeably. In Section \ref{sec:experiment-setup} of the main paper, four different graphs were introduced: \emph{Complete},\emph{Ring}, \emph{Random Geometric} and \emph{Grid}. To elaborate how these graphs are constructed for a set of 16 clients, we take an example of an image from each of the three datasets and split it up into 16 sections, as illustrated in Figure \ref{Fig: DeviceAllocation}. 

For a \emph{Complete} graph, all the devices are connected to the server. For instance, if $D_1$ is selected as the server, then all the other devices $D_i$ for $i$ = $2$, $3$, \dots, $16$ are connected to $D_1$. To construct \emph{Grid} graph, we use compute a \emph{Distance} parameter. For \emph{Grid} graph, \emph{distance} returns true if a selected device lies horizontally or vertically adjacent to a server and only under this circumstance it is connected to the server otherwise it is not. For example, in Figure \ref{Fig: DeviceAllocation}, if $D_3$ is selected as the server, then $D_2$, $D_4$ and $D_7$ are the only devices connected to $D_3$. Another example will be, if $D_{13}$ is selected as the server, then $D_9$ and $D_{14}$ are the only ones connected to the server. 

The \emph{Ring} graph is connected by joining all the devices in a sequential order of increasing indices with the last device connected to the first. In our example it will be constructed by joining, $D_1$ with $D_2$, then $D_2$ with $D_3$ and so on and so forth with $D_{16}$ connected back to $D_1$.
Finally, \emph{Random Geometric Graph} is constructed by connecting a device to all other devices that fall within a certain radius (r) parameter. Hence a lower value of r denotes a device is connected to less devices compared to a larger value of r.

Irrespective of the base graph, \emph{Grid} or \emph{Complete}, when training or testing faults are applied to the selected base graph, during implementation it is assumed that the graph with incorporated faults stays constant for one entire batch and then the graph is reevaluated for the next batch. In our experiments, the batch-size is taken to be 64.

Furthermore, in Figure \ref{Fig: Easy_Hard} we highlight a few examples, to illustrate with MNIST images, why in some cases it is easy to distinguish between images based on partial information and while in other situations, it is not. Thus, device connectivity plays a crucial role in enabling classification tasks.

\begin{figure*}[!ht]
    \centering
    \begin{subfigure}[b]{0.3\linewidth}
        \centering
        \includegraphics[clip, trim=7cm 4cm 8cm 4cm,width=\linewidth]{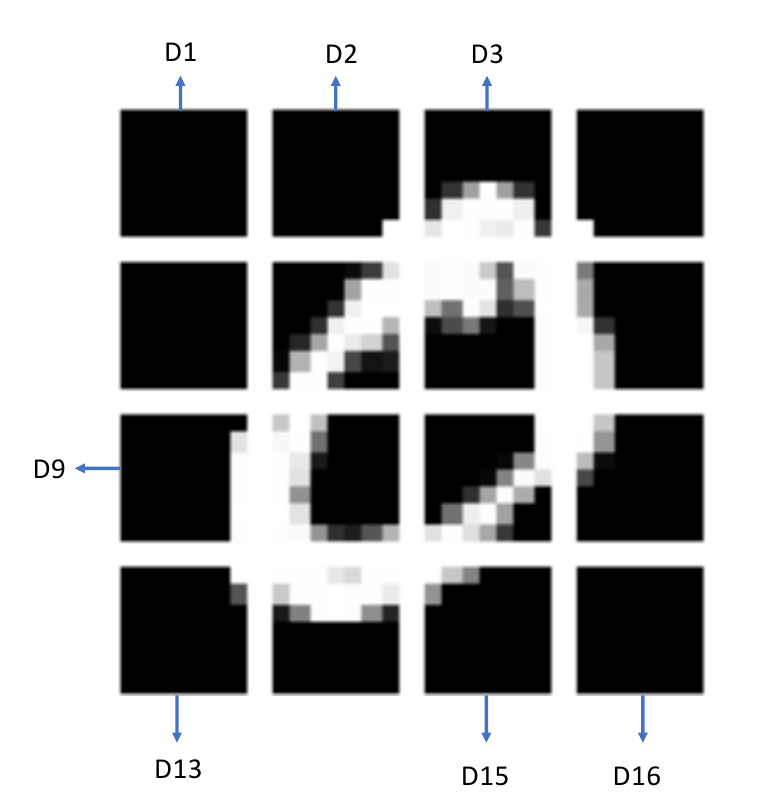}
        \caption{MNIST}
    \end{subfigure}\hspace{0.02\textwidth}
\begin{subfigure}[b]{0.3\textwidth}
        \includegraphics[clip, trim=7cm 4cm 8cm 4cm,width=\linewidth]{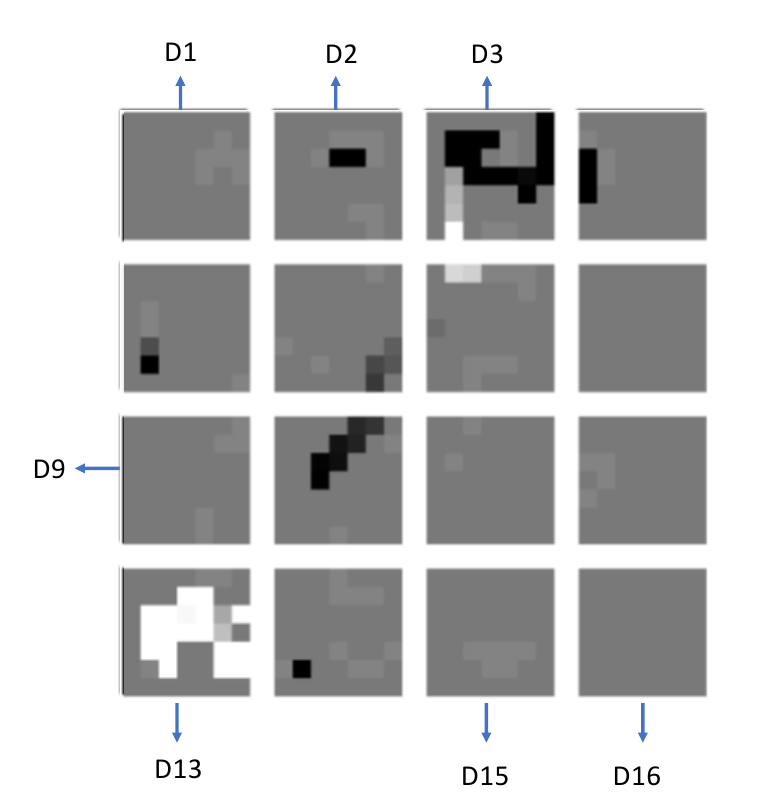}
        \caption{StarCraftMNIST(SCMNIST)}
    \end{subfigure}\hspace{0.02\textwidth}
\begin{subfigure}[b]{0.3\textwidth}
\includegraphics[clip, trim=7cm 4cm 8cm 4cm,width=\linewidth]{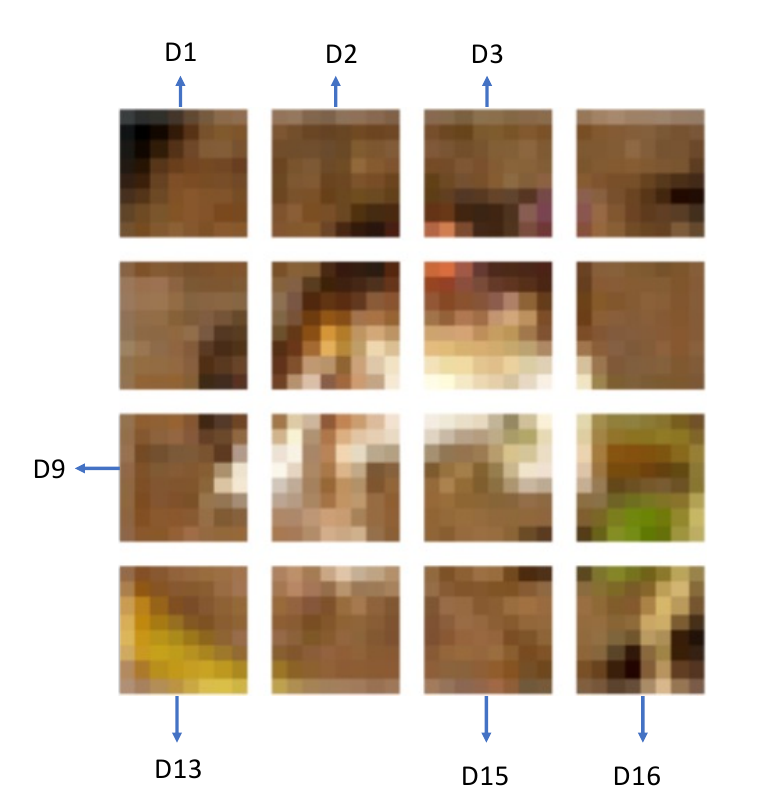}
        \caption{CIFAR-10}
    \end{subfigure}
    \vspace{0em}
    \caption{(a)MNIST, (b)SCMNIST, (c)CIFAR-10  Image split into 16 sections. Each section is assigned to a device/client.$D_i$ denotes a device/client}
    \label{Fig: DeviceAllocation}
\end{figure*}

\begin{figure*}[!ht]
    \centering
    \begin{subfigure}[b]{0.23\linewidth}
        \centering
        \includegraphics[clip, trim=7cm 4cm 8cm 4cm,width=\linewidth]{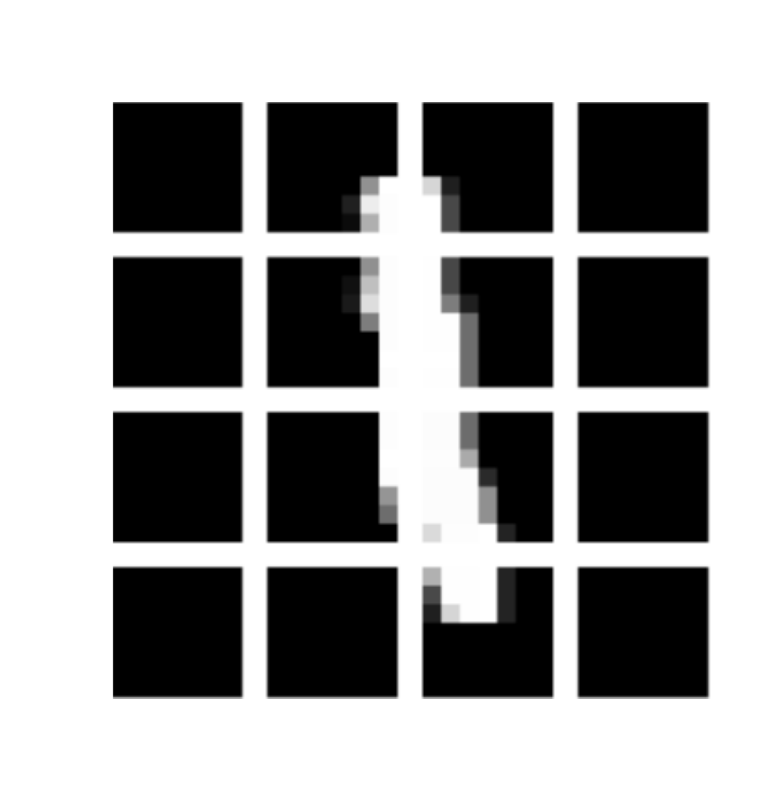}
        \caption{1}
    \end{subfigure}\hspace{0.02\textwidth}
\begin{subfigure}[b]{0.23\textwidth}
        \includegraphics[clip, trim=7cm 4cm 8cm 4cm,width=\linewidth]{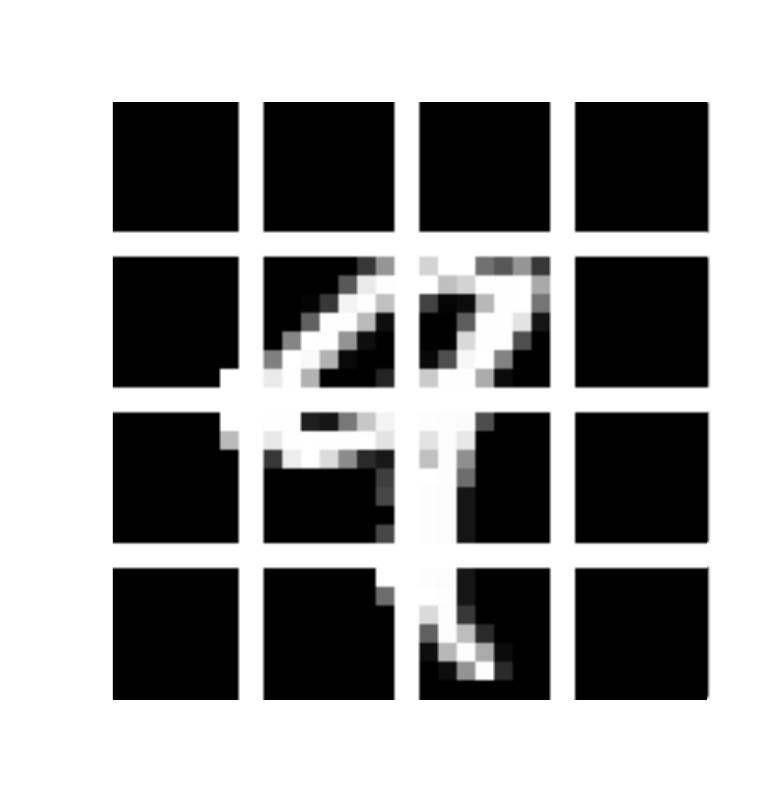}
        \caption{9}
    \end{subfigure}\hspace{0.02\textwidth}
\begin{subfigure}[b]{0.23\textwidth}
\includegraphics[clip, trim=7cm 4cm 8cm 4cm,width=\linewidth]{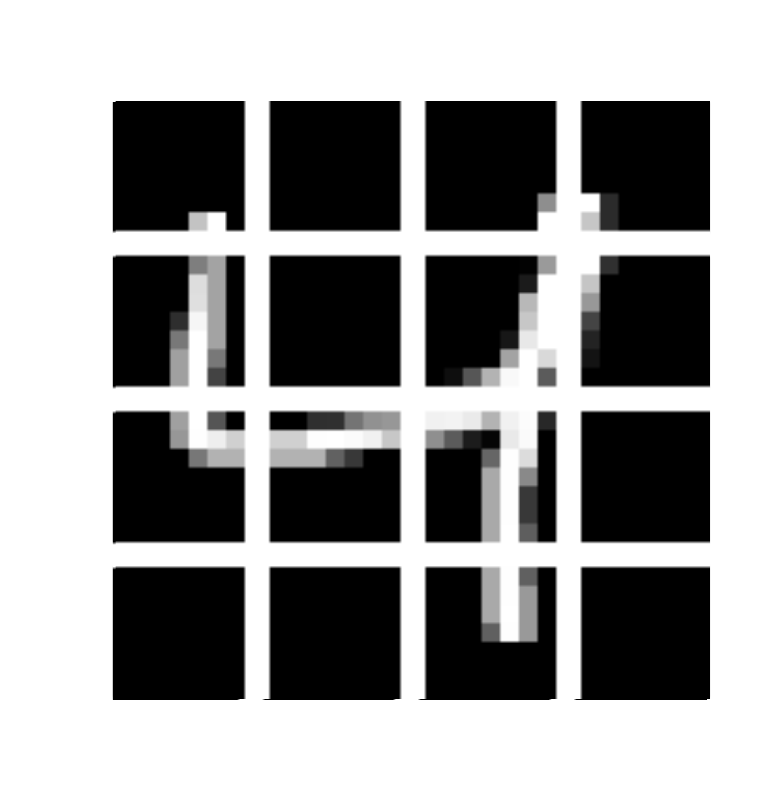}
        \caption{4}
    \end{subfigure}
    \begin{subfigure}[b]{0.23\textwidth}
\includegraphics[clip, trim=7cm 4cm 8cm 4cm,width=\linewidth]{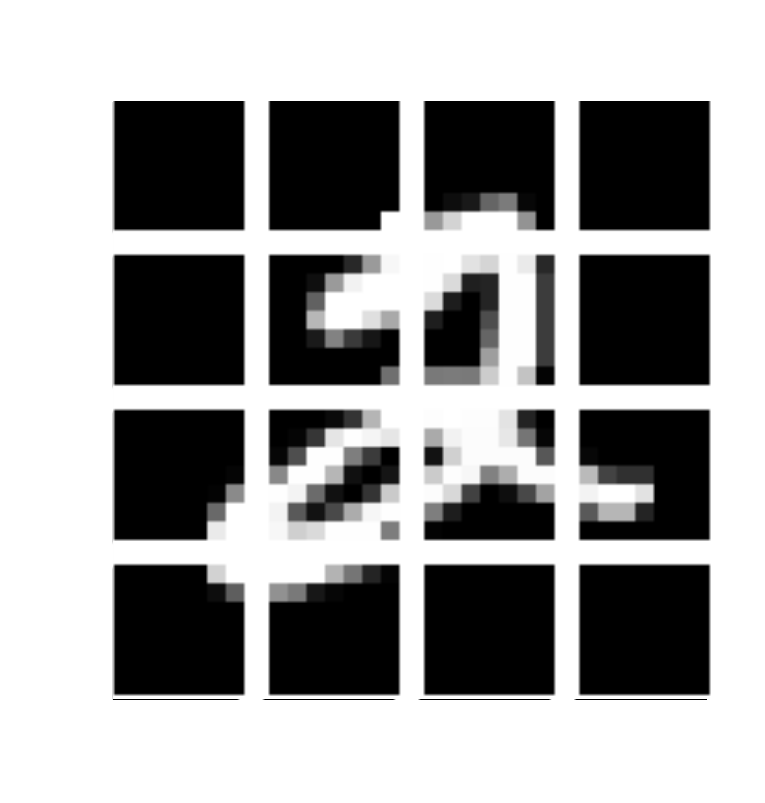}
        \caption{2}
    \end{subfigure}
    \vspace{0em}
    \caption{Based on limited information, some images are easy to distinguish from one another, others are not. For instance, based on the information from just bottom half of the devices, it is hard to distinguish between (a) and (b) while differentiating between images (c) and (d) is achievable just based in the bottom half of devices.}
    \label{Fig: Easy_Hard}
\end{figure*}

\subsection{Training}\label{Training}
For all of our experiments, we train the model for 100 epochs and we always report the result using the model checkpoint with lowest validation loss.
We use a batch size of 64 and Adam optimizer with learning rate 0.001 and $(\beta_1,\beta_2)=(0.9, 0.999)$.
All experiments are repeated using seed 1,2,\dots,16.

For all experiments, we use concatenation as the message passing algorithm where missing values were imputed using zeros (equivalent to dropout).

For 16 devices with all datasets, each device in VFL and MVFL has a model with the following structure: \texttt{Linear(49,16),ReLU,Linear(16,4),ReLU,MP,Linear(64,64),ReLU,Linear(64,10)} where $\texttt{MP}$ means message passing.
For 4 devices with all datasets, each device in VFL and MVFL has a model with the following structure: \texttt{Linear(196,64),ReLU,Linear(64,16),ReLU,MP,Linear(64,64),ReLU,Linear(64,10)}.
For 49 devices with all datasets, each device in VFL and MVFL has a model with the following structure: \texttt{Linear(16,4),ReLU,Linear(4,2),ReLU,MP,Linear(98,98),ReLU,Linear(98,10)}.

For 16 devices with all datasets, each device in DMVFL has a model with the following structure: \texttt{Linear(49,16),ReLU,Linear(16,4),ReLU,MP,DeepLayer,Linear(64,10)}.
For 4 devices with all datasets, each device in DMVFL has a model with the following structure: \texttt{Linear(196,64),ReLU,Linear(16,4),ReLU,MP,DeepLayer,Linear(64,10)}.
For 49 devices with all datasets, each device in DMVFL has a model with the following structure: \texttt{Linear(16,4),ReLU,Linear(16,4),ReLU,MP,DeepLayer,Linear(98,10)}.
$\texttt{DeepLayer}$ are composed of a sequence of \texttt{MultiLinear(64,16)} based on depth and \texttt{MultiLinear(x)}=\texttt{Mean(ReLU(Linear(64,16)(x)),\dots,ReLU(Linear(64,16)(x)))}.
Here we use multiple perceptrons at each layer to make sure the number of parameters between MVFL and DMVFL.
For example, for 16 devices and a depth of 2, we use $16/2=8$ perceptrons at each layer.

All experiments are performed on a NVIDIA RTX A5000 GPU.

\subsection{Handling Faults at Test Time}
During inference, if a communication or device faults, we impute the missing values with zeros for all methods.
Future work could look into other missing value imputation methods that are more effective for the given context.

\section{Additional Experiments}\label{app-sec:experiment-result}

In this section we present some more results from different experiments that we conducted. Like the results in the main paper, we present the Rand test metric over the active set. As we have multiple replicates of an experiment, we take an expectation over the collected Rand metric and this has an averaging effect. Thus, the y axis in the plots of the Appendix are labelled as Test Avg, which is equivalent to Test Active Rand y axis label used in the main paper and these two ways to refer to the metric are used interchangeably. 

Furthermore, in our experiments we were initially using gossip both during inference and training. 
However, we realised that gossip during inference alone is a better approach.
Thus, the results in the main paper are presented using gossip only during inference.
On the other hand, the experiments presented in the Appendix use gossip both during inference and training.

\begin{figure*}[!ht]
    \centering
    \vspace{-1em}
    \begin{subfigure}[t]{0.2\linewidth}
        \centering
        \includegraphics[clip, trim=8cm 3cm 10cm 3cm,width=\linewidth]{Graphics/VFL.pdf}
        \caption{VFL}
    \end{subfigure}\hspace{-0.02\textwidth}
\begin{subfigure}[t]{0.18\textwidth}
        \includegraphics[clip, trim=9cm 3cm 10cm 3cm,width=\linewidth]{Graphics/MVFL.pdf}
        \caption{MVFL}
    \end{subfigure}\hspace{0.01\textwidth}
\begin{subfigure}[t]{0.18\textwidth}
        \includegraphics[clip, trim=9cm 3cm 10cm 3cm,width=\linewidth]{Graphics/DMVFL_G.pdf}
        \caption{MVFL-G}
    \end{subfigure}\hspace{0.01\textwidth}
\begin{subfigure}[t]{0.29\textwidth}
        \includegraphics[clip, trim=9cm 3cm 4.5cm 3cm,width=\linewidth]{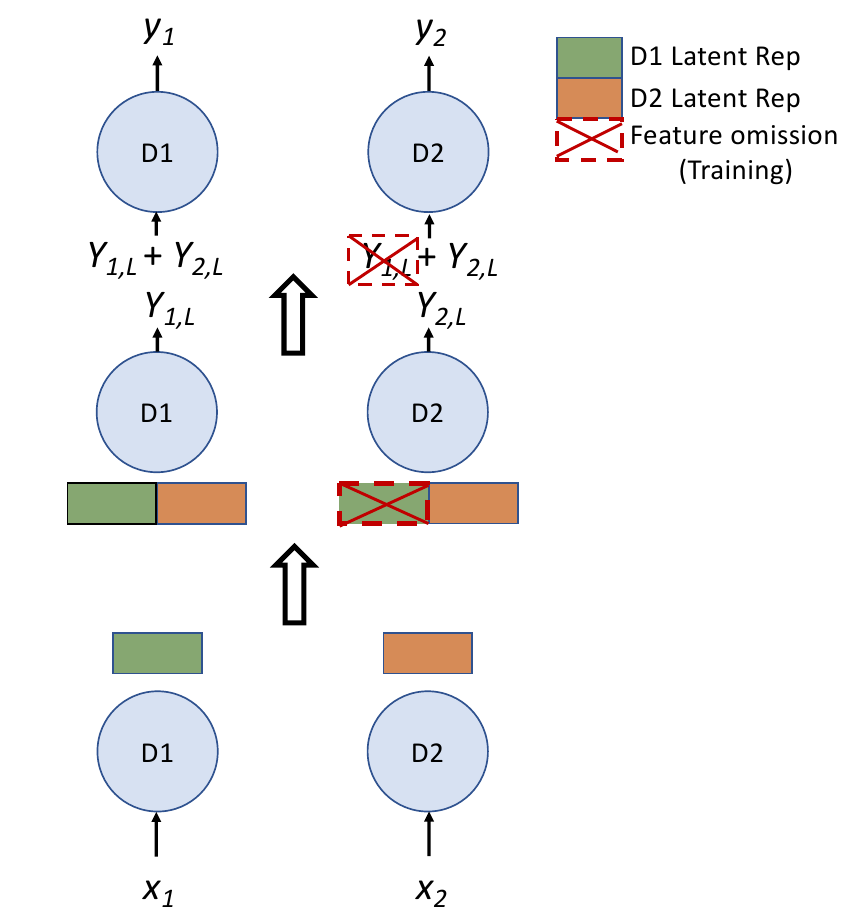}
        \caption{CD-MVFL-G}
    \end{subfigure}\hspace{0.01\textwidth}
    \vspace{-0.5em}
    \caption{
    VFL and the proposed enhancements are illustrated for a network of two fully connected devices, D1 and D2. (a)VFL setup with D1 acting as a client as well as an aggregator. The input to the devices at the first layer $L_1$ are $x_1$ and $x_2$ and output are the latent representations. The input to the server on the second layer $L_2$ is the concatenated latent representation and the output is the prediction $y_1$ (b) MVFL arrangement has both the devices acting as data aggregators aside from being clients. (c) MVFL-G is an extension of MVFL wherein the output log probabilities ($Y_{i,L}$) from each device are averaged before being used for final prediction. (d) CD-MVFL-G is a variant of MVFL-G where during the training phase, representation from D2 to D1 is not communicated by design. CD-MVFL-G is the method that we propose for \contextAbv}
    \label{IL BaseLines}
    \vspace{-1em}
\end{figure*}

\begin{figure*}[!ht]
    \centering
    \begin{subfigure}[h]{\linewidth}
        \includegraphics[width=\linewidth]{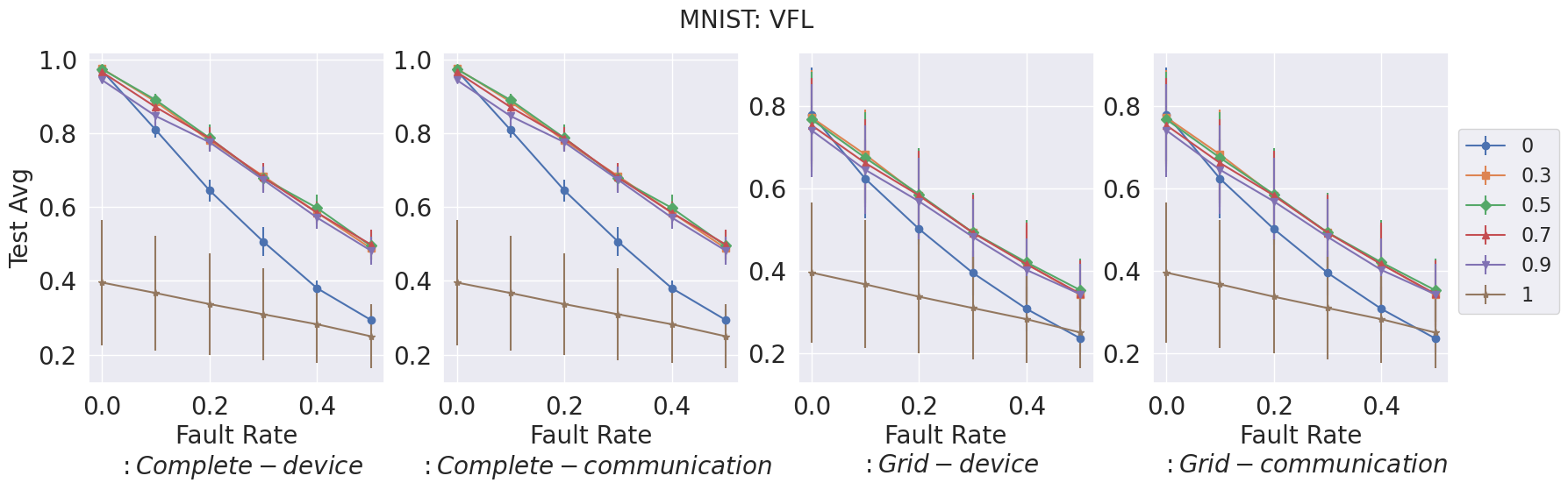}
        \caption{VFL}
    \end{subfigure}
    \begin{subfigure}[h]{\linewidth}
        \includegraphics[width=\textwidth]{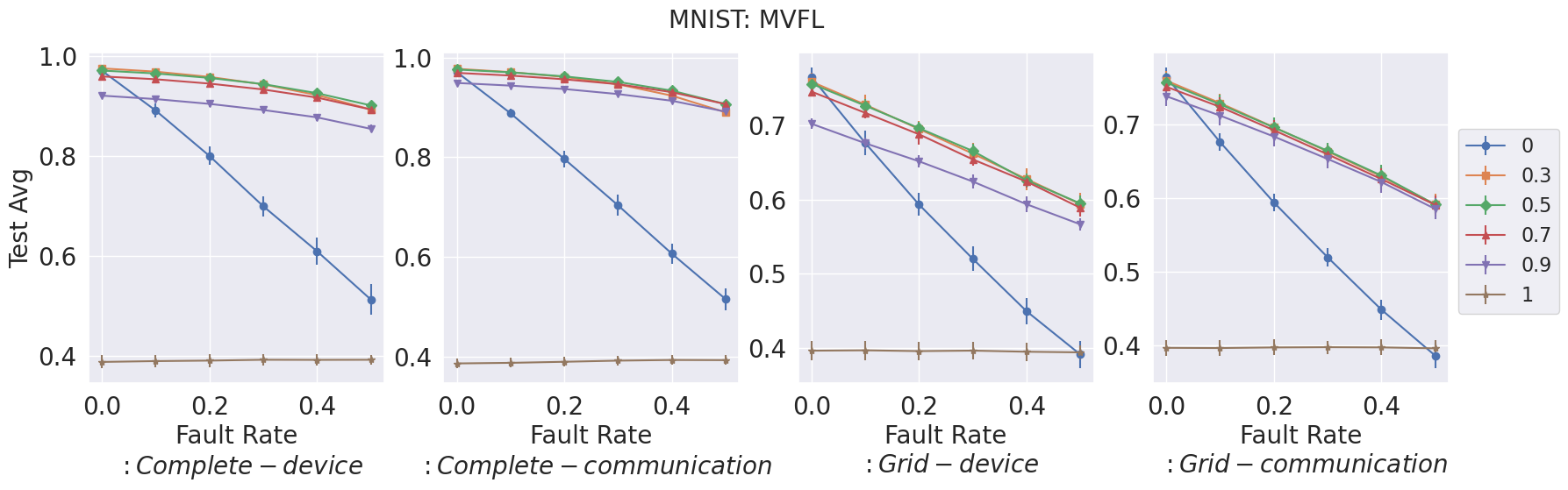}
        \caption{MVFL}
    \end{subfigure}
    \begin{subfigure}[h]{\linewidth}
        \includegraphics[width=\linewidth]{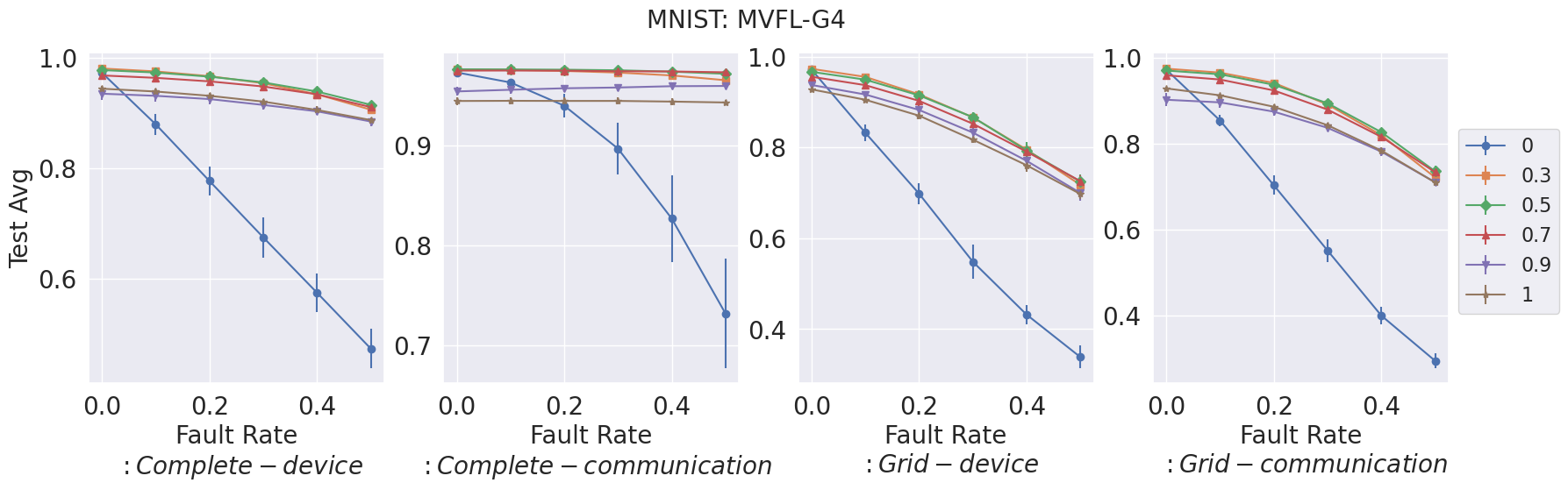}
        \caption{MVFL-G4}
    \end{subfigure}
    \caption{Test average accuracy with different dropout rates for MNIST with 16 devices.  
     Across different configurations,training with dropout makes the model robust against test time faults}
    \label{fig-app:train_fault_MNIST}
\end{figure*}

\begin{figure*}[!ht]
    \centering
    \begin{subfigure}[h]{\linewidth}
        \includegraphics[width=\linewidth]{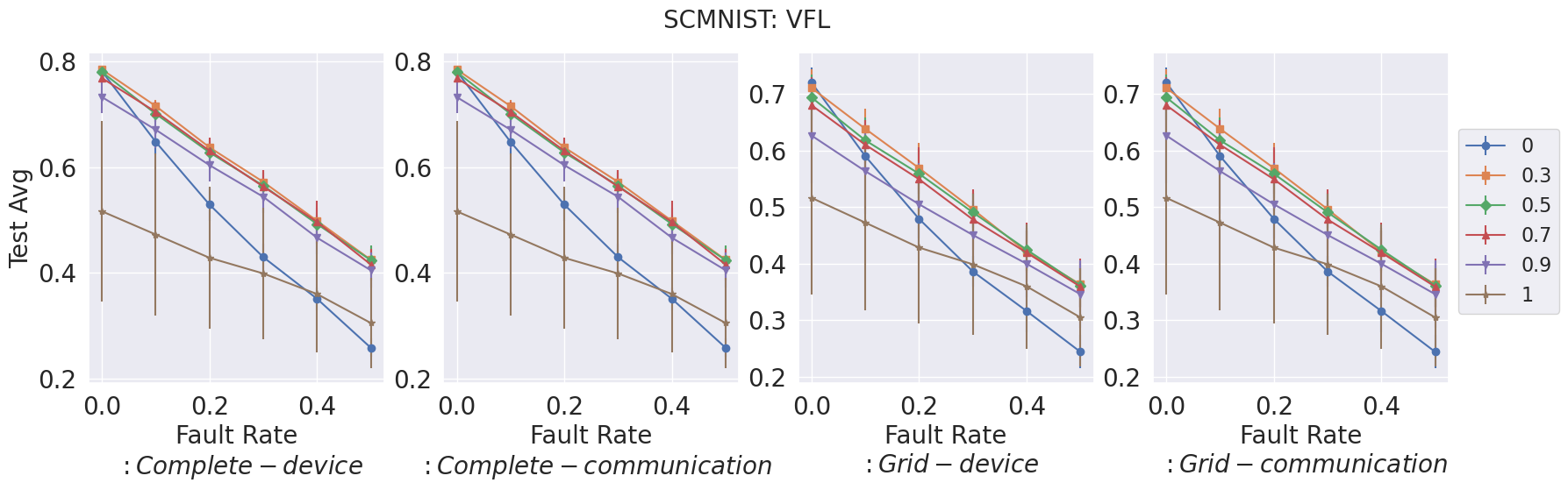}
        \caption{VFL}
    \end{subfigure}
    \begin{subfigure}[h]{\linewidth}
        \includegraphics[width=\textwidth]{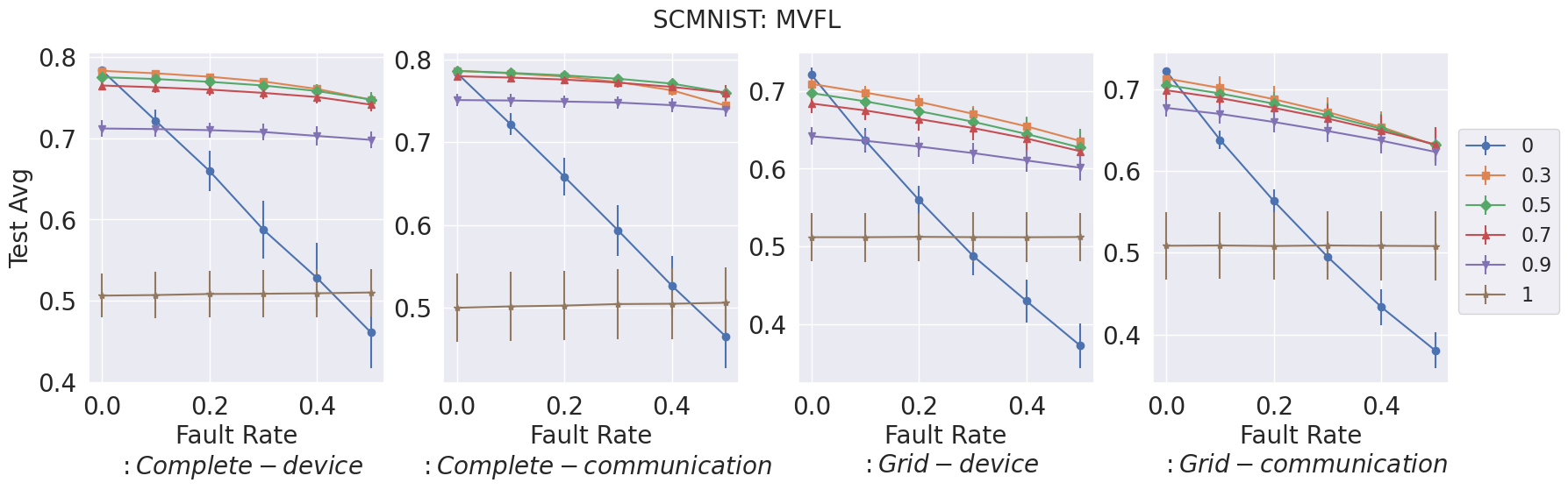}
        \caption{MVFL}
    \end{subfigure}
    \begin{subfigure}[h]{\linewidth}
        \includegraphics[width=\linewidth]{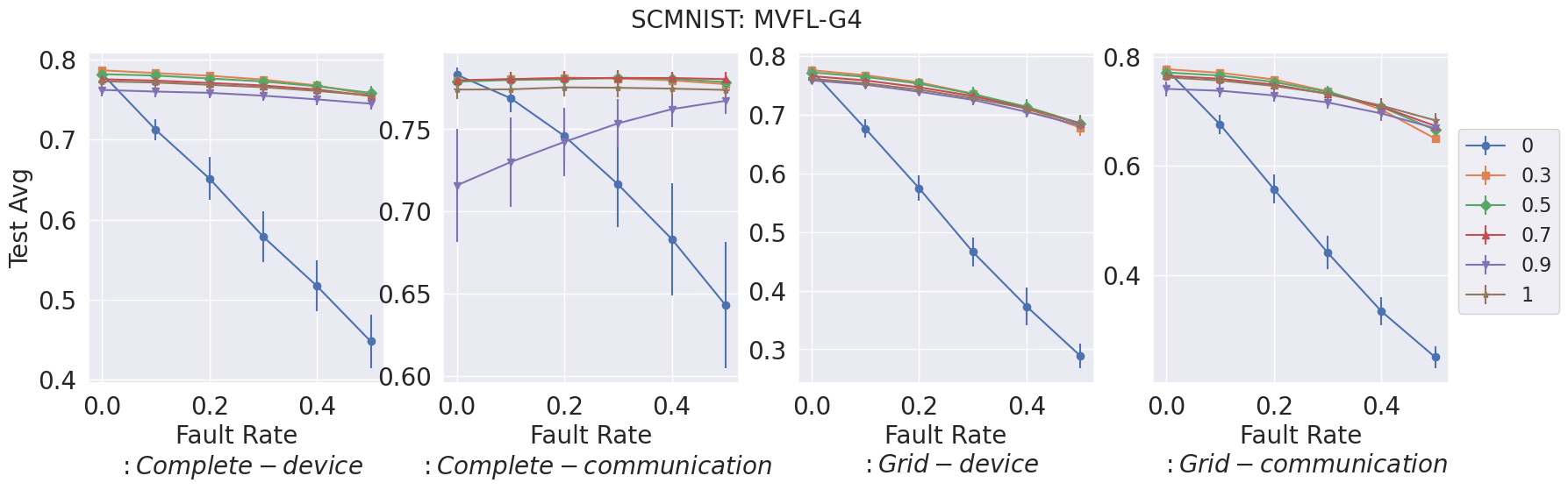}
        \caption{MVFL-G4}
    \end{subfigure}
    \caption{Test average accuracy with different dropout rates for SCMNIST with 16 devices.  
     Across different configurations,training with dropout makes the model robust against test time faults}
    \label{fig-app:train_fault_SCMNIST}
\end{figure*}

\begin{figure*}[!ht]
    \centering
    \begin{subfigure}[h]{\linewidth}
        \includegraphics[width=\linewidth]{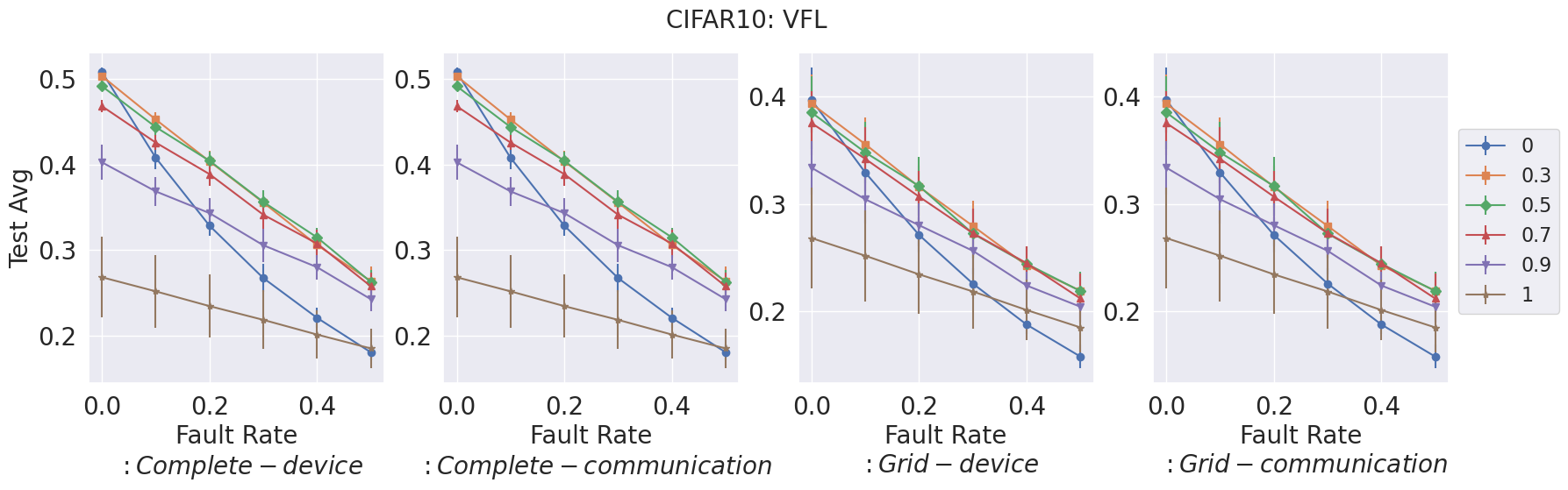}
        \caption{VFL}
    \end{subfigure}
    \begin{subfigure}[h]{\linewidth}
        \includegraphics[width=\textwidth]{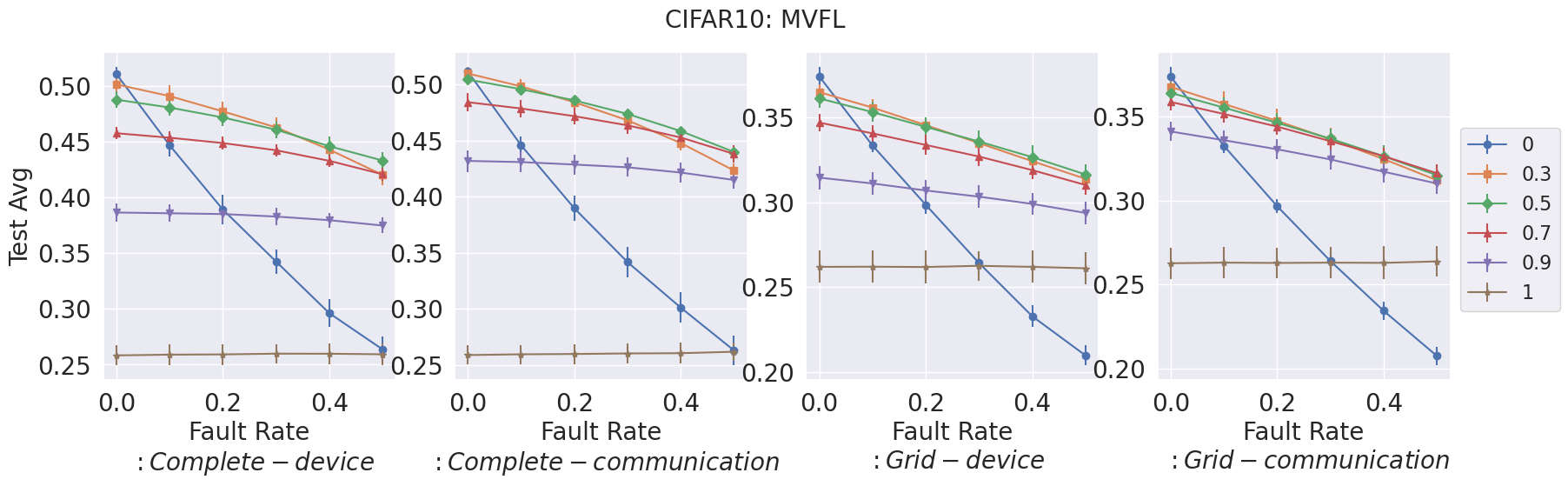}
        \caption{MVFL}
    \end{subfigure}
    \begin{subfigure}[h]{\linewidth}
        \includegraphics[width=\linewidth]{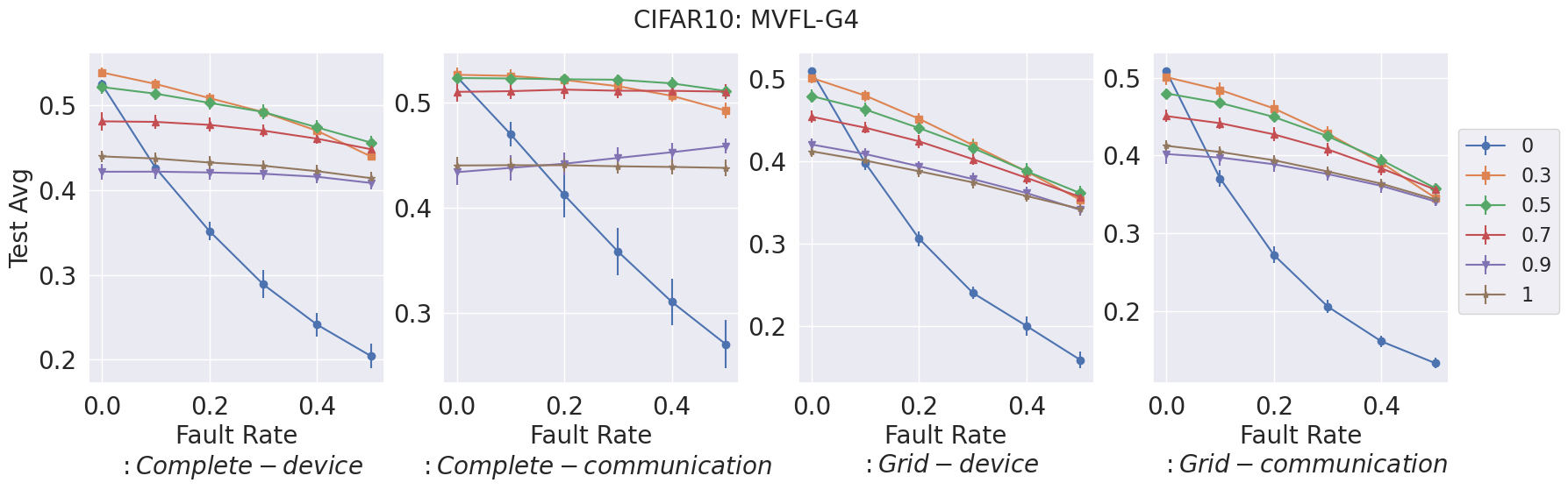}
        \caption{MVFL-G4}
    \end{subfigure}
    \caption{Test average accuracy with different dropout rates for CIFAR10 with 16 devices.  
     Across different configurations,training with dropout makes the model robust against test time faults}
    \label{fig-app:train_fault_CIFAR10}
\end{figure*}

\begin{figure*}[!ht]
    \centering
    \begin{subfigure}[h]{\linewidth}
        \includegraphics[width=\linewidth]{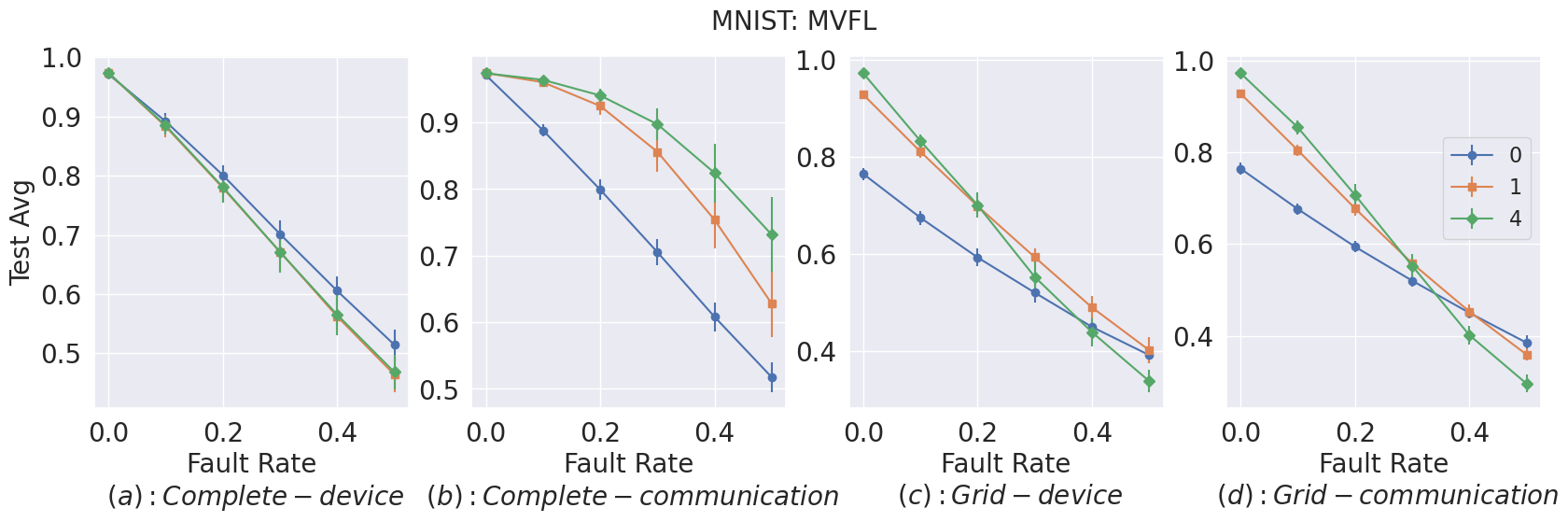}
        \caption{MNIST}
    \end{subfigure}
    \begin{subfigure}[h]{\linewidth}
        \includegraphics[width=\textwidth]{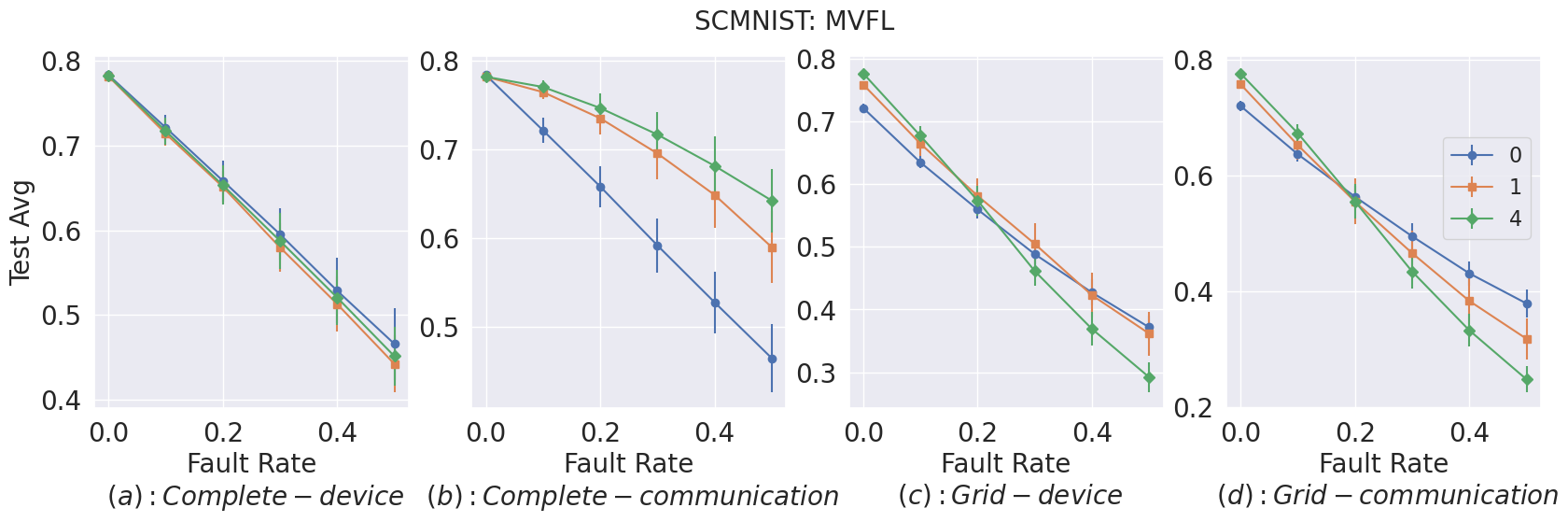}
        \caption{SCMNIST}
    \end{subfigure}
    \begin{subfigure}[h]{\linewidth}
        \includegraphics[width=\linewidth]{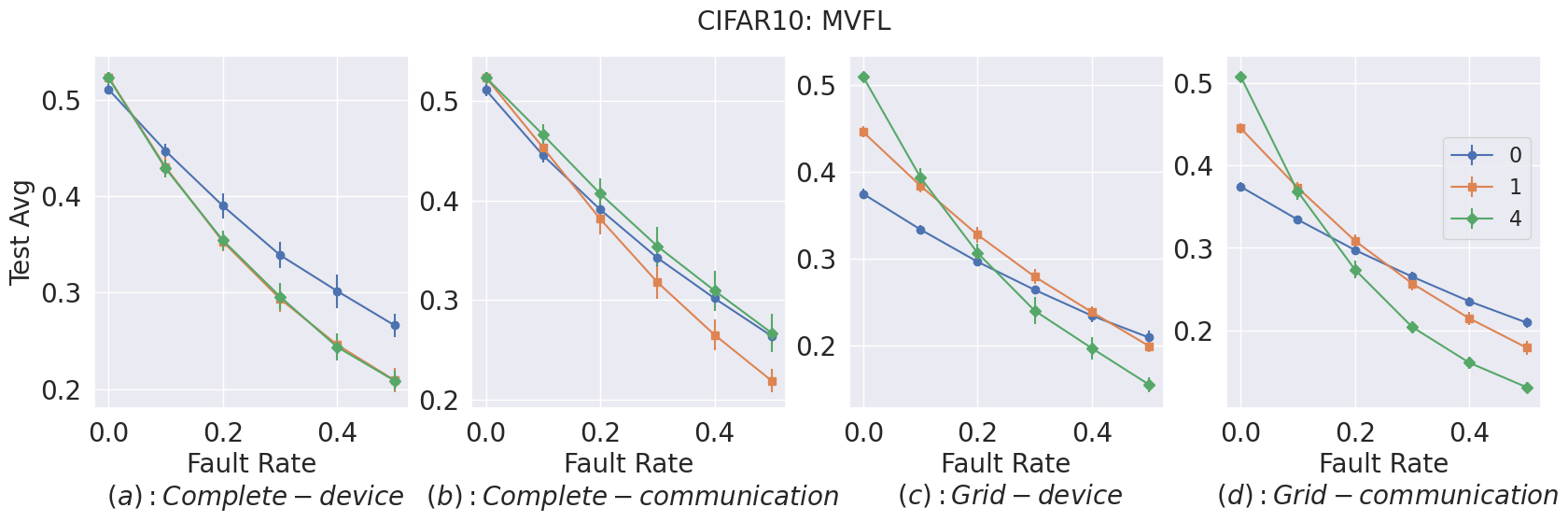}
        \caption{CIFAR-10}
    \end{subfigure}
    \caption{MVFL: Effect of different rounds of Gossip on average performance when evaluated with test time faults for 16 devices}
    \label{fig-app:MVFL_Gossip}
\end{figure*}

\begin{figure*}[!ht]
    \centering
    \begin{subfigure}[h]{\linewidth}
        \includegraphics[width=\linewidth]{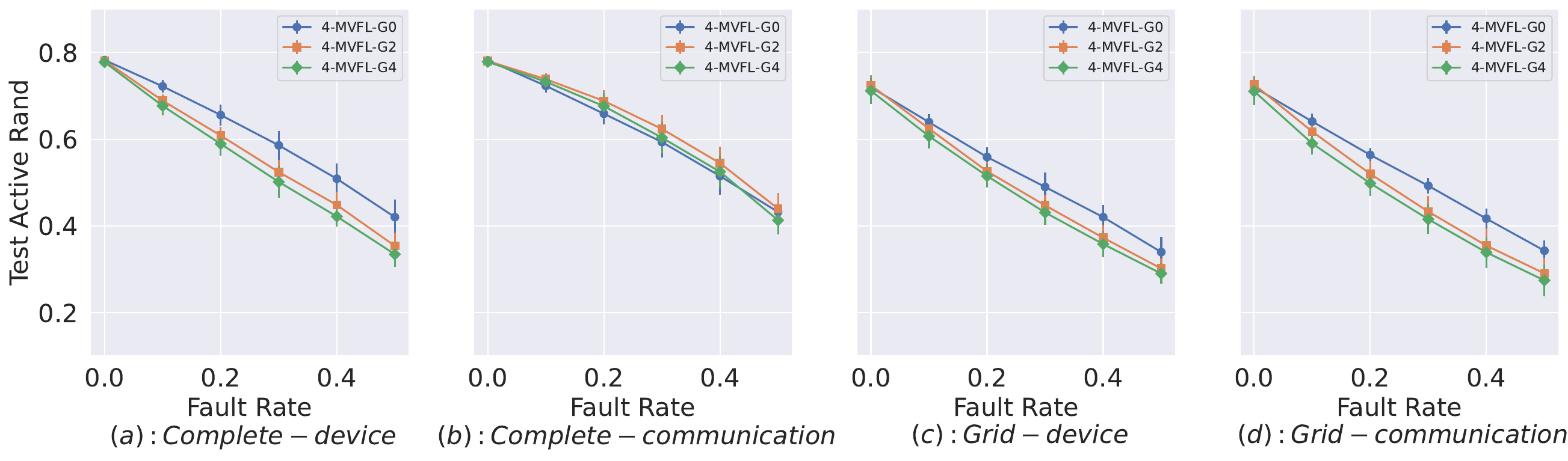}
        \caption{Without Dropout in Training}
    \end{subfigure}
    \begin{subfigure}[h]{\linewidth}
        \includegraphics[width=\textwidth]{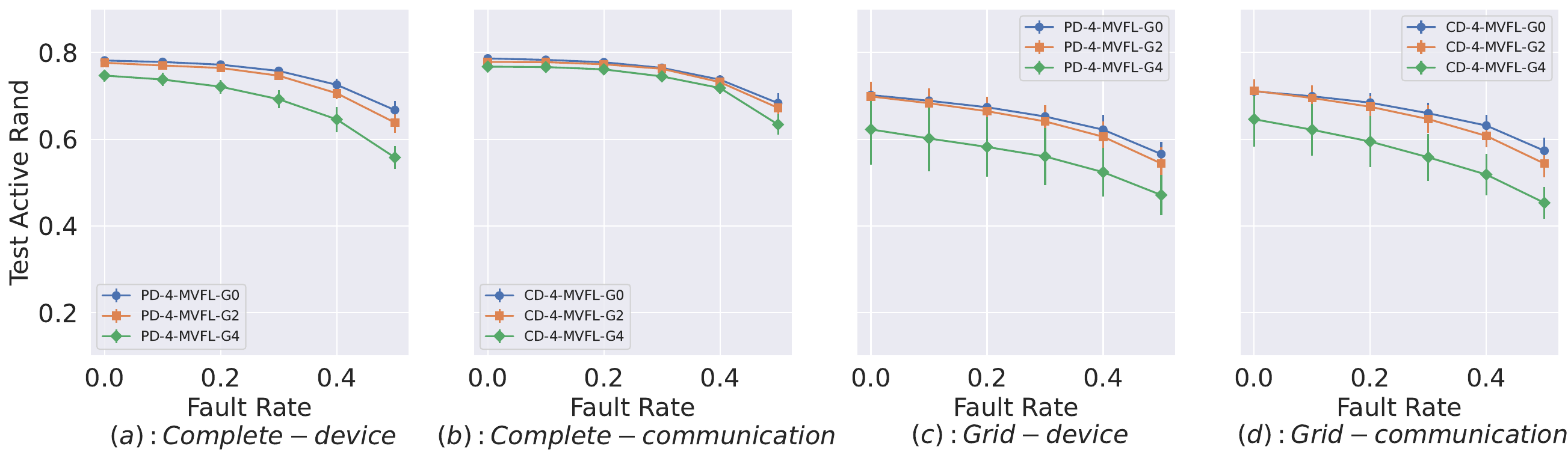}
        \caption{With Dropout in Training}
    \end{subfigure}
    \caption{4-MVFL: Effect of different rounds of Gossip on average performance when evaluated with test time faults for 16 devices for SCMNIST}
    \label{fig-app:KMVFL_Gossip}
\end{figure*}

\begin{figure*}[!h]
    \centering
    \begin{subfigure}[h]{0.23\linewidth}
        \includegraphics[width=\linewidth]{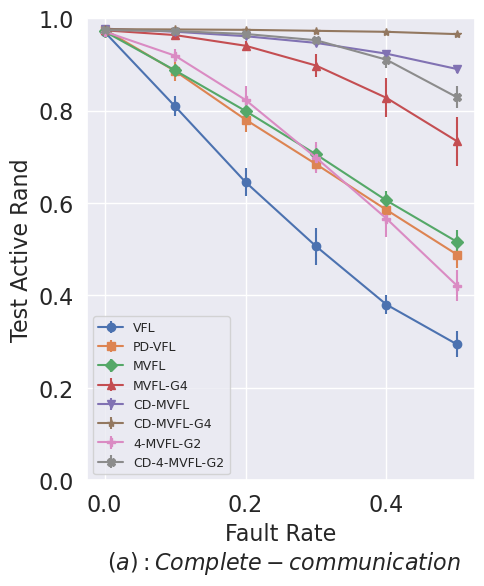}
        \caption{Active Rand}
    \end{subfigure}
    \begin{subfigure}[h]{0.23\linewidth}
        \includegraphics[width=\textwidth]{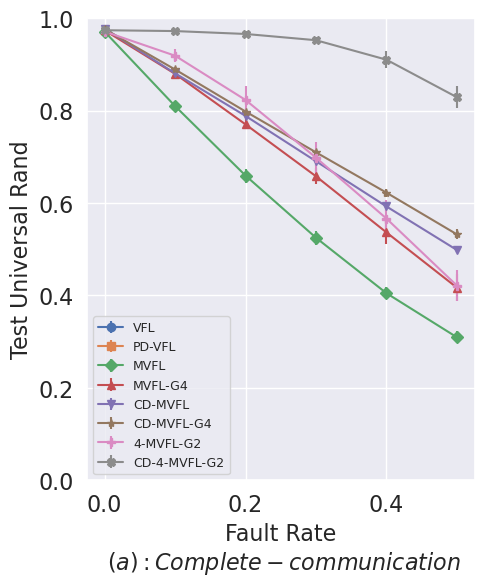}
        \caption{Any Rand}
    \end{subfigure}
    \begin{subfigure}[h]{0.23\linewidth}
        \includegraphics[width=\linewidth]{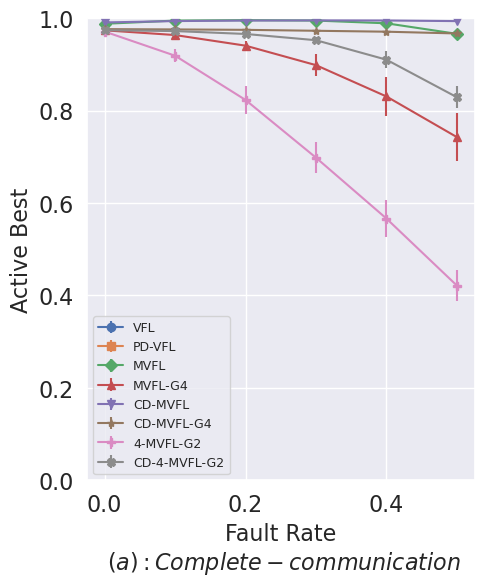}
        \caption{Active Best}
    \end{subfigure}
    \begin{subfigure}[h]{0.23\linewidth}
        \includegraphics[width=\textwidth]{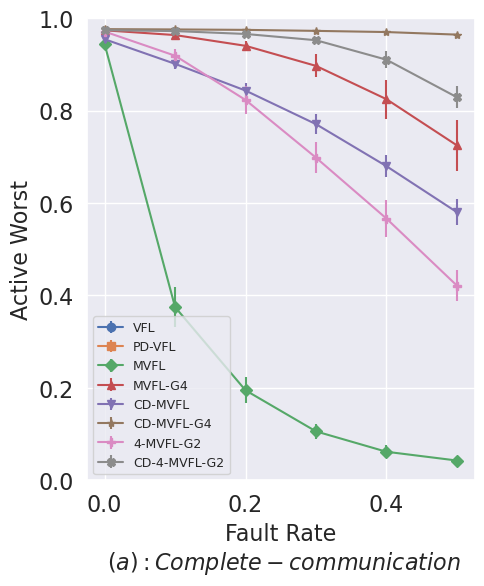}
        \caption{Active Worst}
        \end{subfigure}
    \caption{All metrics reported for MNIST with 16 devices and only test time faults}
    \label{fig-app:16D_MNIST_OtherMetric_KMVFL}
\end{figure*}

\begin{figure*}[!h]
    \centering
    \begin{subfigure}[h]{0.23\linewidth}
        \includegraphics[width=\linewidth]{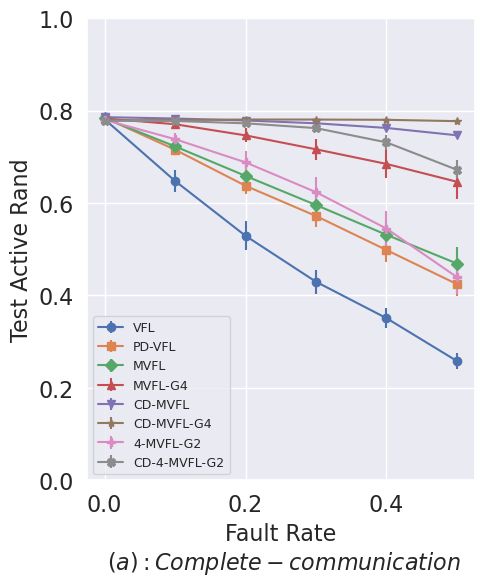}
        \caption{Active Rand}
    \end{subfigure}
    \begin{subfigure}[h]{0.23\linewidth}
        \includegraphics[width=\textwidth]{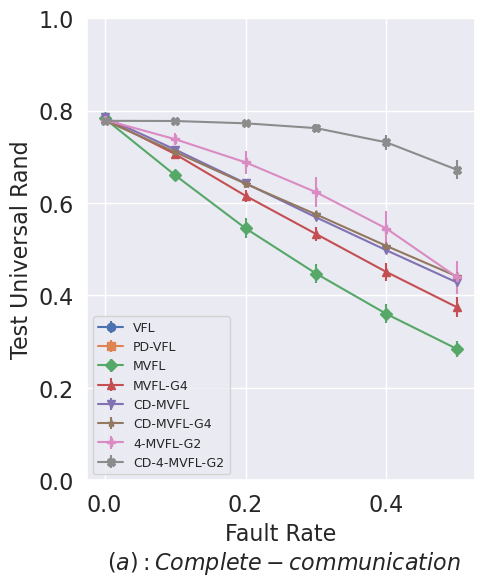}
        \caption{Any Rand}
    \end{subfigure}
    \begin{subfigure}[h]{0.23\linewidth}
        \includegraphics[width=\linewidth]{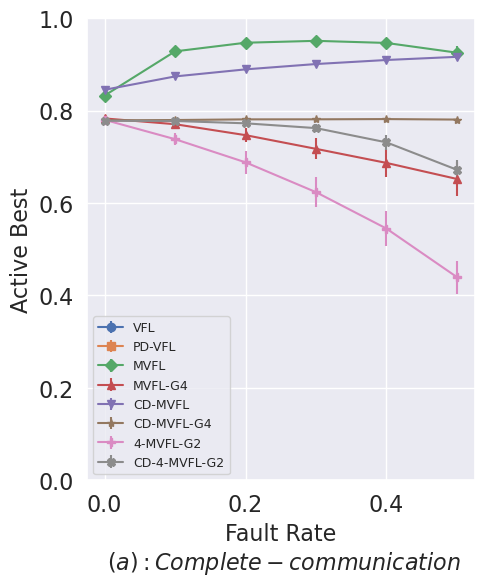}
        \caption{Active Best}
    \end{subfigure}
    \begin{subfigure}[h]{0.23\linewidth}
        \includegraphics[width=\textwidth]{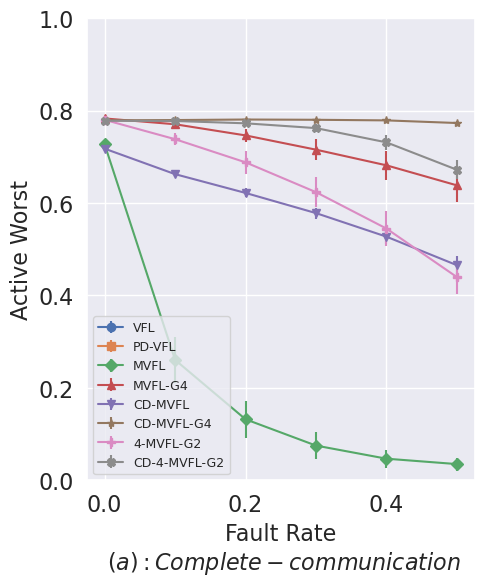}
        \caption{Active Worst}
        \end{subfigure}
    \caption{All metrics reported for SCMNIST with 16 devices and only test time faults}
    \label{fig-app:16D_SCMNIST_OtherMetric_KMVFL}
\end{figure*}

\begin{figure*}[!h]
    \centering
    \begin{subfigure}[h]{0.23\linewidth}
        \includegraphics[width=\linewidth]{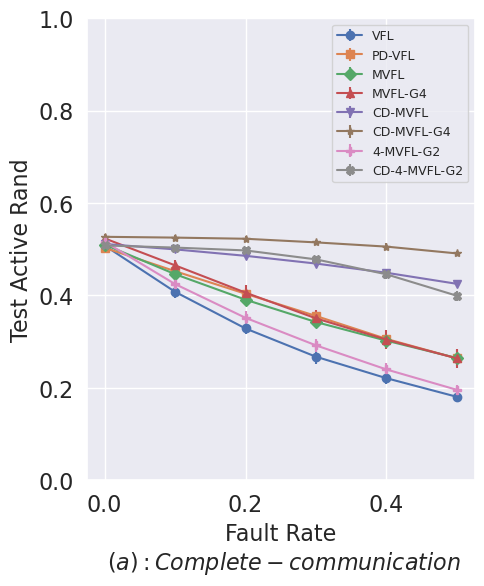}
        \caption{Active Rand}
    \end{subfigure}
    \begin{subfigure}[h]{0.23\linewidth}
        \includegraphics[width=\textwidth]{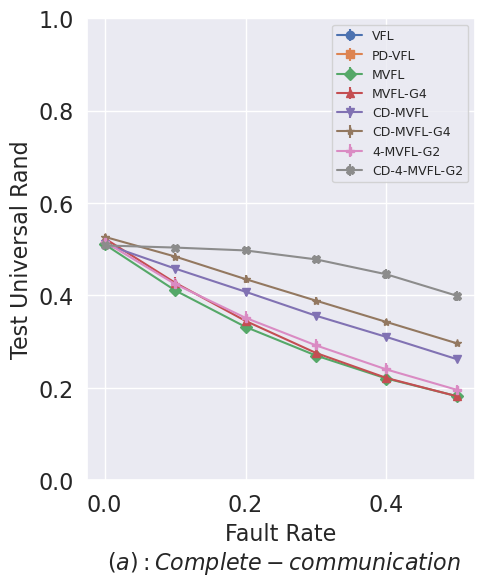}
        \caption{Any Rand}
    \end{subfigure}
    \begin{subfigure}[h]{0.23\linewidth}
        \includegraphics[width=\linewidth]{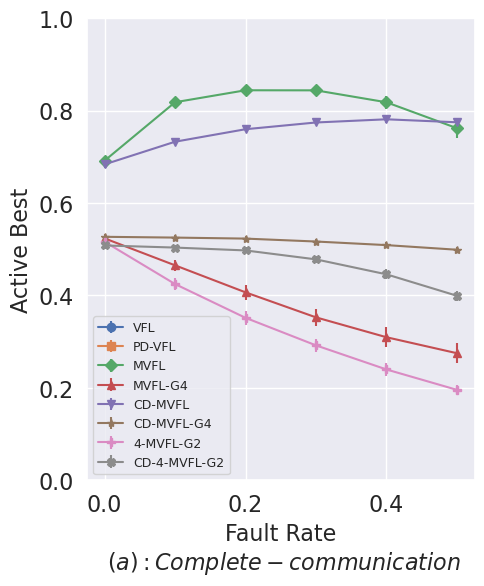}
        \caption{Active Best}
    \end{subfigure}
    \begin{subfigure}[h]{0.23\linewidth}
        \includegraphics[width=\textwidth]{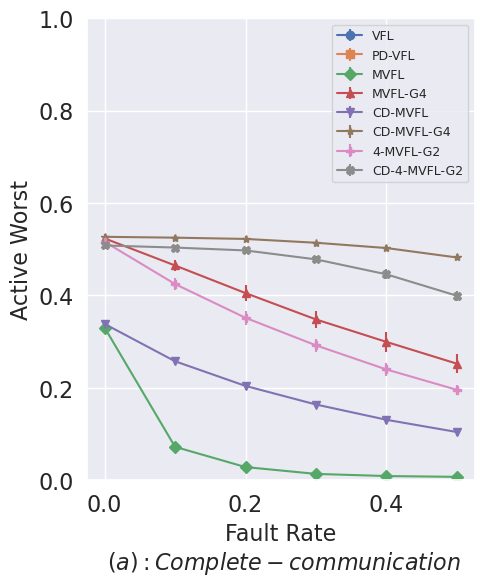}
        \caption{Active Worst}
        \end{subfigure}
    \caption{All metrics reported for CIFAR-10 with 16 devices and only test time faults}
    \label{fig-app:16D_CIFAR10_OtherMetric_KMVFL}
\end{figure*}

\begin{figure*}[!h]
    \centering
    \begin{subfigure}[h]{0.8\linewidth}
        \includegraphics[width=\linewidth]{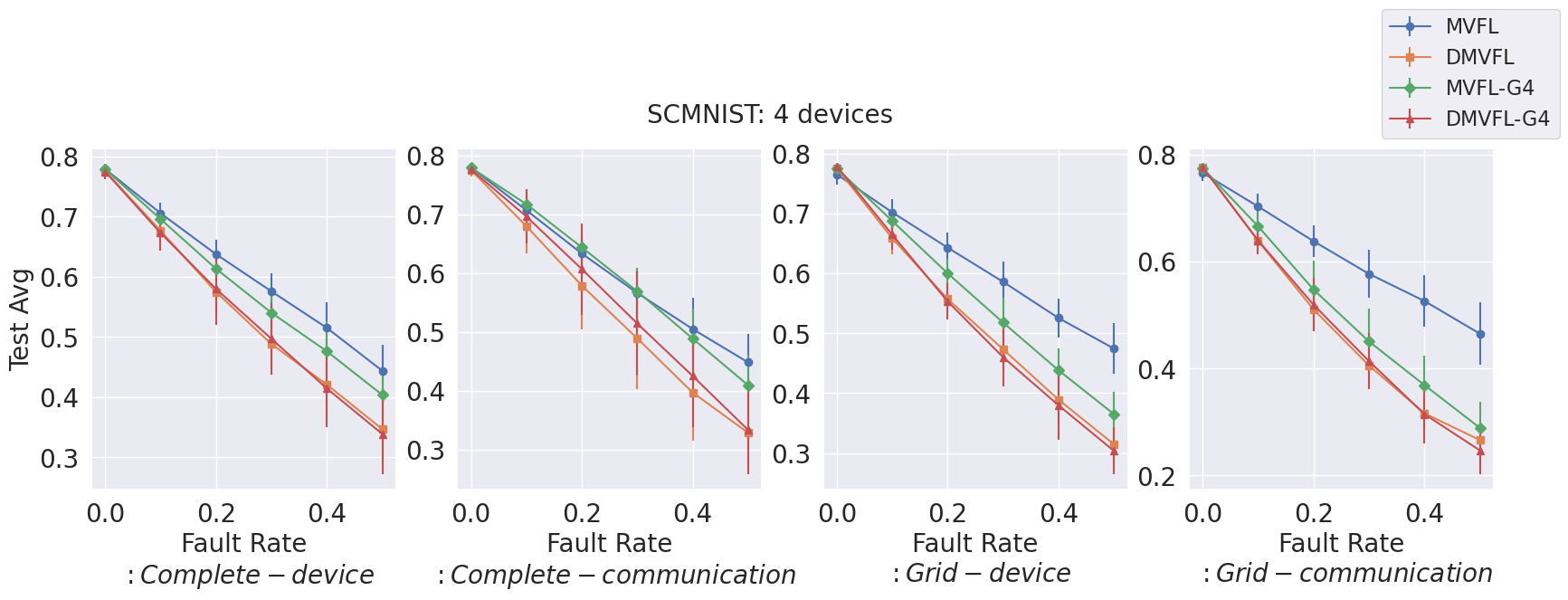}
        \caption{4 devices.}
    \end{subfigure}
    \begin{subfigure}[h]{0.8\linewidth}
        \includegraphics[width=\textwidth]{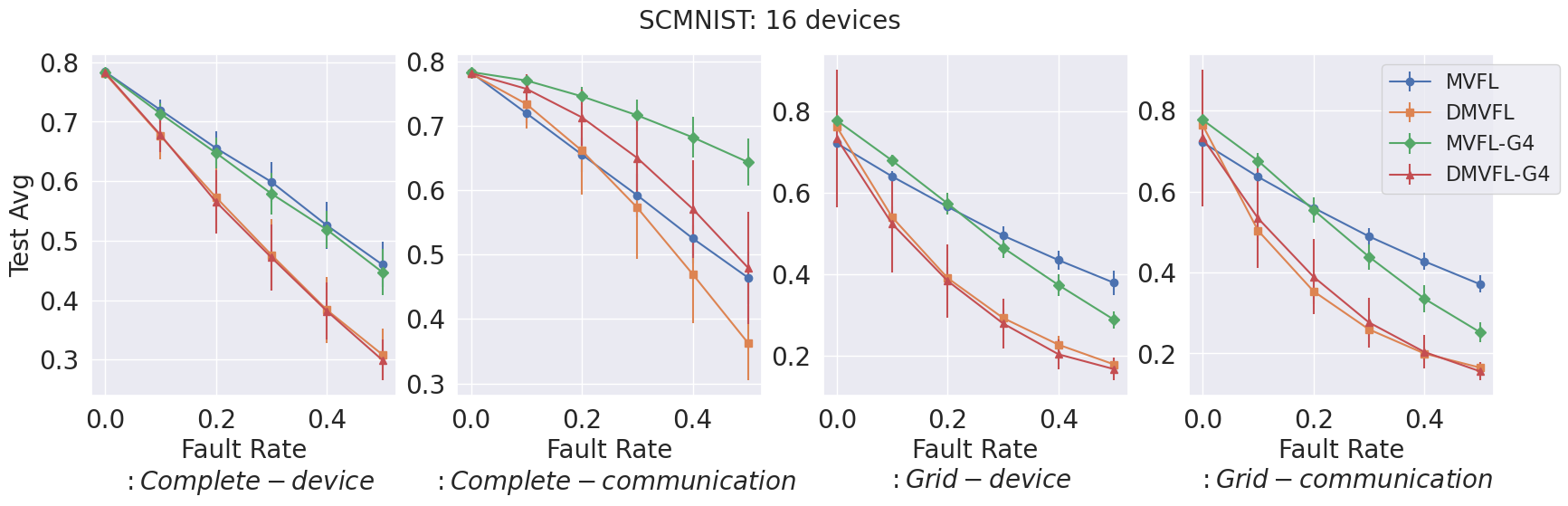}
        \caption{16 devices.}
    \end{subfigure}
    \begin{subfigure}[h]{0.8\linewidth}
        \includegraphics[width=\linewidth]{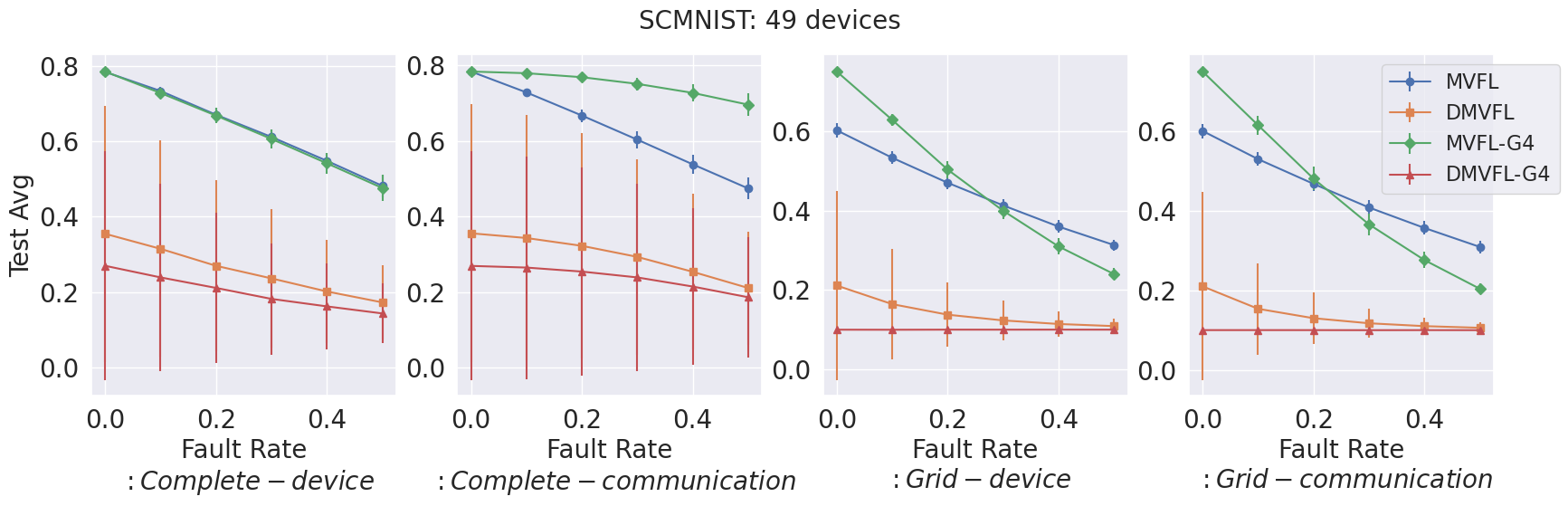}
        \caption{49 devices.}
    \end{subfigure}
    \caption{
    Test average accuracy for different test time fault rates for StarCraftMNIST with 4,16 and 49 devices. Observing the plots it can be concluded that VFL does not do well under different faulting conditions and MVFL or its gossip variant has the best performance.
    }
    \label{fig-app:ND_starmnist}
\end{figure*}

\begin{figure*}[!ht]
    \centering
        \includegraphics[width=0.99\textwidth]{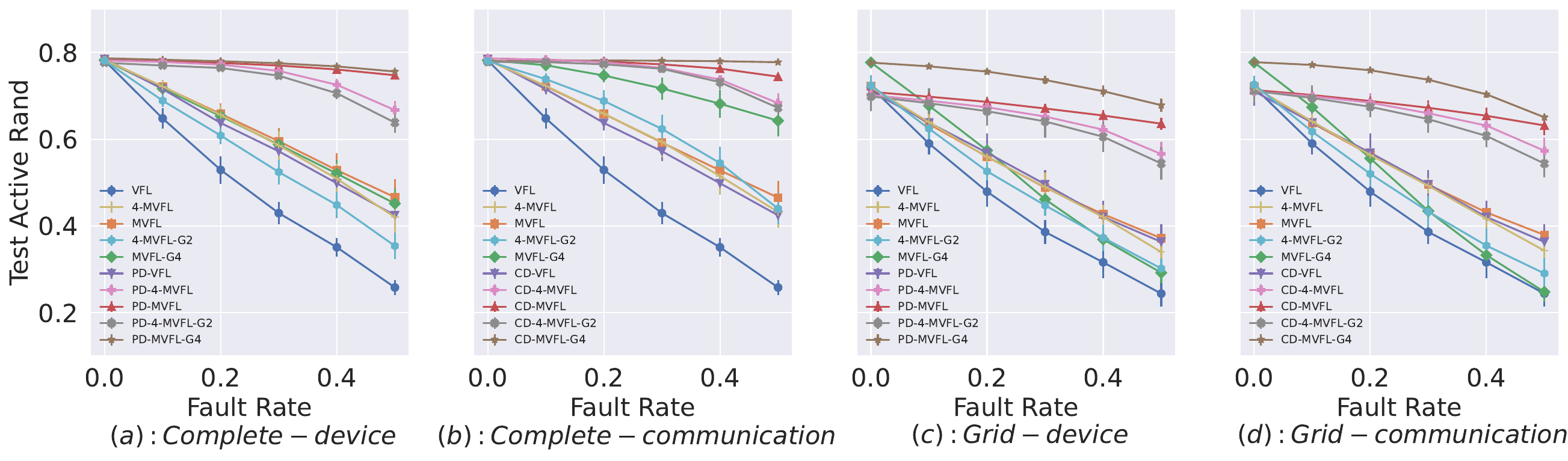}
        \vspace{-1em}
    \caption{Test accuracy with and without communication (CD-) and party-wise (PD-) Drop out method  for StarCraftMNIST with 16 devices. 
    Here we include models trained under an dropout rate of  30\% (marked by 'PD-' or 'CD-'). 
    All results are averaged over 16 runs and error bar represents standard deviation. Across different configurations, MVFL-G4 trained with feature omissions has the highest average performance, while vanilla VFL performance is not robust as fault rate increases.
    As our experiments are repeated multiple times, what we report is the expectation (Avg) over the random active client selection.
    }
    \label{fig:testavg_trainfault_faultrate_starmnist_16_Extended}
    \vspace{-0.5em}
\end{figure*}

\subsection{Partywise Dropout (PD) and Communication wise Dropout (CD) rates}\label{subsec:TrainingF}
In the main paper we presented results with assuming a Dropout rate of $30\%$ for the PD and CD variants. Here we present  the effect of different dropout rates on the test time performance under communication and device faulting regime. Figures \ref{fig-app:train_fault_MNIST}, \ref{fig-app:train_fault_SCMNIST} and \ref{fig-app:train_fault_CIFAR10} show the results for three datasets, MNIST, SCMNIST and CIFAR-10, respectively. Irrespective of communication or device fault scenario, training VFL with an omission rate results in Party wise Dropout. Whereas, for MVFL when studying communication fault having an omission rate results in CD-MVFL model while studying device fault results in PD-MVFL model.  

Across the different sets and models it is observed that using CD and PD variants results in improving the performance during test time faults. Furthermore, on observing Figures \ref{fig-app:train_fault_MNIST}, \ref{fig-app:train_fault_SCMNIST} and \ref{fig-app:train_fault_CIFAR10} (c), it seems that gossiping has a profound impact on the model performance even if the omission rates are 100\% , which means that each device is training it's own local model independently. However, by doing gossip during the testing time, devices are able to reach a consensus that gives the model a performance boost, even during high Dropout rates.

\subsection{Evaluation for a temporal fault model}\label{subsec:Temporal_Graph}

In \cref{fig:temporalGraph} we show results with a temporal communication fault model on a complete connected graph. CD- models are trained under a dropout rate of 30\%.
The temporal fault model uses Markov process to simulate the transition of links or edges in the network between connected and faulty states based on probabilistic rules defined by a transition matrix. 
The transition matrix is such that when fault rate = 0, the probability of staying in non-faulted state is 1. 
However for other fault rates, the transition matrix is: [[p, 1-p],[q , 1-q]], where r is the fault rate, q = (1-p)(1-r)/r, and p is fixed at 0.9. 
p denotes the probability of staying in non-faulted state and q depicts the probability of going from faulted to non-faulted state.
Even on a temporally varying graph, \methodAbv outperforms other baselines.

\begin{figure*}[!ht]
    \centering
        \includegraphics[clip, trim=0cm 0cm 23cm 0cm,width=0.6\linewidth]{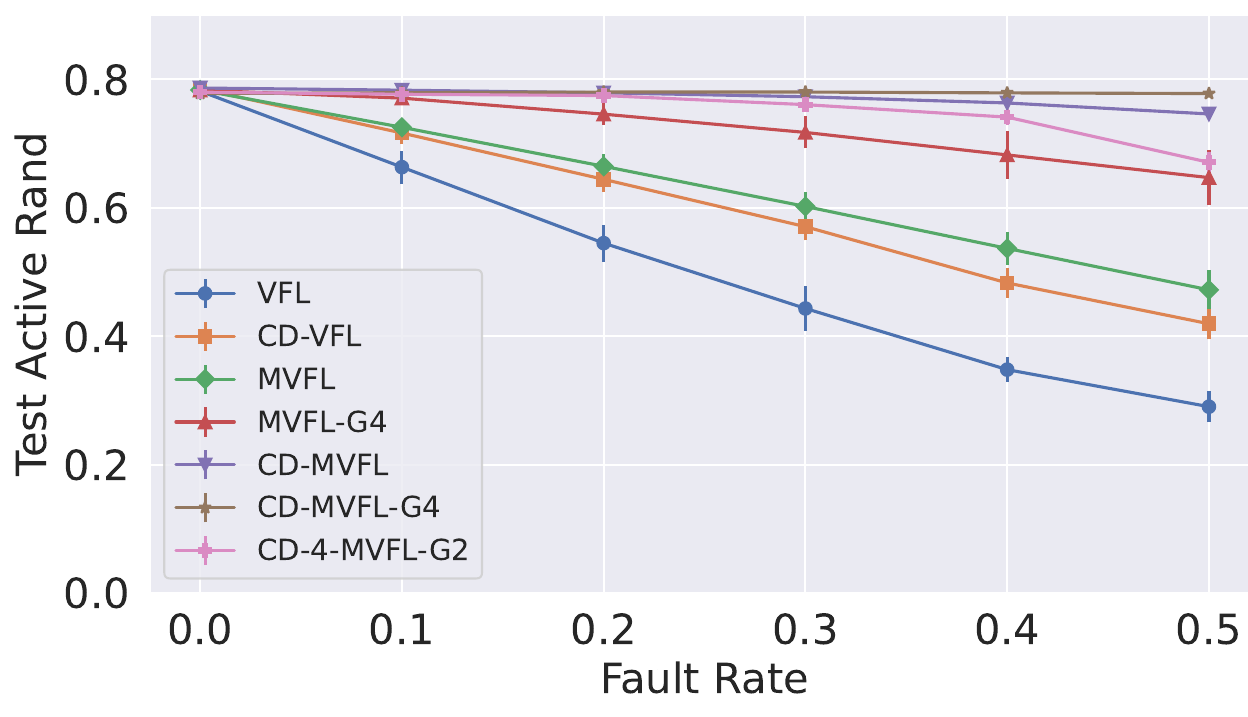}
        \vspace{-1em}
    \caption{Test accuracy with and without communication (CD-) dropout method  for StarCraftMNIST with 16 devices. 
    Here we include models trained under an dropout rate of  30\% (marked by 'CD-'). 
    All results are averaged over 16 runs, and the error bar represents standard deviation. Across different configurations, MVFL-G4 trained with feature omissions has the highest average performance, while vanilla VFL performance is not robust as fault rate increases.
    As our experiments are repeated multiple times, what we report is the expectation (Avg) over the random active client selection.
    }
    \label{fig:temporalGraph}
    \vspace{-0.5em}
\end{figure*}

\subsection{Ablation Studies}\label{subsec:AblationStudies}

\paragraph{Choice of number of Aggregators:} As an initial investigation, we studied the performance of MVFL with different numbers of aggregators for the 16 device grid, complete communication and random geometric graph with a 2.5 radius(r) setting with a 30\% communication fault rate during inference. 
The number of aggregators could be 1 (Standard VFL), 2, 4, 16 (Original MVFL). 
If the number of aggregators is less than 16, then they were chosen at random with uniform probability. 
From the \cref{tab:KMVFL_GG,tab:KMVFL_CompleteCommunication,tab:KMVFL_RGG_2_5} presented below, we see that the major improvement over VFL is achieved by just having 4 devices acting as aggregators. 
Thus, a major performance boost over VFL can be achieved at a minimal increase in communication cost, and shows that it is not necessary to have all the devices act as aggregators and incur large communication overhead. Following this result, one can infer that gossip variants of setup with few number of data aggregators than MVFL will have a significant less communication overhead than gossip variant of MVFL.

\begin{table}[]
\centering
\caption{Complete Communication Graph with 30\% communication fault rate}
\label{tab:KMVFL_CompleteCommunication}
\begin{tabular}{c|c|c|c|c}
Number of aggregators & 1(VFL) & 2     & 4    & 16    \\ \hline
Active Rand (Avg)     & 0.449  & 0.547 & 0.59 & 0.6   \\ 
\# Comm.              & 10.6   & 21    & 42   & 168.5
\end{tabular}
\end{table}

\begin{table}[]
\centering
\caption{Random Geometric Graph\@r=2.5 with 30\% communication fault rate}
\label{tab:KMVFL_RGG_2_5}
\begin{tabular}{c|c|c|c|c}
Number of aggregators & 1(VFL) & 2     & 4    & 16    \\ \hline
Active Rand (Avg)     & 0.42  & 0.52 & 0.58 & 0.59   \\ 
\# Comm.              & 7.4   & 13.9    & 29   & 114.9
\end{tabular}
\end{table}

\begin{table}[]
\centering
\caption{Grid Graph with 30\% communication fault rate}
\label{tab:KMVFL_GG}
\begin{tabular}{c|c|c|c|c}
Number of aggregators & 1(VFL) & 2     & 4    & 16    \\ \hline
Active Rand (Avg)     & 0.386  & 0.449 & 0.492 & 0.491   \\ 
\# Comm.              & 2   & 3.99    & 7.98   & 33.5
\end{tabular}
\end{table}

\paragraph{Effect of number of gossip rounds:} For the gossip (G) variants of MVFL, we are interested in studying the effect the number of gossip rounds has on average performance. In Figure \ref{fig-app:MVFL_Gossip} the effect of three different gossip rounds on the average performance for three different datasets is presented. From the plots we observe that irrespective of the method, \emph{Complete-communication} train fault benefits the most with incorporating gossip rounds. Despite \emph{Grid-communication} being a communication type of fault, gossiping does not improve the average performance. We believe this happens as a grid graph is quite sparse and training faults makes it more sparse. As a result, increasing gossip rounds does not lead to efficient passing of feature information from one client to another due to the sparseness, this is not the case in a complete graph. Furthermore, for reasons mentioned in Section \ref{sec:experiment} of the main paper and \cref{Subsec:Testing_Fault_Extension} of the Appendix, from Figure \ref{fig-app:MVFL_Gossip} we observe that adding any gossip rounds with device faults does not help in improving the performance. However, given the benefit 4 rounds of Gossip provides for Complete-Communication graph, we decided to use 4 Gossip rounds with MVFL.

We also investigated the effect of different number of Gossip rounds when using K-MVFL, in particular when the value of K is 4. \cref{fig-app:KMVFL_Gossip} (a) shows the the performance for different scenarios where dropout is not used during training and \cref{fig-app:KMVFL_Gossip} (b) shows for the condition such that dropout rate of 0.3$\%$ is used during training. It is observed that high number of Gossip rounds 4 is not having any significant benefit to performance and Gossip rounds of 0 and 2 are comparable in performance. Thus, including Gossip when the K=4 is not as beneficial as it was observed for the MVFL case. We conjecture that this likely happening because we are using gossip during training as well as during inference. We believe that using gossip only during inference and not during training will help improve the performance with more gossip rounds. This, for this paper, when using 4-MVFL, we use 2 rounds of Gossip.

\subsection{Exploration of Test Fault Rates and Patterns}\label{Subsec:Testing_Fault_Extension}

In \cref{fig:testavg_trainfault_faultrate_starmnist_16_Extended} we present a more comprehensive representation of performance of different settings of \methodSf. 
\cref{fig:testavg_trainfault_faultrate_starmnist_16} is a subset of \cref{fig:testavg_trainfault_faultrate_starmnist_16_Extended}.

A note regarding gossiping is that mainly helps MVFL in the case of communication fault.
We believe this is because in the device fault case, irrespective of number of gossip rounds, the representations from faulted device cannot be obtained. 
On the other hand, multiple gossip rounds in communication fault scenario has the effect of balancing out the lost representation at a client via neighboring connections.
Switching to the grid baseline network, a major observation here is the degradation in the performance of both MVFL and MVFL-G4.
We conjecture that in this case, clients can only directly communicate with neighboring clients, thus it's harder to get information from clients far away and extra communication leads to less benefit while the smaller network size and receiving more faulted representation become a bottleneck.
Similarly, we notice that MVFL outperforms MVFL-G4 when fault rate is very high, as there is a much higher chance that the network is disconnected in comparison to complete baseline network.
In short, we conclude that when trained with no faults, MVFL is overall the best model while gossiping helps except with high fault rates under the grid baseline network.

\begin{figure*}[!ht]
    \centering
        \includegraphics[width=0.99\textwidth]{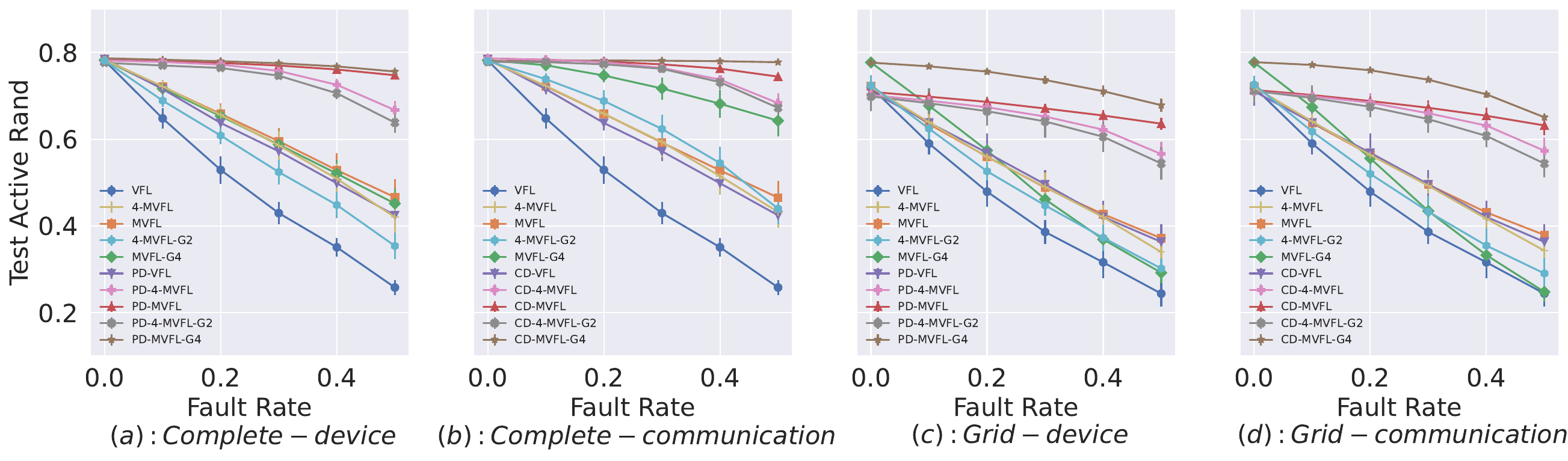}
        \vspace{-1em}
    \caption{Test accuracy with and without communication (CD-) and party-wise (PD-) Dropout method  for StarCraftMNIST with 16 devices. 
    Here we include models trained under an dropout rate of  30\% (marked by 'PD-' or 'CD-'). 
    All results are averaged over 16 runs and error bar represents standard deviation. Across different configurations, MVFL-G4 trained with feature omissions has the highest average performance, while vanilla VFL performance is not robust as fault rate increases.
    As our experiments are repeated multiple times, what we report is the expectation (Avg) over the random active client selection.
    }
    \label{fig:testavg_trainfault_faultrate_starmnist_16_Extended}
    \vspace{0em}
\end{figure*}

\subsection{Extension of Communication and Performance Analysis}\label{subsec:PerformanceCommTradeOff}
In \cref{tab:table_network_withStdev} we present the extension (performance metric is presented with Standard Deviation information) of \cref{tab:table_network}, which is shown in the main paper.

\begin{table*}[t]
\vspace{0cm}
    \centering
        \caption{
        Active Rand (Avg) performance +/- 1 Std Dev at test time with 30 $\%$ communication fault rate.
Compared to VFL, MVFL performs better but it comes at higher communication cost. Thus we propose 4-MVFL as a low communication cost alternative to MVFL. We want to highlight that 4-MVFL with poorly connected graph is still better than VFL with well connected graph, such as 4-MVFL with \textit{RGG} (r=1) versus VFL with \textit{Complete}.}
    \label{tab:table_network_withStdev}
\resizebox{1\columnwidth}{!}{
    \begin{tabular}{l|lr|lr|lr|lr|lr|lr}
\hline
        & \multicolumn{2}{c|}{Complete} & \multicolumn{2}{c|}{\begin{tabular}[c]{@{}c@{}}RGG\\ r=2.5\end{tabular}} & \multicolumn{2}{c|}{\begin{tabular}[c]{@{}c@{}}RGG\\ r=2\end{tabular}} & \multicolumn{2}{c|}{\begin{tabular}[c]{@{}c@{}}RGG\\ r=1.5\end{tabular}} & \multicolumn{2}{c|}{\begin{tabular}[c]{@{}c@{}}RGG\\ r=1\end{tabular}} & \multicolumn{2}{c}{Ring} \\ \hline
        & Avg          & \# Comm.           & Avg                                & \# Comm.                                & Avg                               & \# Comm.                               & Avg                                & \# Comm.                                & Avg                               & \# Comm.                               & Avg         &\#  Comm.        \\ \hline
VFL     &              0.430$\pm$ 0.021& 10.6           &                                    0.406$\pm$ 0.032& 7.4&                                   0.407$\pm$ 0.048& 5.2&                                    0.375$\pm$ 0.043& 3.5&                                   0.386$\pm$ 0.038& 2.0                                &             0.385$\pm$ 0.036&             1.4\\ \cdashline{1-13}
MVFL    &              0.597$\pm$0.033	& 168.5&                                    0.598$\pm$0.018& 114.9&                                   0.545$\pm$0.030& 80.8&                                    0.529$\pm$0.024& 58.7                                &                                   0.491$\pm$0.025& 33.5&             0.507$\pm$0.013&             22.7\\
4-MVFL    &              0.59$\pm$0.033	& 42&                                    0.575$\pm$0.024& 29&                                   0.551$\pm$0.037& 20.4&                                    0.514$\pm$0.027& 14.8                                &                                   0.489$\pm$0.026& 7.98&             0.485$\pm$0.026&             5.6\\
MVFL-G4    &              0.717$\pm$0.017	& 836.2&                                    0.687$\pm$0.032& 572.1&                                   0.653$\pm$0.026& 407.2&                                    0.578$\pm$0.038& 293.9                                &                                   0.441$\pm$0.029& 168.2&             0.338$\pm$0.023&             113.4\\
4-MVFL-G2  &              0.62$\pm$0.027& 126&                                    	0.509$\pm$0.044& 87&                                   0.485$\pm$0.052& 61.2&                                    0.426$\pm$0.041	& 44.8&                                   0.428$\pm$0.034& 23.94&             0.446$\pm$0.046&            
16.8\end{tabular}}
    \vspace{-0cm}
\end{table*}

\subsection{Best, Worst and Select Any Metrics}\label{Subsec:OtherMetric}

In \cref{fig-app:16D_MNIST_OtherMetric_KMVFL,fig-app:16D_SCMNIST_OtherMetric_KMVFL,fig-app:16D_CIFAR10_OtherMetric_KMVFL} we present not only the Rand Active but also Rand Universal, Active Best and Active Worst metrics when evaluation are carried out for 16 Devices/Clients under only test faults. In the main paper, Table \ref{tab:best-models-complete-comm} is a subset of the comprehensive data 
presented here.

In addition, we also share \cref{tab:best-models-complete-comm-Appendix} here, which aggregates information for two fault rates. While the table in the main paper shows data for only 1 fault rate.

\begin{table*}[t]
\vspace{0em}
    \centering
    \caption{Best models for 2 \textit{complete-communication} test fault rates within 1 standard deviation are bolded. More detailed results with standard deviation are shown in the Appendix.}
\vspace{0em}
    \label{tab:best-models-complete-comm-Appendix}
    \resizebox{1\columnwidth}{!}{\begin{tabular}{llllll|llll|lllll}
\toprule
                                  &          & \multicolumn{4}{c}{MNIST}    & \multicolumn{4}{c}{SCMNIST}  & \multicolumn{4}{c}{CIFAR10}   &  \\ 
                                  \cmidrule(lr){3-6}
                                  \cmidrule(lr){7-10}
                                  \cmidrule(lr){11-14}
                                 &          & 
                                 \multicolumn{3}{c}{Active}    & Any &
                                 \multicolumn{3}{c}{Active}    & Any &
                                 \multicolumn{3}{c}{Active}    & Any 
                                 \\
                                 \cmidrule(lr){3-5}
                                 \cmidrule(lr){6-6}
                                 \cmidrule(lr){7-9}
                                 \cmidrule(lr){10-10}
                                 \cmidrule(lr){11-13}
                                 \cmidrule(lr){14-14}
                                 &          & 
                                 Worst   & Rand  & Best  & Rand &
                                 Worst   & Rand  & Best  & Rand &
                                 Worst   & Rand  & Best  & Rand &
                                 \\
\midrule
\multirow{5}{*}{Fault Rate = 0.3} & VFL      & nan & 0.507   & nan   & nan   & nan & 0.430   & nan   & nan   & nan & 0.267   & nan   & nan   &  \\ 
                                  & PD-VFL     & nan & 0.684 & nan & nan & nan & 0.572 & nan & nan & nan & 0.355 & nan & nan &  \\
                                  & 4-MVFL-G2     & 0.632 & 0.693 & 0.751 & 0.526 & 0.572 & 0.624 & 0.675 & 0.482 & 0.238 & 0.293& 0.356& 0.232 &  \\
                                  & MVFL     & 0.106 & 0.705 & \textbf{0.995} & 0.524 & 0.075 & 0.592 & \textbf{0.951} & 0.444 & 0.013 & 0.342& \textbf{0.843}& 0.269 &  \\
                                  & MVFL-G4  & 0.897 & 0.897 & 0.899 & 0.657 & 0.715 & 0.716 & 0.717 & 0.533 & 0.348 & 0.350& 0.352 & 0.275 &  \\
                                  & CD-4-MVFL-G2  & 0.944 & 0.951 & 0.969 & \textbf{0.714} & 0.745 & 0.765 & 0.784 & \textbf{0.589} & 0.438 & 0.472& 0.528 & 0.372 &  \\
                                  & CD-MVFL-G4 & \textbf{0.972} & \textbf{0.973} & 0.972 & 0.709 & \textbf{0.780} & \textbf{0.780} & 0.781 & 0.575 & \textbf{0.514} & \textbf{0.515} & 0.516 & \textbf{0.389} &  \\ \cline{1-14}
\multirow{5}{*}{Fault Rate = 0.5} & VFL      & nan & 0.294   & nan   & nan   & nan & 0.258   & nan   & nan   & nan & 0.181   & nan   & nan   &  \\  
                                  & PD-VFL      & nan & 0.488   & nan   & nan   & nan & 0.424   & nan   & nan   & nan & 0.263   & nan   & nan   &  \\
                                  & 4-MVFL-G2      & 0.345 & 0.421   & 0.532   & 0.296   & 0.372 & 0.447   & 0.558   & 0.321   & 0.136 & 0.192   & 0.276   & 0.164   &  \\
                                  & MVFL     & 0.042 & 0.518 & 0.966 & 0.313 & 0.035 & 0.465 & \textbf{0.925}& 0.280 & 0.007 & 0.264 & \textbf{0.762}& 0.182 &  \\
                                  & MVFL-G4  & 0.724 & 0.732 & 0.742 & 0.415 & 0.643 & 0.643 & 0.846 & 0.372 & 0.252 & 0.267 & 0.275 & 0.182 &  \\
                                  & CD-4-MVFL-G2  & 0.812 & 0.822 & 0.919 & \textbf{0.552} & 0.643 & 0.678 & 0.751 & \textbf{0.464} & 0.341 & 0.392& 0.484 & 0.287 &  \\
                                  & CD-MVFL-G4 & \textbf{0.964} & \textbf{0.965} & \textbf{0.967} & 0.536 & \textbf{0.773} & \textbf{0.777} & 0.781 & 0.437 & \textbf{0.482} & \textbf{0.492} & 0.498 & \textbf{0.295} & \\
                                  \bottomrule
\end{tabular}
}
\vspace{0em}
\end{table*}

\subsection{Evaluation for different number of devices/clients}\label{subsec:DifDevices}

In Figure \ref{fig-app:ND_starmnist} we present average performance as a function of test time faults for three different number of devices. For all the different cases, it is observed that MVFL or its gossip variant performs the best. On observing the \emph{Complete-Communication} plots for Figure \ref{fig-app:ND_starmnist}, it can be seen that MVFL with gossiping has a more significant impact when the number of devices are 49 or 16 compared to when the number of devices are 4. 

\vfill

\end{document}